\documentclass[letterpaper]{article} 
\usepackage{aaai25}  
\usepackage{times}  
\usepackage{helvet}  
\usepackage{courier}  
\usepackage[hyphens]{url}  
\usepackage{graphicx} 
\usepackage{natbib}  
\usepackage{caption} 
\usepackage{algorithm}
\usepackage{algorithmic}

\usepackage{booktabs}       
\usepackage{amsfonts}       
\usepackage{nicefrac}       
\usepackage{microtype}      
\usepackage{xcolor}         
\usepackage{graphicx}
\usepackage{multirow}
\graphicspath{ {image/} }
\usepackage{longtable}
\usepackage{lipsum}
\usepackage{amsmath, bm}
\usepackage{amsthm}
\usepackage{amssymb}
\usepackage{caption}
\usepackage{graphicx}
\usepackage{subcaption}
\usepackage{color}
\usepackage{algorithm}
\usepackage{algorithmic}
\usepackage{multicol}

\newtheorem{Theorem}{Theorem}
\newtheorem{Lemma}{Lemma}
\newtheorem{Assumption}{Assumption}

\newtheorem{Definition}{Definition}

\usepackage{newfloat}
\usepackage{listings}
\DeclareCaptionStyle{ruled}{labelfont=normalfont,labelsep=colon,strut=off} 
\floatstyle{ruled}
\newfloat{listing}{tb}{lst}{}
\floatname{listing}{Listing}
\title{How Does the Smoothness Approximation Method Facilitate Generalization for Federated Adversarial Learning?}
\author{
    Wenjun~Ding\textsuperscript{\rm 1,2}, Ying~An\textsuperscript{\rm 3}, Lixing~Chen\textsuperscript{\rm 4,5}, Shichao~Kan\textsuperscript{\rm 1}, Fan~Wu\textsuperscript{\rm 1,\equalcontrib}, and Zhe~Qu\textsuperscript{\rm 1,2,\equalcontrib}
}
\affiliations{
    \textsuperscript{\rm 1}School of Computer Science and Engineering, Central South University, Changsha, China\\
    \textsuperscript{\rm 2}Xiangjiang Laboratory, Changsha, China\\
    \textsuperscript{\rm 3}Big Data Institute, Central South University, Changsha, China\\
    \textsuperscript{\rm 4}School of Electronic Information and Electrical Engineering, Shanghai Jiao Tong University, Shanghai, China\\
    \textsuperscript{\rm 5}Shanghai Key Laboratory of Integrated Administration Technologies for Information Security, Shanghai, China\\
    \{234712240, anying, kanshichao, wfwufan, zhe\_qu\}@csu.edu.cn, lxchen@sjtu.edu.cn

}

\usepackage{bibentry}

\begin{document}

\maketitle

\begin{abstract}
Federated Adversarial Learning (FAL) is a robust framework for resisting adversarial attacks on federated learning. Although some FAL studies have developed efficient algorithms, they primarily focus on convergence performance and overlook generalization. Generalization is crucial for evaluating algorithm performance on unseen data. However, generalization analysis is more challenging due to non-smooth adversarial loss functions. A common approach to addressing this issue is to leverage smoothness approximation. In this paper, we develop algorithm stability measures to evaluate the generalization performance of two popular FAL algorithms: \textit{Vanilla FAL (VFAL)} and {\it Slack FAL (SFAL)}, using three different smooth approximation methods: 1) \textit{Surrogate Smoothness Approximation (SSA)}, (2) \textit{Randomized Smoothness Approximation (RSA)}, and (3) \textit{Over-Parameterized Smoothness Approximation (OPSA)}. Based on our in-depth analysis, we answer the question of how to properly set the smoothness approximation method to mitigate generalization error in FAL. Moreover, we identify RSA as the most effective method for reducing generalization error. In highly data-heterogeneous scenarios, we also recommend employing SFAL to mitigate the deterioration of generalization performance caused by heterogeneity. Based on our theoretical results, we provide insights to help develop more efficient FAL algorithms, such as designing new metrics and dynamic aggregation rules to mitigate heterogeneity.
\end{abstract}

%

\section{Introduction}
Federated Learning (FL) \cite{mcmahan2017communication, qu2022generalized, li2023fedlga} plays an important role as it allows different clients to train models collaboratively without sharing data samples. Although FL is considered a secure paradigm to protect users' private data, it has been vulnerable to adversarial attacks \cite{wang2020attack, bagdasaryan2020backdoor}. Adversarial examples \cite{szegedy2013intriguing, goodfellow2014explaining} are typically designed to mislead models into producing incorrect outputs, significantly degrading learning performance and achieving the attacker's goals.

To counter these attacks, much effort has been made to improve neural networks' resistance to such perturbations using Adversarial Learning (AL) \cite{goodfellow2014explaining, samangouei2018defensegan, madry2017towards}. Generating adversarial examples during neural network training is one of the most effective approaches for AL \cite{carlini2017towards, athalye2018obfuscated, croce2020reliable}. Consequently, recent studies \cite{zizzo2020fat, hong2021federated, shah2021adversarial, zhou2021adversarially, li2021lomar} have proposed a new FL framework called Federated Adversarial Learning (FAL) to mitigate the impact of adversarial attacks. It is important to improve the robustness of FL in real-world applications.

There are two popular algorithms for FAL: \textit{Vanilla FAL (VFAL)} \cite{shah2021adversarial} and \textit{Slack FAL (SFAL)} \cite{zhu2023combating}. In particular, VFAL combines FedAvg \cite{mcmahan2017communication} with AL on the client side to train the global model. In contrast, SFAL modifies the aggregation process by dynamically adjusting the weights of certain client models during aggregation based on their evaluated importance.

Although these two algorithms have shown significant effectiveness and robustness, they prioritize convergence at the expense of generalization ability \cite{li2023federated, zhang2022distributed}. Generalization is a crucial aspect of evaluating an algorithm, as it measures the performance of trained models on unseen data. One popular approach to studying generalization is to examine algorithmic stability \cite{bousquet2002stability, hardt2016train}, which measures sensitivity to perturbations in the training dataset. Recently, a series of studies \cite{xing2021algorithmic, xiao2022stability, cheng2024stability} have investigated the generalization error in AL.

Unfortunately, these generalization studies primarily focus on the centralized setting. In FAL, one of the most important problems is the heterogeneity across clients, which leads to locally generated adversarial examples being highly biased to each local distribution \cite{zhang2023delving, chen2022calfat}. Experimental results have demonstrated that this exacerbated data impairs the generalization ability of FAL. Thus, the generalization bounds of AL are insufficient for generalizing to FAL. While \cite{sun2023understanding, lei2023stability} explored the relationship between data heterogeneity and generalization in FL, characterizing generalization in FAL poses additional challenges.

Unlike FL, the loss function in FAL \cite{sadeghi2020learning, zhou2021adversarially} is typically non-smooth, violating the stability analysis of algorithms that rely on smooth functions \cite{hardt2016train, sun2023understanding, lei2023stability}. A practical way to tackle this issue is to employ a method of smoothness approximation, e.g., (1) \textit{Surrogate Smoothness Approximation (SSA)} \cite{cui2021addressing}, (2) \textit{Randomized Smoothness Approximation (RSA)} \cite{alashqar2023gradient}, and (3) \textit{Over-Parameterized Smoothness Approximation (OPSA)} \cite{li2023federated, li2022convergence}. Although these methods can mitigate the non-smoothness property, generating uncertain adversarial examples may still lead the global model into sharp valleys and reduce the consistency of FAL, hindering its generalization performance. This raises the question:
\textit{\bf How to properly set the smoothness approximation method to reduce the generalization error for FAL?}

The answer to the above question is two-fold: 1. Which approximation method is most suitable for FAL? 2. How should the parameters of each approximation method be set to reduce the generalization error? Therefore, in this paper, we provide detailed analyses for VFAL \cite{zizzo2020fat, shah2021adversarial, li2022convergence, zhang2022distributed} and SFAL \cite{zhu2023combating} using the three smoothness approximation methods. To the best of our knowledge, this is the first study to investigate the generalization of FAL. Our findings will offer valuable insights for developing efficient FAL algorithms.
The main contributions can be summarized as follows: 

(1) We demonstrate the generalization bound of the three smoothness approximation methods for the VFAL algorithm. Our results first indicate that increased heterogeneity significantly impairs the ability to generalize. Then, the result of RSA shows the best generalization error, scaling with $\mathcal{O}(T)$. Notably, SSA shows improved performance by adding less noise. For OPSA, properly controlling the width of the neural network can reduce the generalization error. (2) We also provide the generalization bound of SFAL and compare it with VFAL. Our analysis demonstrates the impact of different global aggregation methods on generalization bound. In particular, considering some AL-related metrics, i.e., local adversarial loss, as a client contribution evaluation criterion to design dynamic global aggregation rules helps to improve generalization. (3) Based on our results, we have proposed some new metrics related to adversarial loss, e.g., contrastive loss and adversarial penalty, that may not have been considered in previous work. We believe that incorporating these metrics for dynamic global aggregation or into the local training could help design a more efficient FAL algorithm.

\section{Related Work}
FL is considered an efficient and privacy-preserving method for distributed learning environments. Generally, FL can be categorized into aggregation schemes \citep{mcmahan2017communication, qu2022generalized} and optimization schemes \citep{pmlr-v97-mohri19a, reddi2020adaptive}. However, FL remains vulnerable to adversarial attacks. Adversarial examples pose a significant threat to learning models, as perturbations in input data can mislead classifiers. To combat this, AL has been proposed to bolster robustness \citep{goodfellow2014explaining, samangouei2018defensegan, madry2017towards}. Recent studies have also delved into robust overfitting \citep{rice2020overfitting}. However, directly integrating AL into FL poses several challenges, including poor convergence, data heterogeneity, and communication costs. To address these challenges, \cite{zizzo2020fat, shah2021adversarial, zhang2022distributed, li2023federated, zhu2023combating} have proposed corresponding FAL algorithms.

Generalization analysis has been widely used to evaluate algorithms in both FL and AL. \cite{pmlr-v97-mohri19a} leverages Rademacher complexity to develop a uniform convergence bound for FL. Subsequently, various FL studies have introduced generalization bounds based on Rademacher complexity \cite{sun2022communication, qu2023prevent}. \cite{hu2022generalization} considers a faster rate with Bernstein conditions and bounded losses under convex objectives. Additionally, \citep{sun2023understanding, lei2023stability} explore generalization upper bounds using on-average stability \cite{kearns1997algorithmic, kuzborskij2018data}. Moreover, \cite{xing2021algorithmic} explores stability by highlighting the non-smooth nature of adversarial loss, and \cite{xiao2022stability} develops stability bounds using smoothness approximation. Building on these insights, \cite{xiao2022smoothed, xing2021generalization} propose smoothed versions of SGDmax and robust deep neural networks.

\section{Preliminaries}

\subsection{Federated Adversarial Learning (FAL)}
To improve the robustness of FL against adversarial attacks, some studies have developed the FAL framework \citep{zizzo2020fat, shah2021adversarial, zhang2022distributed, li2023federated, zhu2023combating}. Although these frameworks have demonstrated efficiency from the convergence perspective, their generalization performance may limit their applicability, which motivates us to analyze the underlying issues.

Typically, we consider a FAL framework with $m$ clients. Each client $i$ obtain the local dataset $(x_i,y_i)\in\mathcal{S}_i$ coming from the unknown distribution $P_i(P_i\neq P_{i^{\prime}},i\neq i^\prime$) with the size $n_i = |\mathcal{S}_i|$. Let $\ell(\theta,(x,y))$ be the loss function with the learning model $\theta$. The local objective of each client $i$ is to minimize the local population risk, which is defined by:
\begin{equation}\label{Eq:loaclrisk}
    R_i(\theta) = \mathbb{E}_{(x_i,y_i)\sim P_i}[\ell\left(f_\theta[x_i+A_\rho(f_\theta,x_i,y_i)],y_i\right)],
\end{equation}
where $A_{\rho}$ is an attack of strength $\rho > 0$ and intended to deteriorate the loss in the following way:
\begin{equation}\label{attack}
    A_\rho(f_\theta,x_i,y_i):=\underset{\delta\in B_p (0,\rho)}{\operatorname*{argmax}}\{\ell(f_\theta(x_i+\delta),y_i)\},
\end{equation}
where $B_p (0,\rho)$ is a $\mathcal{L}_p$ ball with radius $\rho$ and $p =1,2$ or, $\infty$ for different types of attacks. For simplicity, we rewrite $\ell\left(f_\theta[x+A_\rho(f_\theta,x,y)],y\right) = \ell_\rho (\theta,z)$ and $z=\{x,y\}$ in this paper.
Beyond the individual local objectives, all $m$ clients collaboratively minimize the global objective, defined as $R(\theta)=\frac{1}{m}\sum_{i=1}^mR_i(\theta)$. However, directly minimizing the global objective $R(\theta)$ is challenging due to the unknown distributions $P_i$. Thus, a common way to approach $R(\theta)$ is to minimize the following global empirical risk:
\begin{equation}
\label{emrisk}
R_{\mathcal{S}}(\theta):= \frac{1}{m}\sum_{i=1}^{m}R_{\mathcal{S}_{i}}(\theta)=\frac{1}{m}\sum_{i=1}^{m} \frac{1}{n_i} \sum_{j=1}^{n_{i}} \ell_{\rho}\left(\theta ; z_{i, j}\right),
\end{equation}
where $R_{\mathcal{S}_i}(\theta)$ is the local empirical risk $R_{\mathcal{S}_i}(\theta) := \frac{1}{n_i} \sum_{j = 1}^{n_i} \ell_{\rho}(\theta;z_{i,j})$ with the local sample dataset $\mathcal{S}_i$. To investigate the generalization error in FAL, we focus on analyzing its algorithmic stability, which evaluates how sensitive a model is to changes in its training data. In the context of FAL, where the training dataset is distributed across various clients, it is crucial to assess how perturbations in the data at each client level affect the global model's stability.

\subsection{Stability and Generalization}
The generalization adversarial risk $\varepsilon_{gen}(\theta)$ is defined as the difference between population and empirical risk, i.e., $\varepsilon_{gen}:=R(\theta)-R_{\mathcal{S}}(\theta)$. For a potentially randomized algorithm $\mathcal{A}$ that takes a dataset $\mathcal{S}$ as input and outputs a random vector $\theta=\mathcal{A}(\mathcal{S})$, we can define its expected generalization adversarial risk over the randomness of a training set $\mathcal{S}$ and stochastic algorithm $\mathcal{A}$ as follows:
\begin{equation}\label{define:generalization}
    \varepsilon_{gen}(\mathcal{A}) = \mathbb{E}_{\mathcal{S,A}}[R(\mathcal{A}(\mathcal{S})) - R_{\mathcal{S}}(\mathcal{A}(\mathcal{S})].
\end{equation}
One of the most popular ways to approach \eqref{define:generalization} is to consider the on-average stability \cite{kearns1997algorithmic, shalev2010learnability, kuzborskij2018data}. Moreover, in our focused FAL, the generalization adversarial risk can be decomposed as $\varepsilon_{gen}(\mathcal{A}) = \mathbb{E}_{\mathcal{S,A}}[\frac{1}{m}\sum_{i=1}^{m}(R_i(\mathcal{A}(\mathcal{S}))-R_{\mathcal{S}_i}(\mathcal{A}(\mathcal{S})))]$, which implies that the generalization ability of the server is related to each client. Therefore, we are interested in the change in algorithm performance when one data sample is perturbed in any client. We introduce the following definitions related to on-average stability for the FAL framework, which is similar to those in \cite{sun2023understanding, lei2023stability}.
\begin{Definition}[Neighboring Datasets]\label{Definition 1}
   Given the entire dataset $\mathcal{S}=\cup_{i=1}^{m} \mathcal{S}_{i}$, where $\mathcal{S}_{i}$ is the local dataset of the $i$-th client with $\mathcal{S}_{i}=\{z_{i, 1}, \ldots, z_{i, n_{i}}\}$, $\forall i \in [m]$, another dataset is called as neighboring to $\mathcal{S}$ for client $i'$, denoted by $\mathcal{S}^{(i')}$, if $\mathcal{S}^{(i')}:=\cup_{i \neq i'} \mathcal{S}_{i} \cup \mathcal{S}_{i'}^{\prime}$, where $\mathcal{S}_{i'}^{\prime}=\{z_{i', 1}, \ldots, z_{i', j-1}, z_{i', j}^{\prime}, z_{i', j+1}, \ldots, z_{i', n_{i}}\}$ with $z_{i', j}^{\prime} \sim P_{i'}$, $\forall j \in\left[n_{i'}\right]$. And we call $z_{i', j}^{\prime}$ the perturbed sample in $\mathcal{S}^{(i')}$.
\end{Definition}

\begin{Definition}[On-Average Stability for FAL]\label{Definition 2}
     A randomized algorithm $\mathcal{A}$ is $\epsilon$-on-average stability if given any two neighboring datasets $\mathcal{S}$ and $\mathcal{S}^{(i')}$, then
    $$ \max _{j \in\left[n_{i}\right]} \mathbb{E}_{\mathcal{A}, \mathcal{S}, z_{i', j}^{\prime}}|\ell_{\rho}(\mathcal{A}(\mathcal{S}) ; z_{i', j}^{\prime}) - \ell_{\rho}(\mathcal{A}(\mathcal{S}^{(i')}) ; z_{i', j}^{\prime})| \leq \epsilon,$$
    where $z_{i',j}'$ is the perturbed sample in $\mathcal{S}^{(i')}$, $\forall i \in[m] $.      
\end{Definition}
On-average stability basically means any perturbation of samples across all clients cannot lead to a big change in the model trained by the algorithm in expectation.
Moreover, we need to state a global assumption and a key lemma to indicate the data heterogeneity which is crucial for our later theorems.
\begin{Assumption}\label{Assumption 1}
    (Lipschitz continuity). The loss function $\ell$ satisfies the following Lipschitz smoothness conditions: $\Vert \ell(\theta_1,z)-\ell(\theta_2,z)\Vert \leq L\Vert\theta_1-\theta_2\Vert$, $\Vert\nabla \ell(\theta_1,z)-\nabla \ell(\theta_2,z)\Vert \leq L_{\theta}\Vert \theta_1 - \theta_2 \Vert$, and $\Vert\nabla \ell(\theta,z_1)-\nabla \ell(\theta,z_2)\Vert \leq L_{z}||z_1-z_2\Vert_p$, where $\nabla$ is the abbreviation for $\nabla_{\theta}$ used throughout the paper.
\end{Assumption}
Note that Assumption~\ref{Assumption 1} is widely used in existing studies \cite{xing2021algorithmic, li2023federated, zhu2023combating, kanai2023relationship}. Intuitively, different local distributions affect the global population risk and hence may affect the model generalization as well. To effectively measure the exacerbated heterogeneity of client $i$ in AL, we account for both their original data distribution $P_i$ and the distribution of adversarially generated samples $\tilde{P_i}$. 

\begin{Lemma}
    \label{Lemma 1}
    Under Assumption 1 and given $i\in [m]$, for any $\theta$ we have
    \begin{equation*}
        \Vert\nabla R_i(\theta)-\nabla R(\theta)\Vert\leq (2\rho L_z+6L)D_i,
    \end{equation*}
    where $D_i$ = $\max \{d_{TV}(\tilde{P}_i, P_i), d_{TV}(P_i, P), d_{TV}(\tilde{P},P)\}$.
\end{Lemma}

{\bf Remark 1.} The total variation distance $d_{TV}$ is used to compare these distributions against their respective global counterparts, $P$ and $\tilde{P}$. We define $D_i$ as the maximum total variation observed among these comparisons. This metric captures the extent of data variation from both regular and adversarial perspectives, with higher values indicating greater heterogeneity. This lemma reveals that when AL tries to gain more robustness through stronger adversarial generation, the heterogeneity will be exacerbated. When $\rho = 0$, we have $d_{TV}(\tilde{P}_i, P_i)=0,  d_{TV}(\tilde{P},P)=0$, thus $D_i$ is the same as \citep{sun2023understanding} in FL. In particular, we define $D_{\max}=\max_{i\in [m]} D_i$ to denote the maximum heterogeneity among all clients in the FAL framework.

\section{Vanilla FAL Algorithm}
To approach the global objective in \eqref{emrisk}, \cite{zizzo2020fat, shah2021adversarial, li2022convergence, zhang2022distributed} designed the Vanilla FAL (VFAL) algorithm. The main idea of VFAL is to leverage AL in FedAvg \cite{mcmahan2017communication} locally. Each client runs on a local copy of the global model $\theta^t$ with its local data to conduct AL. Then, the server receives updated model parameters $\{\theta_{i,K}^{t+1}\}_{i=1}^m$ for all clients and performs the following aggregation: $\theta^{t+1}=\frac{1}{m}\sum_{i=1}^m\theta_{i,K}^{t+1}$, where $K$ is the epoch of local training. The parameters $\theta^{t+1}$ for the global model are then sent back to each client for another epoch of training. Next, we use the on-average bound to derive the generalization error in the VFAL algorithm through the following theorem.
\begin{Theorem}\label{Theorem 1}
    If a VFAL algorithm $\mathcal{A}$ is $\epsilon$-on-averagely stable, we can obtain the generalization error $\varepsilon_{gen}(\mathcal{A})$ as follows:
    $$\mathbb{E}_{\mathcal{S,A}}\left[\frac{1}{m}\sum_{i=1}^{m}(R_i(\mathcal{A}(\mathcal{S}))-R_{\mathcal{S}_i}(\mathcal{A}(\mathcal{S})))\right]\leq \epsilon.$$
\end{Theorem}

Most existing studies analyze the generalization error of algorithm stability under smooth loss functions \cite{sun2023understanding, huang2023understanding, lei2023stability}, exploring the dependence of FL generalization properties on heterogeneity. However, \cite{liu2020loss} suggest that the adversarial loss $\ell_{\rho}(\theta,z)$ remains non-smooth, even if we assume the standard loss $\ell(\theta,z)$ is smooth. This non-smoothness violates some basic properties in stability analysis \cite{hardt2016train, lei2023stability}, bringing additional challenges.

To address this issue, a natural approach is to use smoothness approximation techniques, such as (1) \textit{Surrogate Smoothness Approximation (SSA)} \cite{cui2021addressing}, (2) \textit{Randomized Smoothness Approximation (RSA)} \cite{alashqar2023gradient}, and (3) \textit{Over-Parameterized Smoothness Approximation (OPSA)} \cite{li2023federated, li2022convergence}. Beyond data heterogeneity in FAL, non-smoothness in approximation can affect the generalization error, further complicating the problem within the FAL framework. Therefore, we aim to investigate the relationship between generalization error and these three smoothness approximation methods. This will help in clearly understanding the problem and in designing more efficient algorithms for FAL.

\subsection{Surrogate Smoothness Approximation}
Using the surrogate smoothness helps improve the quality of the gradient of the original function $\ell_\rho$, potentially improving generalization \cite{xie2020smooth}. In particular, we reconsider the following surrogate loss to substitute $\ell_{\rho}$ in \eqref{emrisk}:
\begin{equation}
    h(\theta;z_{i,j})=\max_{\|z_{i,j}-z_{i,j}^\prime\|_p\leq\rho}\ell (\theta;z_{i,j}^\prime),
\end{equation}
and each client performs $\theta_{i,k+1}^t = \theta_{i,k}^t -\eta_t \nabla h(\theta_{i,k}^t ,z_{i,j})$, where $k\in [K]$ and $\eta_t$ is the local stepsize at the global epoch $t$. Based on the surrogate loss $h$, we have a set of properties of SSA, which can be defined as follows:
\begin{Definition}\label{Definition 3}
    Let $\beta \geq 0,\xi \geq 0$ and $h(\theta)$ be a differentiable function. We say $h(\theta)$ is $\xi$-approximately $\beta$-gradient Lipschitz, if $\forall~\theta_1$ and $\theta_2$, we have
    $$ \left\|\nabla h\left(\theta_{1}\right)-\nabla h\left(\theta_{2}\right)\right\| \leq \beta\left\|\theta_{1}-\theta_{2}\right\|+\xi,~~\xi=2\rho z.$$
\end{Definition}
From Definition~\ref{Definition 3}, we can see that the SSA method dynamically inherits the non-smooth properties by AL. If $\rho=0$, $h$ is gradient Lipschitz; otherwise, $\rho > 0$, $h$ is a general non-smooth function. In particular, the maximization operation of the surrogate function $h$ improves the continuity of the gradient and helps smooth out potentially sharp gradients.
\begin{Theorem}\label{Theorem 2}
    Let the step size be chosen as $\eta_{t} \leq \frac{1}{\beta K(t+1)}$. Under Assumption~\ref{Assumption 1}, the generalization bound $\varepsilon_{gen}$ with the SSA method satisfies:
    \begin{equation*}
           \mathcal{O}\left(\rho T\log T+\frac{T\sqrt{\log T \Delta}}{mn_{\mathrm{min}}}+\frac{T\log T(\rho+1)D_{\mathrm{max}}}{mn_{\mathrm{min}}}\right),
    \end{equation*}
    where $\Delta = \mathbb{E}[R(\theta_0)] - \mathbb{E}[R(\theta^*)]$ and $n_{\min} = \min_{i \in [m]} n_i$.
\end{Theorem}

\textbf{Remark 2.} As shown in Theorem \ref{Theorem 2}, the first term, $\rho T\log T$, represents the approximation error from SSA, which is affected by $\rho$ of AL. We can see that a smaller $\rho$ helps reduce the error in this dominant term. The second term $\frac{T\sqrt{\log T \Delta}}{mn_{\min}}$ pertains to the algorithm's convergence performance. However, in practice, the third term, $\frac{T\log T (\rho +1)D_{\max}}{mn_{\min}}$, arises from the heterogeneity, which quantifies the variation between the original and adversarial data. This term reveals the increased heterogeneity leads to poorer robust generalization ability. When $\rho = 0$, the first term disappears, while the third still reflects the original heterogeneity, similar to FL \cite{sun2023understanding}. 

\subsection{Randomized Smoothness Approximation}
In RSA, we consider a smoothed alternative version of $\ell_{\rho}$ in \eqref{emrisk} based on the randomized smoothing technique \cite{duchi2012randomized, lin2022gradient}. Specifically, let $\gamma\geq 0, \ell_{\gamma}(\theta,z_{i,j}) = \mathbb{E}_{u\sim\mathbb{P}^d}[\ell_{\rho}(\theta+\gamma u_{i,j},z_{i,j})]$, where $\gamma$ is a smoothing parameter, $\mathbb{P}$ is an uniform distribution on a unit ball in $d$-dimensional space with the $\mathcal{L}_2$-norm, and $u_{i,j}\sim \mathbb{P}^d$ is a random variable generated from $\mathbb{P}^d$. By adding multiple instances of randomized noise to the parameters, we aim to steer the new objective function, $\ell_\gamma$, towards a flatter weight loss landscape compared to the original objective $\ell_\rho$, which can moderate the non-smooth property \cite{neyshabur2017exploring, kanai2023relationship}. 

In practice, the gradient of $\ell_{\gamma}(\theta)$ is difficult to compute due to the expectation. Therefore, we can estimate it using Markov chain Monte-Carlo techniques \cite{vrugt2009accelerating} on each client $i$ as follows: 
\begin{equation}
    \theta_{i,k+1}^t = \theta_{i,k}^t -\frac{ \eta_t }{Q}\sum_{q=1}^Q \nabla \ell_{\rho}(\theta_{i,k}^t+\gamma u_{i,k_q},z_{i,k}).
\end{equation}
Note that in each local epoch $k \in [K]$, $u_{i,k_q}$ is sampled from $\mathbb{P}^d$ locally, and we need to repeat the number of $Q$ times samples to reduce the estimated error between real gradient and expected gradient \cite{duchi2012randomized, lin2022gradient}.

\begin{Theorem} \label{Theorem 3}Let $c \geq 0$ be a constant and the step size be chosen as $\eta_t \leq \frac{\gamma}{4K\sqrt{d}cL(t+1)}$. Under Assumption \ref{Assumption 1}, the generalization bound $\varepsilon_{gen}$ with the RSA method satisfies: 
    \begin{equation*}
        \mathcal{O}\left(\frac{T^{1/4}\log T}{\sqrt{Q}}+\frac{T^{3/4}\sqrt{\Delta}}{mn_{\mathrm{min}}}+\frac{T((\rho+1)D_{\mathrm{max}})^{1/3}}{mn_{\mathrm{min}}}\right),
    \end{equation*}
     where $\Delta = \mathbb{E}[R(\theta_0)] - \mathbb{E}[R(\theta^*)]$ and $n_{\min} = \min_{i \in [m]} n_i$.
\end{Theorem}

\textbf{Remark 3.} Similar to Theorem~\ref{Theorem 2}, the generalization error bound of RSA is also composed of three terms. We can see the first term, i.e., $\frac{T^{\frac{1}{4}}\log T}{\sqrt{Q}}$, where $\frac{1}{\sqrt{Q}}$ represents the discrepancy in the gradient estimate between real gradient and expected gradient, is independent on $\rho$. Note that this term is smaller as $Q$ increases, but a larger $Q$ imposes an unbearable computational cost on clients. Hence, the adjustment of $Q$ should be appropriately in a small range. In addition, since each client optimizes a new smoothed objective function, $\ell_\gamma$, the second term $\frac{T^{\frac{3}{4}}\sqrt{\Delta}}{mn_{\mathrm{min}}}$ and the third term $\frac{T((\rho+1)D_{\mathrm{max}})^{\frac{1}{3}}}{mn_{\mathrm{min}}}$ are generated based on a flatter weight space compared to the original. This is why they are independent of $Q$ and the order of $T$ is smaller.

\subsection{Over-parameterized Smoothness Approximation}
In OPSA \cite{lei2022stability}, we focus on the following shallow neural network of the form in \eqref{Eq:loaclrisk}:
\begin{equation}\label{Eq:shallowNN}
    f_W(x):=\sum_{\tau=1}^s\mu_\tau\varphi(\langle w_\tau,x\rangle),
\end{equation}
where we fix $\mu_{\tau} \in \{-1/\sqrt{s},1/\sqrt{s}\}, \varphi:\mathbb{R}\mapsto\mathbb{R}$ is an activation function and $W=(w_1,\ldots,w_s)\in\mathbb{R}^{d\times s}$ is the weight matrix. In \eqref{Eq:shallowNN}, $w_\tau$ denotes the weight of the edge connecting the input to the $\tau$-th hidden node, and $\mu_\tau$ is the weight of the edge connecting the $\tau$-th hidden node to the output node. Here $s$ is the number of nodes in the hidden layer, commonly referred as the width of an over-parameterized neural network. Although the network's width is crucial, it does not conflict with the model parameters $\theta$ as a vector, since $\theta$ represents the entire set of model parameters, including the weight matrix $W$ and other parameters. The difference lies only in the form of parameter representation. Thus, we have $f_\theta=f_W$ and $\ell_\rho(\theta,z)=\ell_\rho(W,z)$, demonstrating that both notations fundamentally describe the same model parameterization. Following the studies \cite{richards2021stability,lei2022stability, wang2023generalization}, we state two standard assumptions as follows:
\begin{Assumption}
    (Activation). The activation $\varphi(\cdot)$ is continuous and twice differentiable with constant $B_\varphi,B_{\varphi^{\prime}},B_{\varphi^{\prime\prime}}\geq0$ bounding $|\varphi(\cdot)|\leq B_{\varphi},|\varphi^{\prime}(\cdot)|\leq B_{\varphi^{\prime}}$ and $|\varphi''(\cdot)|\leq B_{\varphi''}$.
\end{Assumption}
\begin{Assumption}\label{Assumption 3}
    (Inputs, labels, parameters, and the loss function). For constants $C_{x},C_y,C_{W},C_0>0$, inputs belong to $\mathcal{B}_2^d(C_{x})$, labels belong to $[-C_y,C_y]$, the weight matrix $W$ are confined within the bounded domain $\mathcal{B}_2^d(C_{W})$, and the loss is uniformly bounded by $C_0$ almost surely.
\end{Assumption}
In over-parameterized network, we can bound the Hessian scales of $\ell_{\rho}$, i.e., $\|\nabla^2 \ell_{\rho}\|\leq \zeta_\theta$, by using the above assumptions to ensure its $\zeta_\theta$-smoothness, where $\zeta_\theta=2(C_x^2+\rho^2)\big(B_{\varphi^{\prime}}^2+B_{\varphi^{\prime\prime}}B_\varphi+\frac{B_{\varphi^{\prime\prime}}C_y}{\sqrt s}\big)$ \cite{lei2022stability,wang2023generalization,sitawarin2021sat}. 
\begin{Theorem} \label{Theorem 4}
    Without loss of generality, we assume $4K\eta_t C_0 \geq 1$ and $s\geq 16\eta_t^2T^2K^2(b'H_K)^2(1+2\eta_t\zeta_\theta)^2$,where $b'= C_x^2B_{\varphi''}(C_xB_{\varphi'}+\sqrt{2C_0}),H_K=2\sqrt{K\eta_t C_0}$. Let the step size be chosen as $\eta_t\leq \frac{1}{6K\zeta_\theta}$. Under Assumptions \ref{Assumption 1}-\ref{Assumption 3}, the generalization bound $\varepsilon_{gen}$ with the OPSA method satisfies:
    \begin{equation*}
        \mathcal{O}\left(\frac{T^{\frac{1}{2}}\sqrt{\Delta}}{mn_{\mathrm{min}}}+\frac{T(\rho^2\sqrt{s}+1)D_{\mathrm{max}}}{mn_{\mathrm{min}}}\right),
    \end{equation*}
     where $\Delta = \mathbb{E}[R(\theta_0)] - \mathbb{E}[R(\theta^*)]$ and $n_{\min} = \min_{i \in [m]} n_i$.
\end{Theorem}
\textbf{Remark 4}. In Theorem~\ref{Theorem 4}, the approximation error appearing in Theorems \ref{Theorem 2}-\ref{Theorem 3} is eliminated due to the smoothness of $\ell_{\rho}$ under over-parameterization. In addition, the first term, $\frac{T^{\frac{1}{2}}\sqrt{\Delta}}{mn_{\mathrm{min}}}$, is smaller because over-parameterization helps the model avoid the complex non-convex area during the optimization process, allowing algorithms like gradient descent to find the optimal solutions \cite{arora2018optimization} more quickly. However, hidden behind this benefit is the curse of width exacerbation, denoted as $\mathcal{O}(\rho^2\sqrt{s})$. This is similar to the findings in \cite{wu2021wider, hassani2024curse}, where the curse of width $\mathcal{O}(\rho\sqrt{s})$ was observed to impact perturbation stability in AL. Therefore, it is necessary to control for the exacerbation of heterogeneity by the width of the network.

\textbf{Remark 5.} Based on the results in Theorems~\ref{Theorem 2}-\ref{Theorem 4}, RSA is the most effective method for reducing the generalization error by constructing a higher-quality mediator function. In additional, by mitigating the attack strength of $\rho$, we can improve the dominate term $\mathcal{O}(\rho T\log T)$ in SSA. The results of OPSA consist of two terms, but it has the largest generalization error due to the presence of $\sqrt{s}$, which leads to the highest order term of $\mathcal{O}(T^2)$. Note that all three theorems include an additional term related to the number of $K$ local training epochs. However, due to the lower order, e.g., $\frac{T\sqrt{\log T}}{mn_{\min}K^{1/2}}$ of SSA in Theorem~\ref{Theorem 2}, we dismiss these terms as they do not have significant impacts.

\section{Slack FAL Algorithm}
From the above results, we can see how to set smoothness approximation to reduce the generalization error in degrading VFAL algorithm \cite{zizzo2020fat, shah2021adversarial, hong2021federated}. Instead of the VFAL algorithm, \cite{zhu2023combating} proposed a simple but effective SFAL algorithm based on an $\alpha$-slack decomposed mechanism. 

\subsection{$\alpha$-Slack Decomposed Mechanism}
In the VFAL algorithm, the simple averaging aggregation strategy often causes the global model to be biased towards clients with significant data heterogeneity. In contrast, the SFAL algorithm effectively corrects this bias by dynamically reweighting the contributions of clients through the $\alpha$-slack decomposed mechanism. In particular, the $\alpha$-slack decomposed mechanism uses the local adversarial loss $R_{\mathcal{S}_{i}}(\theta)$ as an AL-related metric to identify more heterogeneous clients and weaken their weights, thereby emphasizing the importance of those less heterogeneous clients.

\begin{table*}[t!]
\centering
\begin{tabular}{c|c|c}
\toprule
App. Method               & VFAL & SFAL \\ \midrule
SSA    & $\mathcal{O}(\rho T\log T+\frac{T\sqrt{\log T \Delta}}{mn_{\mathrm{min}}}+\frac{T\log T(\rho+1)D_{\mathrm{max}}}{mn_{\mathrm{min}}})$   &$\mathcal{O}(\rho T\log T+\frac{T\sqrt{\log T \Delta}}{r_\alpha mn_{\mathrm{min}}}+\frac{T\log T(\rho+1)D_{\mathrm{max}}}{r_\alpha mn_{\mathrm{min}}})$\\ \hline
RSA & $\mathcal{O}(\frac{T^{\frac{1}{4}}\log T}{\sqrt{Q}}+\frac{T^{\frac{3}{4}}\sqrt{\Delta}}{mn_{\mathrm{min}}}+\frac{T((\rho+1)D_{\mathrm{max}})^{\frac{1}{3}}}{mn_{\mathrm{min}}})$ & $\mathcal{O}(\frac{T^{\frac{1}{4}}\log T}{\sqrt{Q}}+\frac{T^{\frac{3}{4}}\sqrt{\Delta}}{r_\alpha mn_{\mathrm{min}}}+\frac{T((\rho+1)D_{\mathrm{max}})^{\frac{1}{3}}}{r_\alpha mn_{\mathrm{min}}})$\\ \hline
OPSA  & $\mathcal{O}(\frac{T^{\frac{1}{2}}\sqrt{\Delta}}{mn_{\mathrm{min}}}+\frac{T(\rho^2\sqrt{s}+1)D_{\mathrm{max}}}{mn_{\mathrm{min}}})$  &  $\mathcal{O}(\frac{T^{\frac{1}{2}}\sqrt{\Delta}}{r_\alpha mn_{\mathrm{min}}}+\frac{T(\rho^2\sqrt{s}+1)D_{\mathrm{max}}}{r_\alpha mn_{\mathrm{min}}})$                                    \\ \bottomrule
\end{tabular}
\caption{Main theoretical results.}
\label{Tab:MainResults}
\end{table*}

Given $\alpha\in[0,1), \hat{m} \leq \frac{m}{2}$ with the local empirical risk sorted by $\{ R_{\mathcal{S}_{i}}(\theta)\}$ in ascending order, we have $\sum_{i=1}^{\hat{m}}R^{\phi_{(i)}}_{\mathcal{S}_{i}}(\theta) \leq \sum_{i=\hat{m}+1}^{m} R^{\phi_{(i)}}_{\mathcal{S}_{i}}(\theta)$, where $\phi_{(\cdot)}$ is a function which maps the index to original empirical risk. Then we modified the original global objective \eqref{emrisk} as follows:
    \begin{equation}
    \begin{split}
    R_{\mathcal{S}}(\theta) = \frac{1}{\tilde{m}} \bigg[(1+\alpha)&\sum_{i=1}^{\hat{m}}R_{\mathcal{S}_i}^{\phi_{(i)}}(\theta)\\
    & +(1-\alpha)\sum_{i=\hat{m}+1}^mR_{\mathcal{S}_i}^{\phi_{(i)}}(\theta)\bigg],
    \end{split}
    \end{equation}
where $\tilde{m}=(1+\alpha)\hat{m} + (1-\alpha)(m-\hat{m})$. Note that $\tilde{m}$ is introduced to ensure that the weights of each local risk are normalized in the global aggregation. To flexibly emphasize the importance of partial populations, we have 
\begin{equation}\label{alpha}
    \alpha = 1-\frac{\sum_{i=1}^{\hat{m}}R^{\phi_{(i)}}_{\mathcal{S}_{i}}(\theta)}{\sum_{i=\hat{m}+1}^{m} R^{\phi_{(i)}}_{\mathcal{S}_{i}}(\theta)}.
\end{equation}
In particular, if partial populations $\sum_{i=1}^{\hat{m}}R_{\mathcal{S}_i}^{\phi_{(i)}}(\theta)$ is smaller, we need a larger $\alpha$ to enhance their contribution in aggregation. At the high level, we assign the client-wise slack during aggregation to upweight the clients having the small AL loss (simultaneously downweight those with large loss). Moreover, we would like to emphasize that this $\alpha$-slack mechanism does not affect the result of Theorem \ref{Theorem 1}, which is the key theorem in our proof of the generalization bound using on-average stability.
\begin{Theorem}\label{Theorem 5}
    If a SFAL algorithm $\mathcal{A}$ is $\epsilon$-on-averagely stable, we can obtain the generalization error $\varepsilon_{gen}(\mathcal{A})$ as follows: 
    \begin{equation}
    \begin{split}
        & \mathbb{E}_{\mathcal{S,A}} \left[\frac{1+\alpha}{\tilde{m}}\sum_{i=1}^{\hat{m}}(R^{\phi_{(i)}}_i(\mathcal{A}(\mathcal{S})) -R^{\phi_{(i)}}_{\mathcal{S}_i}(\mathcal{A}(\mathcal{S})))\right] +\\
        & \mathbb{E}_{\mathcal{S,A}} \left[\frac{1-\alpha}{\tilde{m}}\sum_{i=\hat{m}+1}^{m}(R^{\phi_{(i)}}_i(\mathcal{A}(\mathcal{S}))-R^{\phi_{(i)}}_{\mathcal{S}_i}(\mathcal{A}(\mathcal{S})))\right] \leq \epsilon.\nonumber
    \end{split}
    \end{equation}  
\end{Theorem}
\textbf{Remark 6.} Theorem \ref{Theorem 5} implying that under SFAL, analyzing generalization using the algorithmic stability is more challenging. Given the dynamic reweighting nature of the $\alpha$-slack mechanism, we need to consider two scenarios in which the client $i$, holding the perturbed sample $z_{i',j}'$ where defined in Definition \ref{Definition 1}, experiences dynamic reweighting across different global aggregation epochs:(1) upweighting, when $i\in [1,\hat{m}]$, and (2) downweighting, when $i\in [\hat{m}+1, m]$.

\subsection{Generalization bound of SFAL}
Here, we present the generalization bound of the three smoothness approximation methods in SFAL.
\begin{Theorem} \label{Theorem 6}
    Given $\alpha \in [0,1),\hat{m} \leq \frac{m}{2}, r_\alpha = 1+ \frac{\alpha}{1-\alpha}\frac{2\hat{m}}{m}$, we can obtain the following results in SFAL: 
    
\noindent 1. Under Assumption~\ref{Assumption 1}, let the step size be chosen as $\eta_{t} \leq \frac{\tilde{m}}{m\beta K(t+1)}$, the generalization bound $\varepsilon_{gen}$ with the SSA method satisfies:
        \begin{equation*}
            \mathcal{O}\left(\rho T\log T+\frac{T\sqrt{\log T \Delta}}{r_\alpha mn_{\mathrm{min}}}+\frac{T\log T(\rho+1)D_{\mathrm{max}}}{r_\alpha mn_{\mathrm{min}}}\right).
        \end{equation*}
\noindent 2. Under Assumption~\ref{Assumption 1}, let $c\geq 0$ be a constant and the step size be chosen as $\eta_t\leq \frac{\tilde{m}\gamma}{4mK\sqrt{d}cL(t+1)}$. Under Assumption \ref{Assumption 1}, the generalization bound $\varepsilon_{gen}$ with the RSA method satisfies: 
        \begin{equation*}
            \mathcal{O}\left(\frac{T^{\frac{1}{4}}\log T}{\sqrt{Q}}+\frac{T^{\frac{3}{4}}\sqrt{\Delta}}{r_\alpha mn_{\mathrm{min}}}+\frac{T((\rho+1)D_{\mathrm{max}})^{\frac{1}{3}}}{r_\alpha mn_{\mathrm{min}}}\right).
        \end{equation*}
\noindent 3. Without loss of generality, we assume $4K\eta_t C_0 \geq 1$ and $s\geq 16\eta_t^2T^2K^2(b'H_K)^2(1+2\eta_t\zeta_\theta)^2$,where $b'= C_x^2B_{\varphi''}(C_xB_{\varphi'}+\sqrt{2C_0}),H_K=2\sqrt{K\eta_t C_0}$. Let the step size be chosen as $\eta_t\leq \frac{\tilde{m}}{6mK\zeta_\theta}$. Under Assumptions \ref{Assumption 1}-\ref{Assumption 3}, the generalization bound $\varepsilon_{gen}$ with the OPSA method satisfies:
        \begin{equation*}
            \mathcal{O}\left(\frac{T^{\frac{1}{2}}\sqrt{\Delta}}{r_\alpha mn_{\mathrm{min}}}+\frac{T(\rho^2\sqrt{s}+1)D_{\mathrm{max}}}{r_\alpha mn_{\mathrm{min}}}\right),
        \end{equation*}
        where $\Delta = \mathbb{E}[R(\theta_0)] - \mathbb{E}[R(\theta^*)]$ and $n_{\min} = \min_{i \in [m]} n_i$.
\end{Theorem}

\textbf{Remark 7.} In Theorem 6, as $\alpha \to 1$, $r_\alpha$ also becomes larger due to the corresponding increase in $\frac{\alpha}{1-\alpha}$, effectively mitigating the generalization error. When $\alpha = 0$, it yields the result of the VFAL algorithm. Moreover, \eqref{alpha} implies that there is a balance between $\alpha$ and $\hat{m}$. When $\hat{m}$ increase, $\sum_{i=1}^{\hat{m}}R^{\phi_{(i)}}_{\mathcal{S}_{i}}(\theta)$ also increases, which means we need to decrease  $\alpha$ to avoid overemphasizing the importance of partial populations. Combining the results of VFAL and SFAL, we obtain Table \ref{Tab:MainResults}. Note that the $\alpha$-slack mechanism operates on the server side, while the three smoothness approximation methods operate on the client side. We observe that both the term including $\Delta$ and the term including $D_{\max}$ in SFAL, which are generated from the aggregation process, are optimized compared to VFAL. Therefore, SFAL enhances the generalization of these methods without altering their strengths and weaknesses.

\textbf{Remark 8.} Combining the analysis of the above results, we gain some insights that help to design more efficient FAL algorithms. Motivated by SFAL, we propose designing a new contrastive loss for weight assignment to measure heterogeneity exacerbation. Based on Lemma \ref{Lemma 1}, the local data drift $D_i=d_{TV}(\tilde{P}_i, P_i)$ should be the dominant term in $D_{\mathrm{max}}$. Therefore, we consider a new dynamic aggregation strategy, i.e., $R_{\mathcal{S}}(\theta) = \sum_{i=1}^m\frac{d_{TV}(\tilde{P}_i, P_i)}{\sum_{i=1}^m{d_{TV}(\tilde{P}_i, P_i)}}R_{\mathcal{S}_i}(\theta)$. Concretely, at the client level, we can approximate $d_{TV}(\tilde{P}_i, P_i)$ in terms of the generated adversarial samples and original samples. We believe this aggregation strategy effectively combines FL and AL. \textcolor{red}{} For local training, we may design a new adversarial penalty, i.e., $\|\theta_{\mathrm{adv}}-\theta_{\mathrm{ori}}\|$, to address heterogeneity $D_{\mathrm{max}}$ caused by AL. Here, $\theta_{\mathrm{adv}}$ is trained on adversarial samples, and $\theta_{\mathrm{ori}}$ is trained on original samples. Compared to SFAL, our proposed algorithms integrate information from the original samples, which may not only improve the generalization ability for adversarial samples but also have the potential to enhance the ability for non-adversarial samples.

\begin{figure*}[t!]
    \includegraphics[width=0.33\linewidth, height=0.23\linewidth]{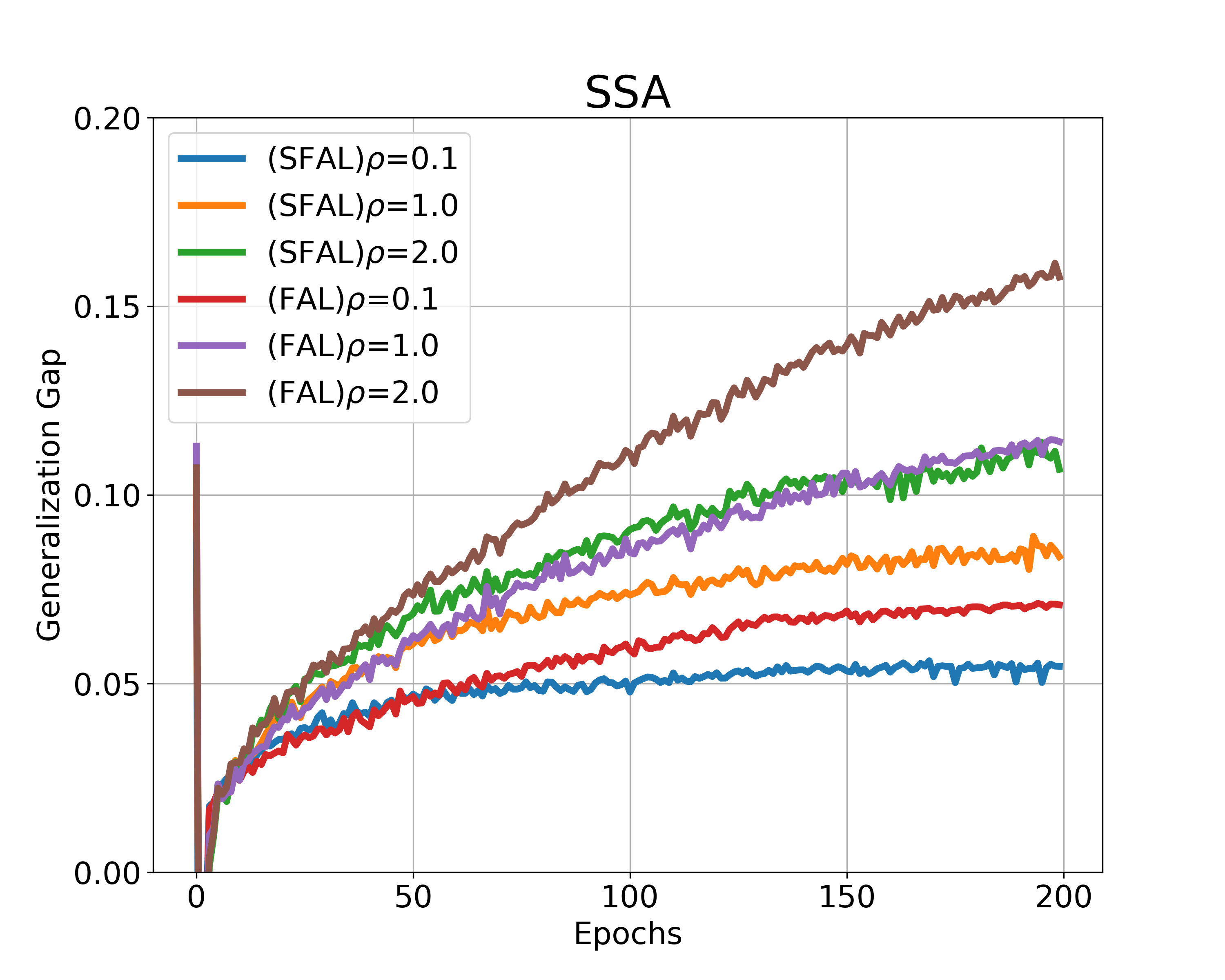}
    \includegraphics[width=0.33\linewidth, height=0.23\linewidth]{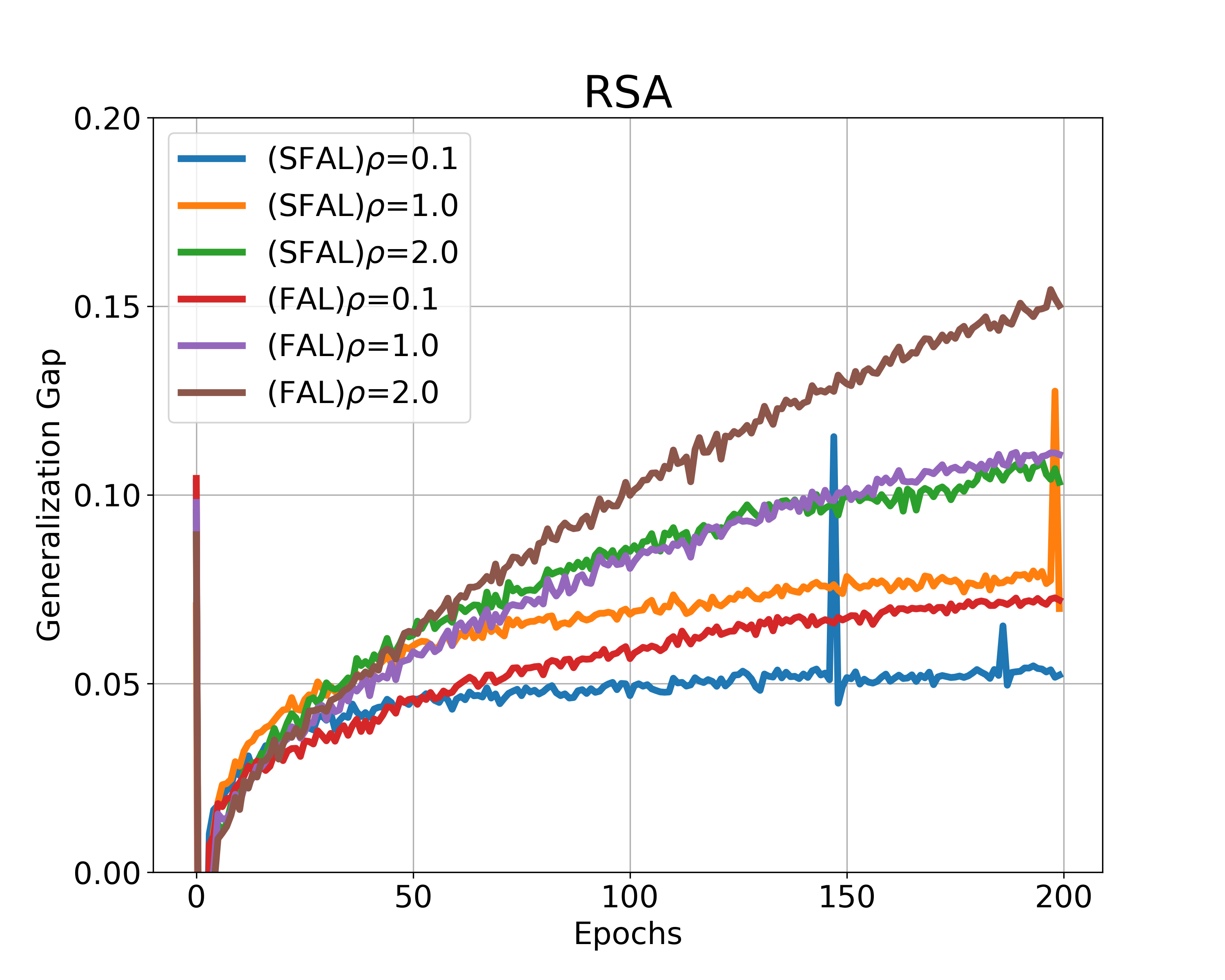}
    \includegraphics[width=0.33\linewidth, height=0.23\linewidth]{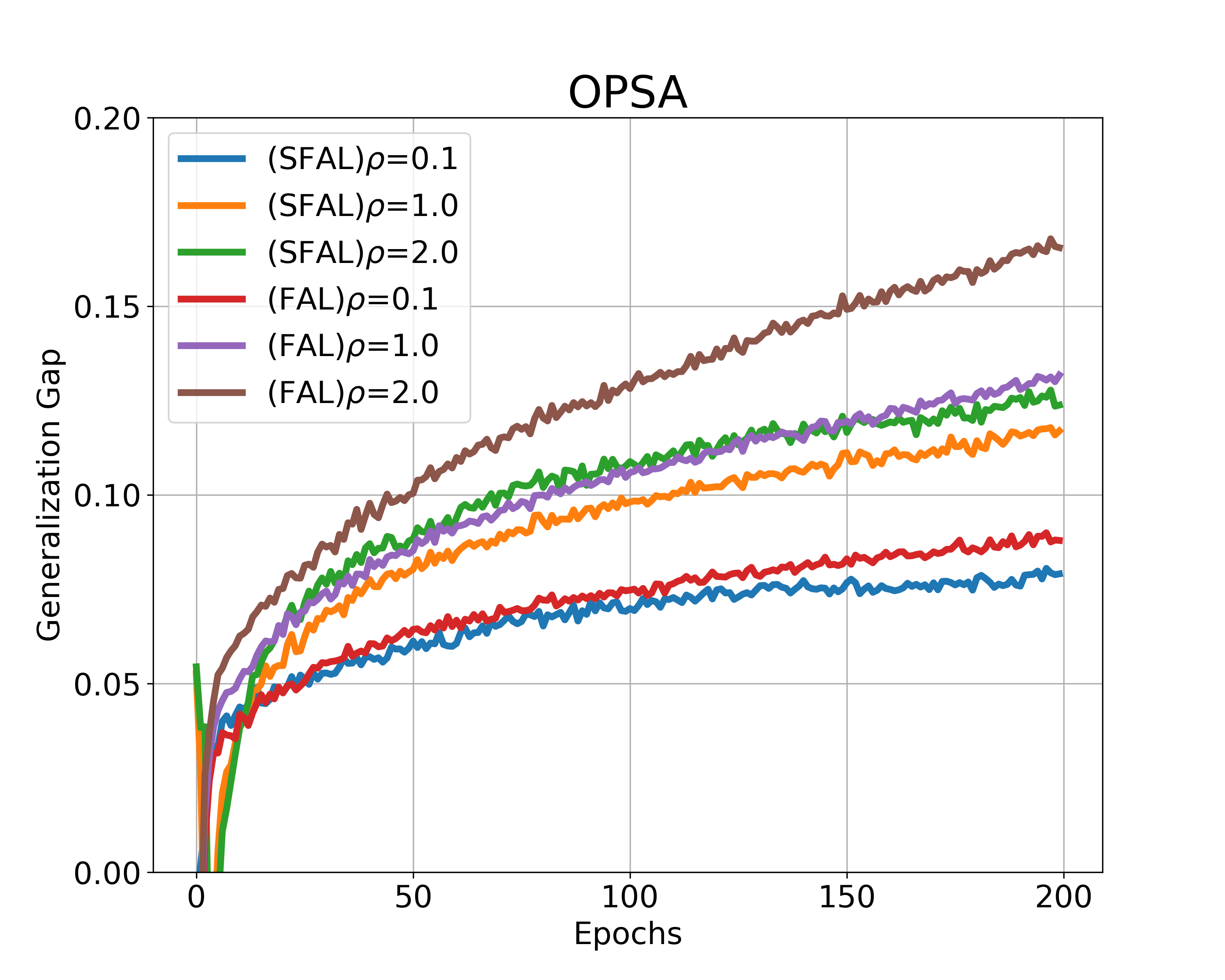}
    \caption{Generalization Gap of the attack strength $\rho$ on SVHN. ($m=40,a=2.0$)}
    \label{svhnfig:rho}
\end{figure*}
\begin{figure*}[t!]
    \includegraphics[width=0.33\linewidth, height=0.23\linewidth]{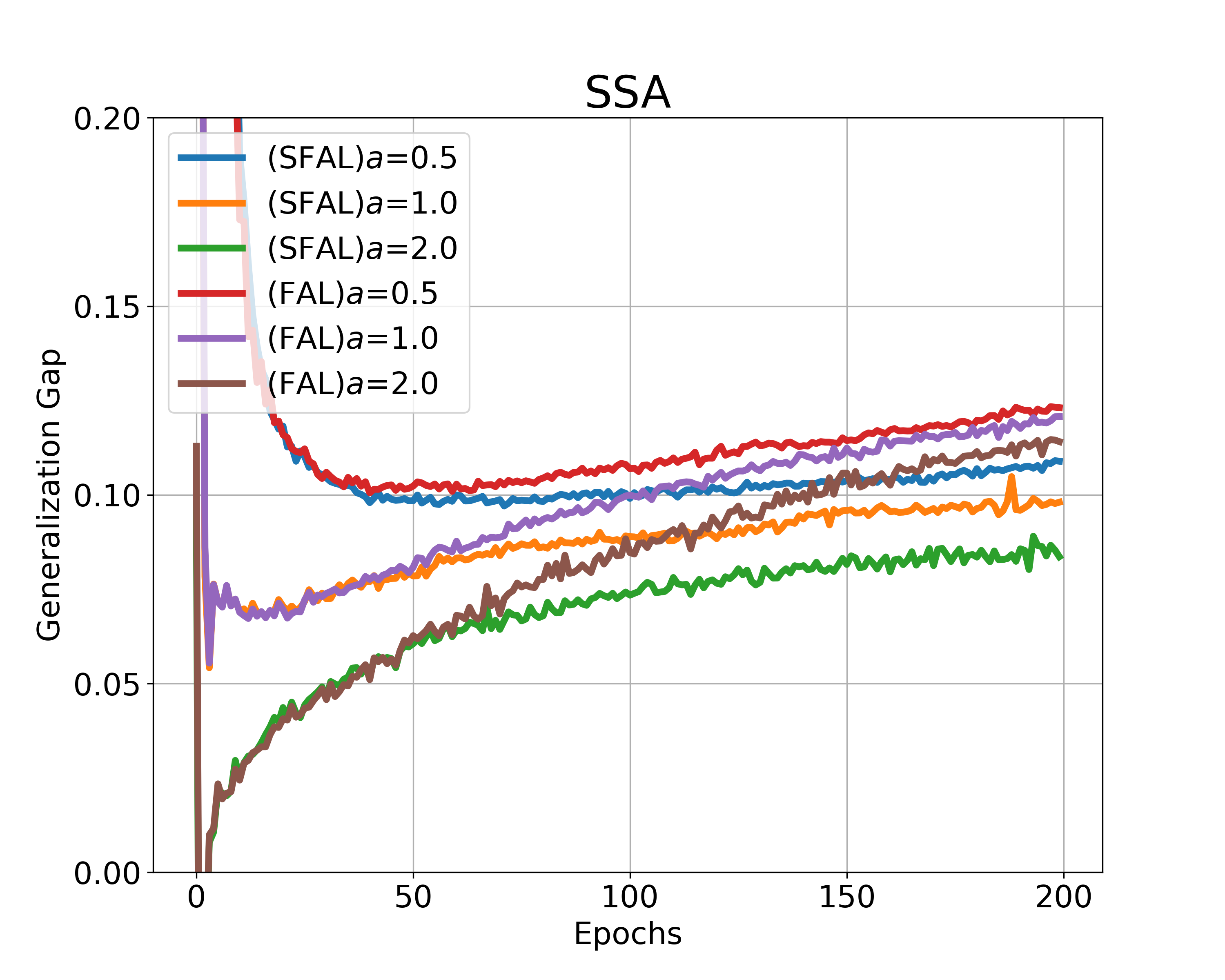}
    \includegraphics[width=0.33\linewidth, height=0.23\linewidth]{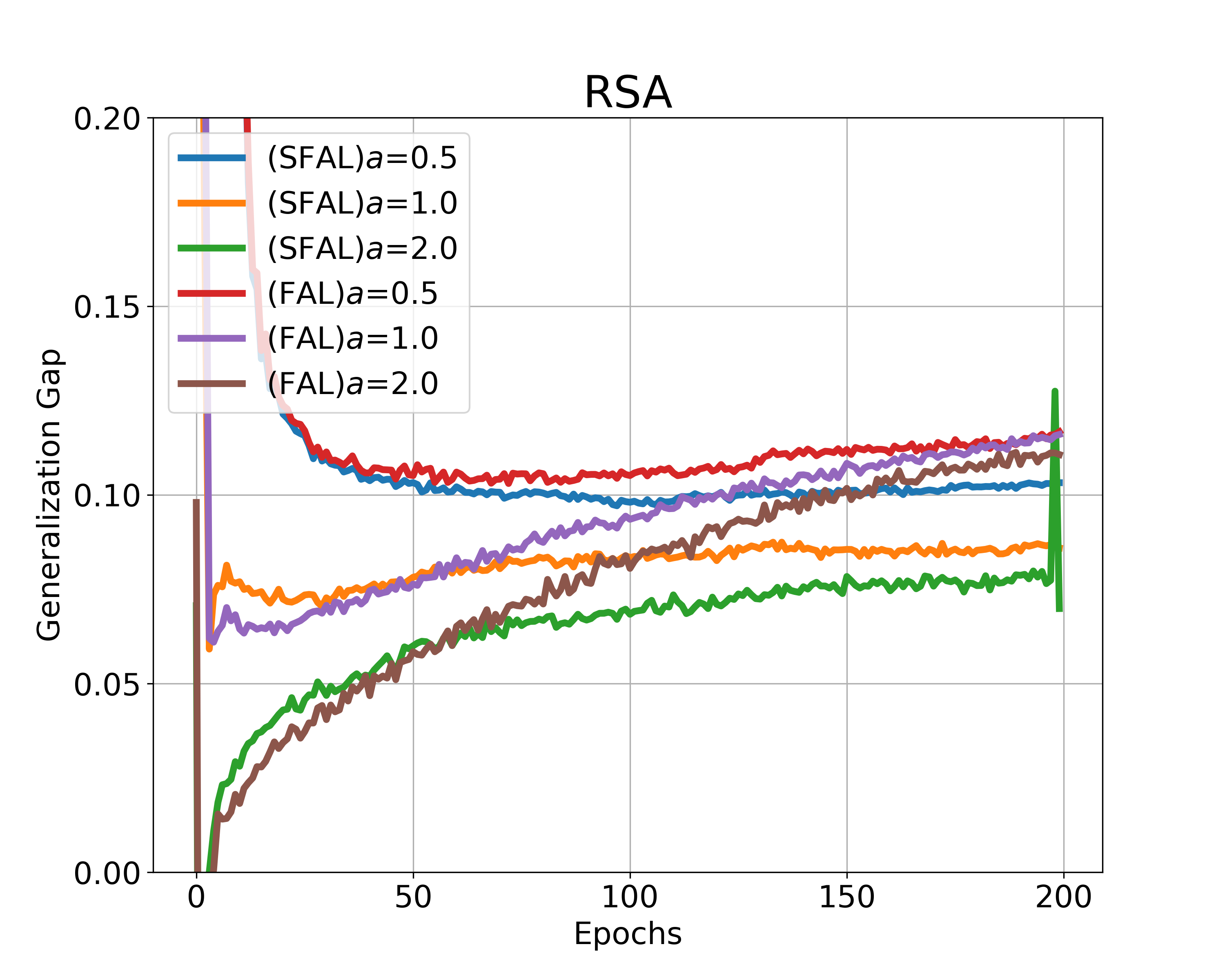}
    \includegraphics[width=0.33\linewidth, height=0.23\linewidth]{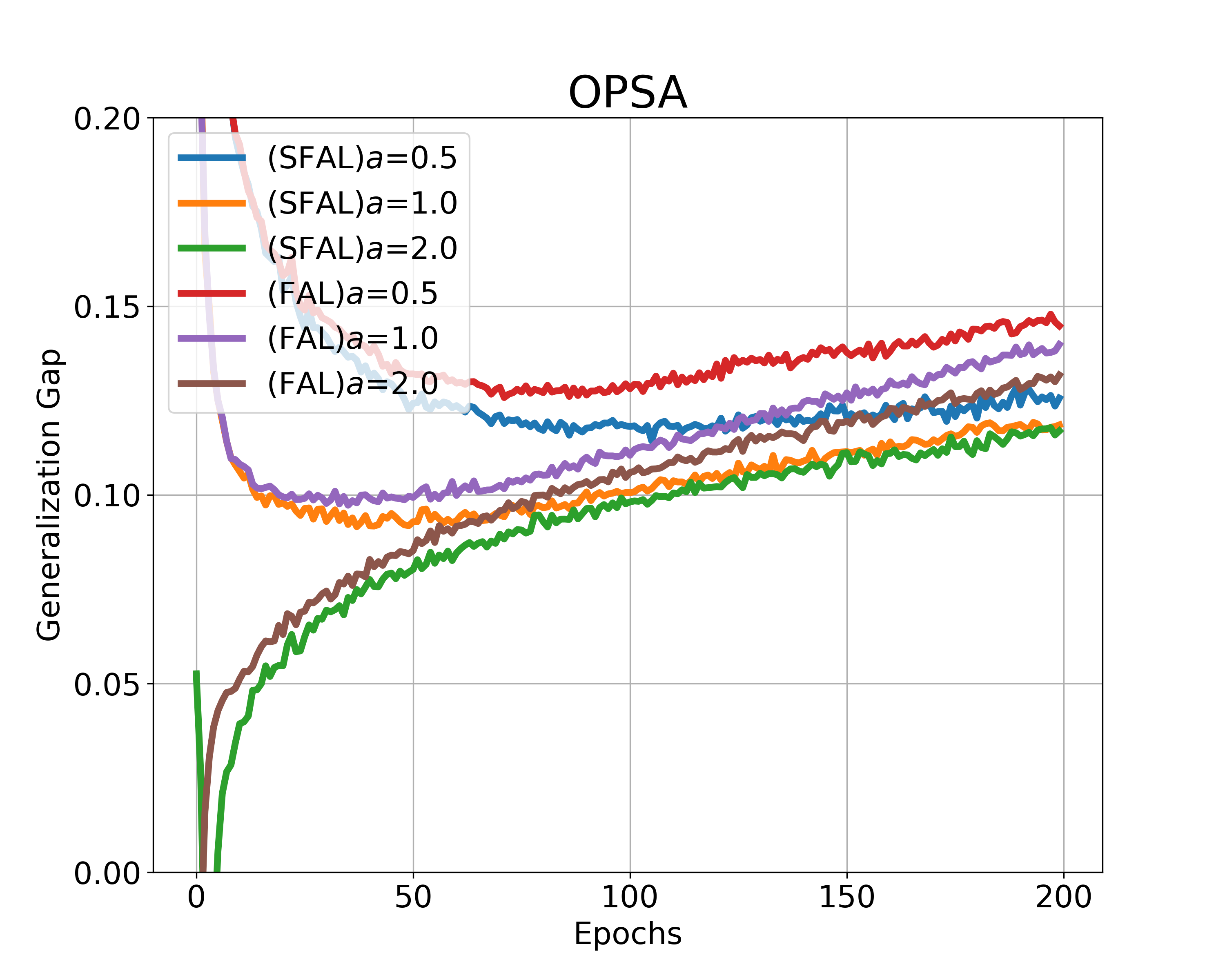}
    \caption{Generalization Gap of the skew parameter $a$ on SVHN. ($m=40,\rho=1.0$)}
    \label{svhnfig:skew}
\end{figure*}
\begin{figure*}[t!]
    \includegraphics[width=0.33\linewidth, height=0.23\linewidth]{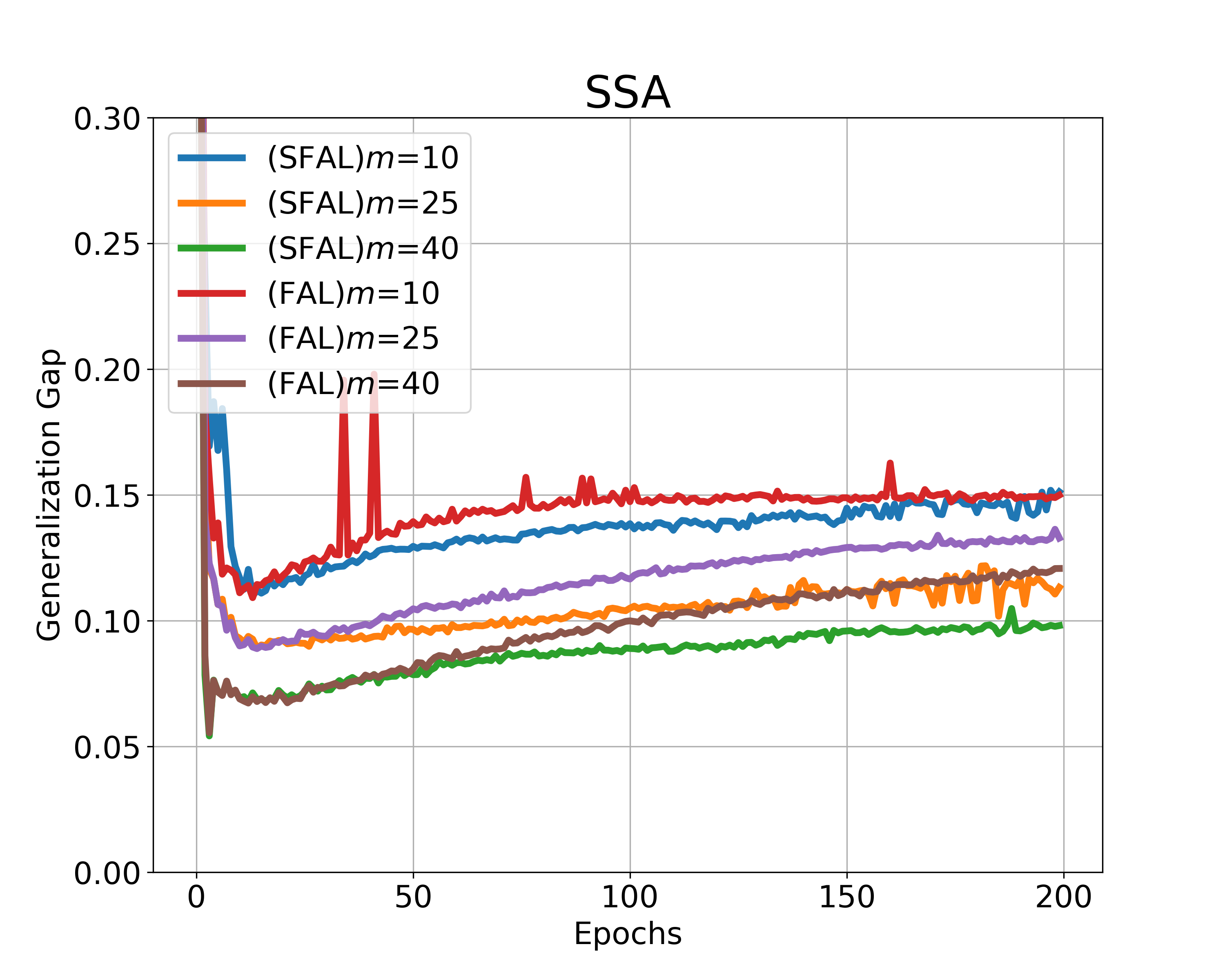}
    \includegraphics[width=0.33\linewidth, height=0.23\linewidth]{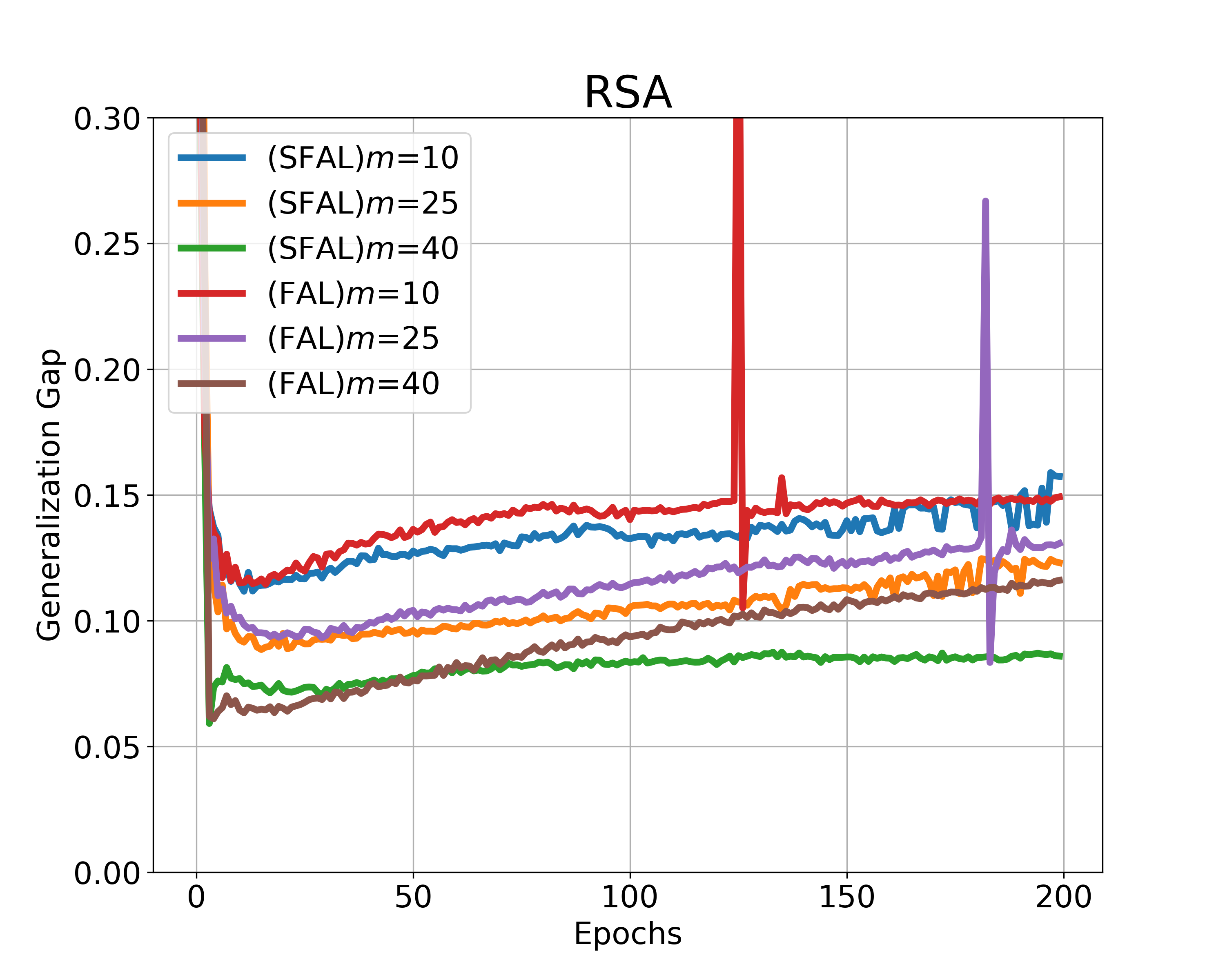}
    \includegraphics[width=0.33\linewidth, height=0.23\linewidth]{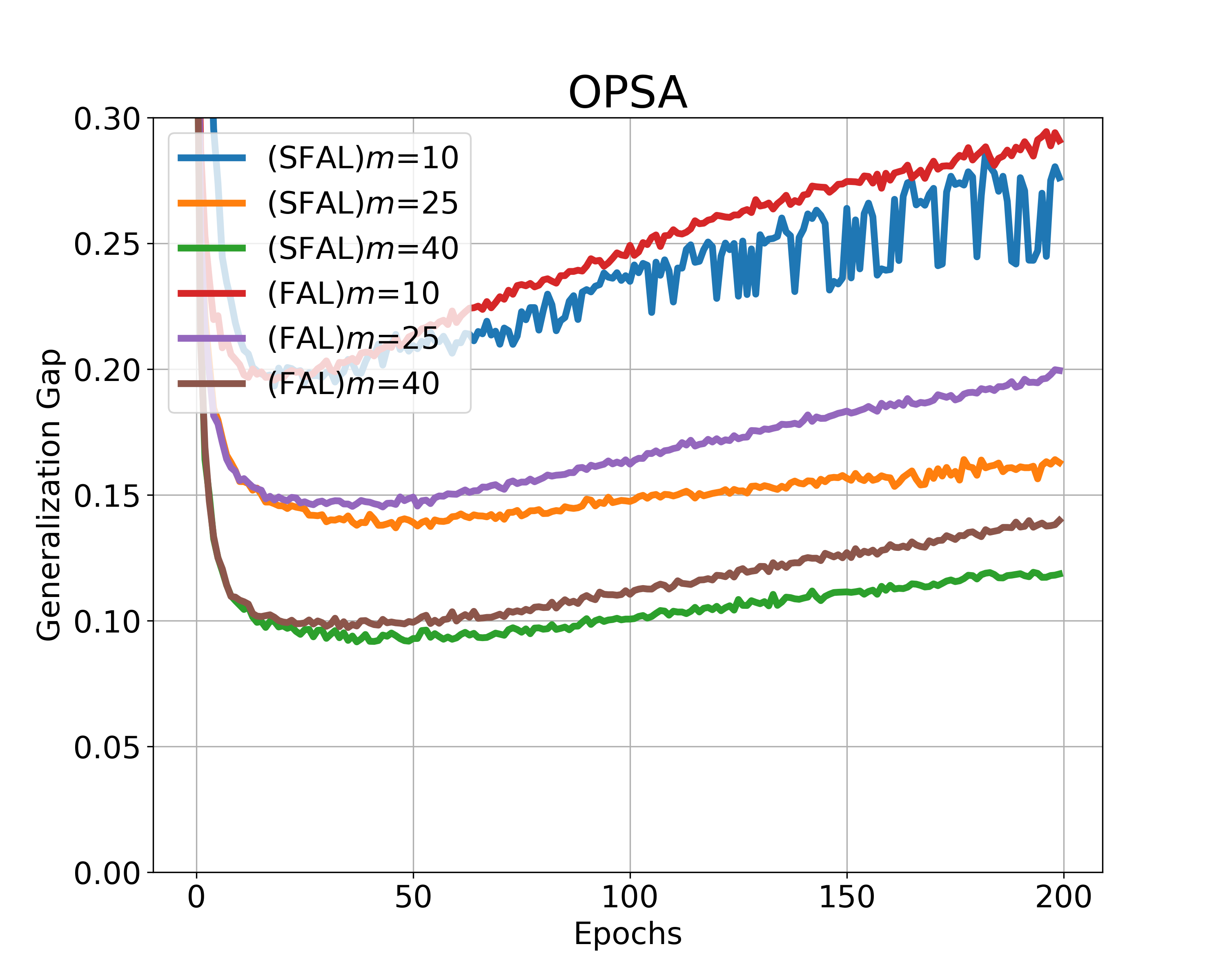}
    \caption{Generalization Gap of the number of client $m$ on SVHN. ($a=1.0,\rho=1.0$)}
    \label{svhnfig:numberclient}
\end{figure*}

\section{Experiments}
\textbf{Setups.} Here, we conduct the experiments on the SVHN dataset \cite{netzer2011reading} under heterogeneous scenarios. Following \cite{shah2021adversarial}, we partition the dataset based on labels using a skew parameter $a$. This distribution allows each of the $m$ clients to receive samples from certain classes. We define $\mathcal{Y}$ as the set of all classes, equally divided among $m$ clients to form $\mathcal{Y}_\varrho$. Each client thus holds $(100-(m-1)\times a)\%$ of its data from classes in $\mathcal{Y}_\varrho$ and $a\%$ from other classes. The generalization gap is measured as the difference between train and test accuracy. For the inner problems, we adopt the $\ell_\infty$ PGD adversarial training in \cite{madry2017towards}, the step size in the inner maximization is set to be $\rho/4$. In SFAL, we set $\hat{m}=m/5$. More detailed settings and other results are provided in supplementary.

\textbf{Impact on $\rho$.} Comparing different $\rho$ in Figure~\ref{svhnfig:rho}, we can see that the generalization gap increases as $\rho$ gets larger, which indicate that stronger attacks produce greater heterogeneity in our bound, i.e., $\rho D_{\mathrm{max}}$. 

\textbf{Impact on $a$.} Comparing different $a$ in Figure~\ref{svhnfig:skew}, we can find that the enhanced heterogeneity of FL itself similarly increases the generalization gap, which is consistent with Lemma \ref{Lemma 1}, i.e., $D_{i}=d_{TV}(P_i,P)$. 

\textbf{Impact on $m$.} Comparing different $m$ in Figure~\ref{svhnfig:numberclient}, we can infer that increasing the number of clients can effectively reduce the generalization gap since the effect of heterogeneity on generalization can be dispersed.

In summary, RSA achieves the minimum generalization error compared to OPSA and SSA. Additionally, under suitable $\alpha$, SFAL effectively reduces the generalization error compared to VFAL.

\section{Conclusion}
In this paper, using three different smoothness approximation methods: SSA, RSA, and OPSA, we provide generalization upper bounds for two popular FAL algorithms: VFAL and SFAL. Our bounds explicitly explain how to set proper parameters in these approximation methods and which one is more helpful in reducing the generalization error for FAL. In general, the RSA method can achieve the best generalization performance. We also find that SFAL always performs better than SFAL due to its re-weighted aggregation strategy. Based on our analysis, we give some useful insights into designing more efficient and dynamic global aggregation strategies to mitigate heterogeneity under the FAL context.


\bibliography{aaai25}

\onecolumn
\appendix
\section{VFAL Algorithms} 
\begin{algorithm}[htbp]
	\caption{VFAL of SSA}
	\textbf{Input}: initialization $\theta_0$, global communication round $T$, local updating iterations $K$, learning rate $\eta_t$\\
	\textbf{Output} $\theta^T$ given by the server
	\begin{algorithmic}[1] 
		\FOR{$t=0,1...,T-1$}
		\STATE          $\theta^{t+1}_{i,0}=\theta_t, \forall i=1,..,m$
		\FOR{$k=0,1,...,K-1$}
		\STATE              $\theta^{t+1}_{i,k+1}=\theta^{t+1}_{i,k}-\eta_t \nabla h(\theta^{t+1}_{i,k},z_{i,j})$
		\ENDFOR
		\STATE      $\theta_{t+1}=\frac{1}{m}\sum_{i=1}^m \theta^{t+1}_{i,K}$
		\ENDFOR
	\end{algorithmic}
\end{algorithm}
\begin{algorithm}[htbp]
	\caption{VFAL of RSA}
	\textbf{Input}: initialization $\theta_0$, global communication round $T$, local updating iterations $K$, learning rate $\eta_t$, smoothing parameter $\gamma$\\
	\textbf{Output} $\theta^T$ given by the server
	\begin{algorithmic}[1] 
		\FOR{$t=0,1...,T-1$}
		\STATE          $\theta^{t+1}_{i,0}=\theta_t, \forall i=1,..,m$
		\FOR{$k=0,1,...,K-1$}
		\STATE          Sample $u_{i,k_q}\in\mathbb{P}^d$ uniformly from a unit sphere in $\mathbb{P}^d$.
		\STATE          $\theta_{i,k+1}^t = \theta_{i,k}^t -\frac{ \eta_t }{Q}\sum_{q=1}^Q \nabla \ell_{\rho}(\theta_{i,k}^t+\gamma u_{i,k_q},z_{i,k}).$
		\ENDFOR
		\STATE      $\theta_{t+1}=\frac{1}{m}\sum_{i=1}^m \theta^{t+1}_{i,K}$
		\ENDFOR
	\end{algorithmic}
\end{algorithm}
\begin{algorithm}[htbp]
	\caption{VFAL of OPSA}
	\textbf{Input}: initialization $\theta_0$, global communication round $T$, local updating iterations $K$, learning rate $\eta_t$\\
	\textbf{Output} $\theta^T$ given by the server
	\begin{algorithmic}[1] 
		\FOR{$t=0,1...,T-1$}
		\STATE          $\theta^{t+1}_{i,0}=\theta_t, \forall i=1,..,m$
		\FOR{$k=0,1,...,K-1$}
		\STATE              $\theta^{t+1}_{i,k+1}=\theta^{t+1}_{i,k}-\eta_t \nabla \ell_\rho(\theta^{t+1}_{i,k},z_{i,j})$
		\ENDFOR
		\STATE      $\theta_{t+1}=\frac{1}{m}\sum_{i=1}^m \theta^{t+1}_{i,K}$
		\ENDFOR
	\end{algorithmic}
\end{algorithm}
\newpage
\section{SFAL Algorithms}
\begin{algorithm}[htbp]
	\caption{SFAL of SSA}
	\textbf{Input}: initialization $\theta_0$, global communication round $T$, local updating iterations $K$, learning rate $\eta_t$\\
	\textbf{Output} $\theta^T$ given by the server
	\begin{algorithmic}[1] 
		\FOR{$t=0,1...,T-1$}
		\STATE          $\theta^{t+1}_{i,0}=\theta_t, \forall i=1,..,m$
		\FOR{$k=0,1,...,K-1$}
		\STATE              $\theta^{t+1}_{i,k+1}=\theta^{t+1}_{i,k}-\eta_t \nabla h(\theta^{t+1}_{i,k},z_{i,j})$
		\STATE          Accumulate the loss $\mathcal{L}_{loss}^{i.t}$ in the last local epoch
		\ENDFOR
		\STATE      $\mathcal{L}_{all} \longleftarrow [\mathcal{L}_{loss}^{1,t}, \mathcal{L}_{loss}^{2,t}, \cdots,\mathcal{L}_{loss}^{i,t}], \mathcal{L}_{sorted} \longleftarrow  $ AscendingSort($\mathcal{L}_{all}$)
		\STATE      $\theta_{t+1}=\frac{1+\alpha}{\tilde m}\sum_{i=1}^{\hat{m}} \theta^{t+1}_{i,K} + \frac{1-\alpha}{\tilde m}\sum_{i=\hat{m}+1}^{m} \theta^{t+1}_{i,K}$
		\ENDFOR
	\end{algorithmic}
\end{algorithm}
\begin{algorithm}[htbp]
	\caption{SFAL of RSA}
	\textbf{Input}: initialization $\theta_0$, global communication round $T$, local updating iterations $K$, learning rate $\eta_t$, smoothing parameter $\gamma$\\
	\textbf{Output} $\theta^T$ given by the server
	\begin{algorithmic}[1] 
		\FOR{$t=0,1...,T-1$}
		\STATE          $\theta^{t+1}_{i,0}=\theta_t, \forall i=1,..,m$
		\FOR{$k=0,1,...,K-1$}
		\STATE          Sample $u_{i,k_q}\in\mathbb{P}^d$ uniformly from a unit sphere in $\mathbb{P}^d$.
		\STATE          $\theta_{i,k+1}^t = \theta_{i,k}^t -\frac{ \eta_t }{Q}\sum_{q=1}^Q \nabla \ell_{\rho}(\theta_{i,k}^t+\gamma u_{i,k_q},z_{i,k}).$
		\STATE          Accumulate the loss $\mathcal{L}_{loss}^{i.t}$ in the last local epoch
		\ENDFOR
		\STATE      $\mathcal{L}_{all} \longleftarrow [\mathcal{L}_{loss}^{1,t}, \mathcal{L}_{loss}^{2,t}, \cdots,\mathcal{L}_{loss}^{i,t}], \mathcal{L}_{sorted} \longleftarrow  $ AscendingSort($\mathcal{L}_{all}$)
		\STATE      $\theta_{t+1}=\frac{1+\alpha}{\tilde m}\sum_{i=1}^{\hat{m}} \theta^{t+1}_{i,K} + \frac{1-\alpha}{\tilde m}\sum_{i=\hat{m}+1}^{m} \theta^{t+1}_{i,K}$
		\ENDFOR
	\end{algorithmic}
\end{algorithm}
\begin{algorithm}[htbp]
	\caption{SFAL of OPSA.}
	\textbf{Input}: initialization $\theta_0$, global communication round $T$, local updating iterations $K$, learning rate $\eta_t$\\
	\textbf{Output} $\theta^T$ given by the server
	\begin{algorithmic}[1] 
		\FOR{$t=0,1...,T-1$}
		\STATE          $\theta^{t+1}_{i,0}=\theta_t, \forall i=1,..,m$
		\FOR{$k=0,1,...,K-1$}
		\STATE              $\theta^{t+1}_{i,k+1}=\theta^{t+1}_{i,k}-\eta_t \nabla \ell_\rho(\theta^{t+1}_{i,k},z_{i,j})$
		\STATE          Accumulate the loss $\mathcal{L}_{loss}^{i.t}$ in the last local epoch
		\ENDFOR
		\STATE      $\mathcal{L}_{all} \longleftarrow [\mathcal{L}_{loss}^{1,t}, \mathcal{L}_{loss}^{2,t}, \cdots,\mathcal{L}_{loss}^{i,t}], \mathcal{L}_{sorted} \longleftarrow  $ AscendingSort($\mathcal{L}_{all}$)
		\STATE      $\theta_{t+1}=\frac{1+\alpha}{\tilde m}\sum_{i=1}^{\hat{m}} \theta^{t+1}_{i,K} + \frac{1-\alpha}{\tilde m}\sum_{i=\hat{m}+1}^{m} \theta^{t+1}_{i,K}$
		\ENDFOR
	\end{algorithmic}
\end{algorithm}
\newpage
\section{Experiment Setup And Additional Experiments}
For SSA and RSA, we experimented by adversarial training SmallCNN (\cite{zhang2019theoretically}) with six convolutional layers and two linear layers. And for OPSA, we compress SmallCNN into a shallow neural network with three convolutional layers and two linear layers.\\
\textbf{Dataset.} To validate our results, we conduct the experiments on two benchmark datasets, i.e., SVHN \cite{netzer2011reading} and CIFAR10 \cite{krizhevsky2009learning}. Fot simulating the Non-IID scenario, we follow \cite{mcmahan2017communication} and \cite{shah2021adversarial} to distribute the training data based on their labels. To be specific, a skew parameter $a$ is  utilized in the data partition introduced by \cite{shah2021adversarial}, which enables $m$ clients to get a majority of the data samples from a subset of classes. We denote the set of all classes in a dataset as $\mathcal{Y}$ and create $\mathcal{Y}_{\varrho}$ by dividing all the classes labels equally among $m$ clients. Accordingly, we split the data acroos $m$ clients that each client has $(100-(m-1)\times a)\%$ of data for the class in $\mathcal{Y}_{\varrho}$ and $a\%$ of data in other split sets.\\
\textbf{Federated Adversarial training.} We use the cross entropy loss as the original loss $\ell$. We adopt the 10-step projection gradient descent(PGD-10)(\cite{madry2017towards}) to generate adversarial examples. The attack strength is $\rho$ and the step size is $\rho/4$ in $L_{\infty}$-norm. We report robust accuracy a s the ratio fo correctly classified adversarial examples generated by PGD-10, and the robust generalization gap as the gap between robust training and robust test accuracy. For each client, we set the local updated epoches $K=3$ and the batch sizes of data loaders are set to 128 for SVHN and 64 for CIFAR10, respectively. Additionally, the learning rate are fixed to be 0.01. For the training schedule, SGD is adopted with 0.9 momentum and the weight decay = 0.0001.

\begin{figure*}[htb]
	\centering
	\includegraphics[width=0.45\linewidth, height=0.33\linewidth]{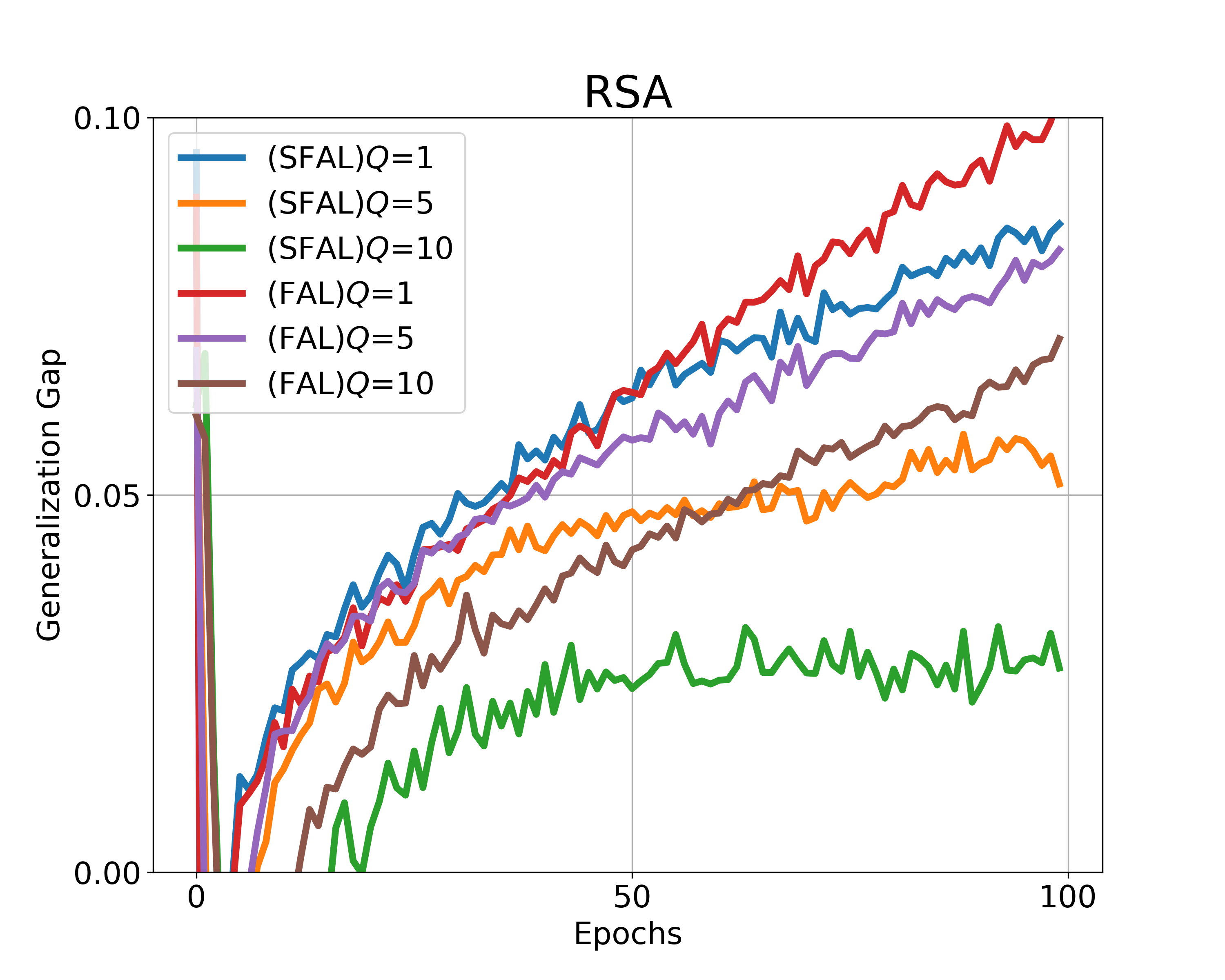}
	\caption{Generalization Gap of the number of noise $Q$ on SVHN. ($m=40,a=2.0,\rho=2.0$)}
	\label{fig:Q}
\end{figure*}
\textbf{Imapct on $Q$.} As shown in Figure \ref{fig:Q}, we can see that adjusting $Q$ significantly reduces the generalization error. This supports our view that more perturbations to the model bring the actual weight space closer to the ideal weight space. In practice, we recommend setting $Q \leq 5$ to reduce computational pressure on each client while still achieving good generalization error.
\subsection{More Experiments on CIFAR10}
\begin{figure*}[htb]
	\includegraphics[width=0.33\linewidth, height=0.23\linewidth]{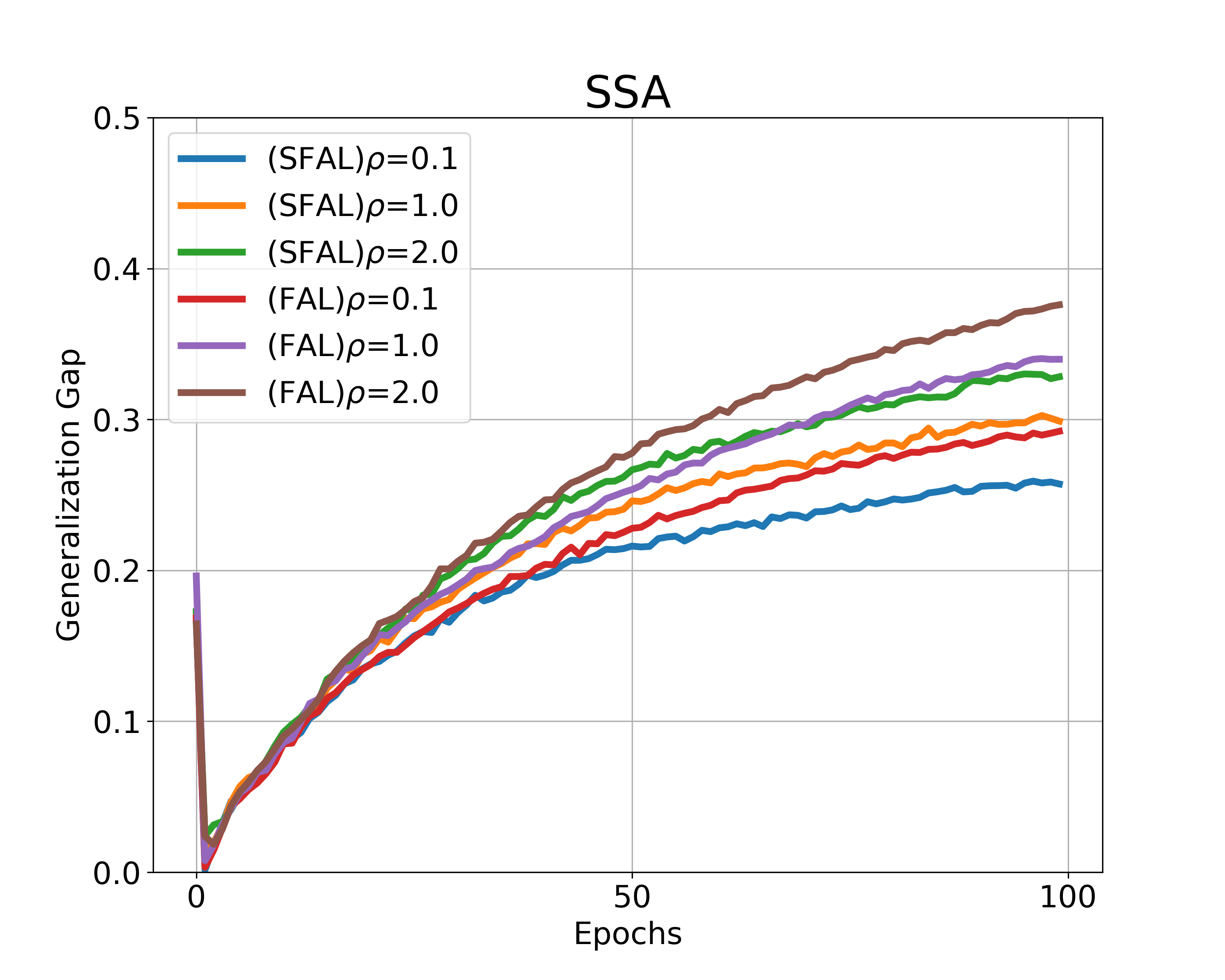}
	\includegraphics[width=0.33\linewidth, height=0.23\linewidth]{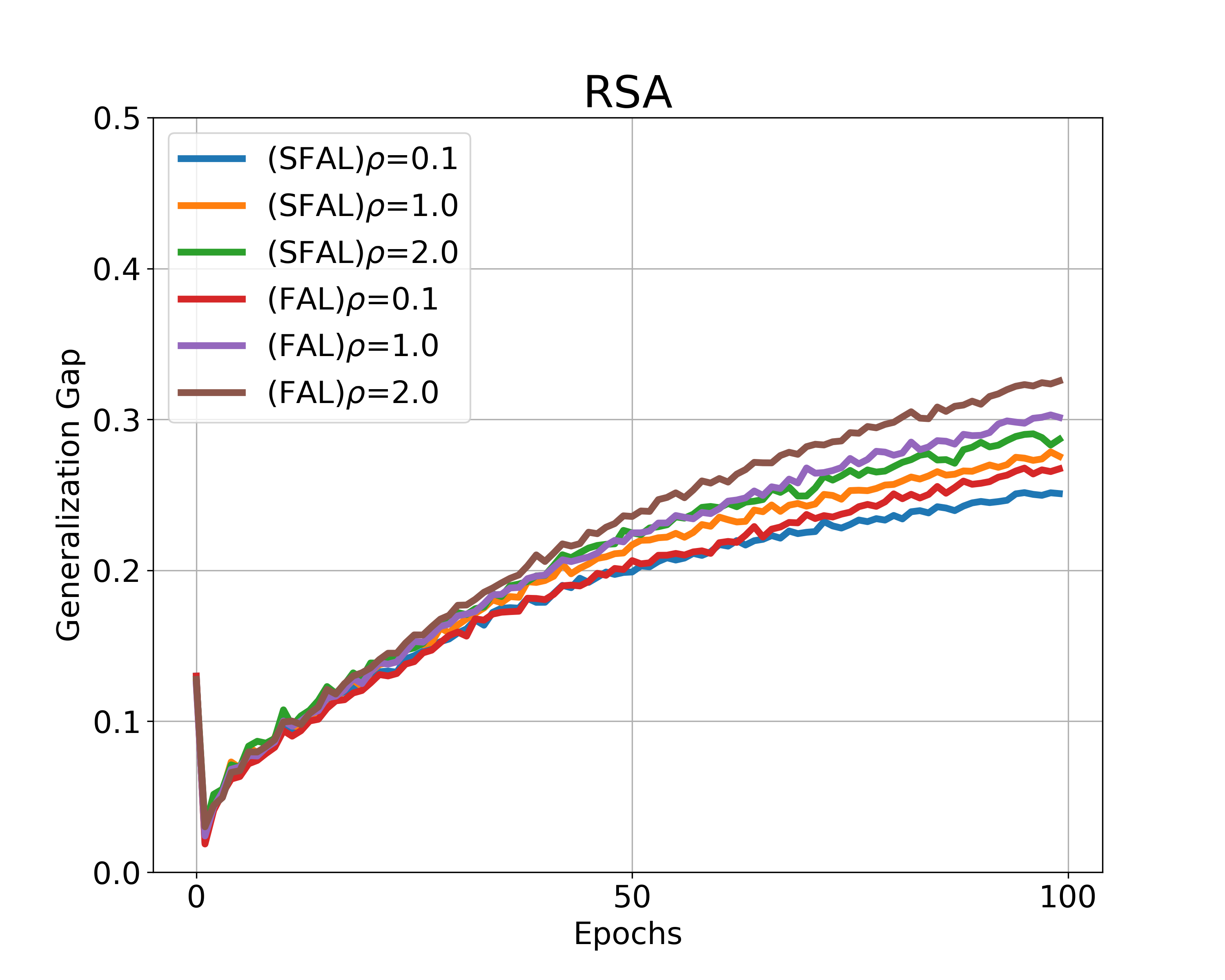}
	\includegraphics[width=0.33\linewidth, height=0.23\linewidth]{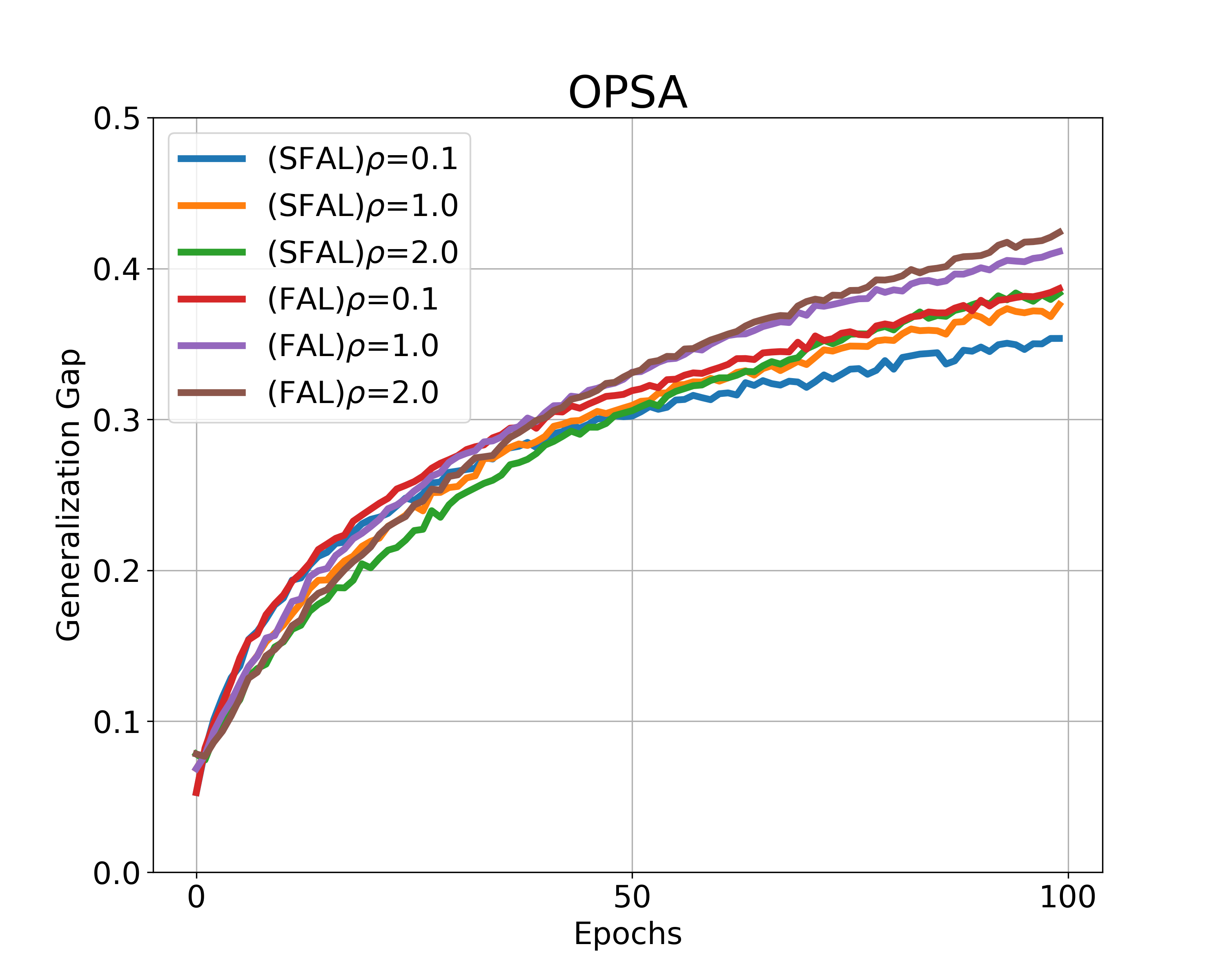}
	\caption{Generalization Gap of the attack strength $\rho$ on CIFAR10. ($m=40,a=2.0$)}
	\label{cifar10fig:rho}
\end{figure*}
\begin{figure*}[htb]
	\includegraphics[width=0.33\linewidth, height=0.23\linewidth]{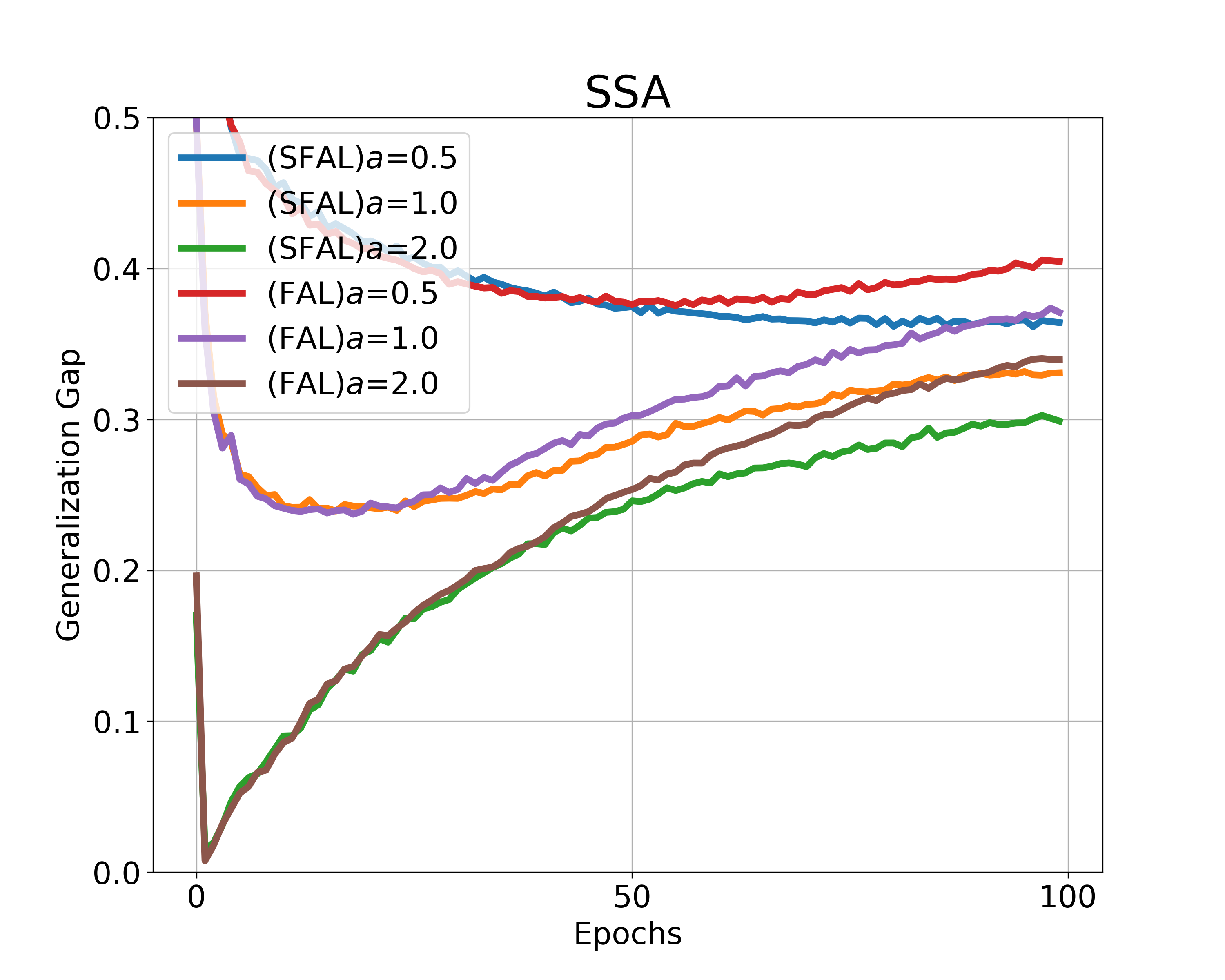}
	\includegraphics[width=0.33\linewidth, height=0.23\linewidth]{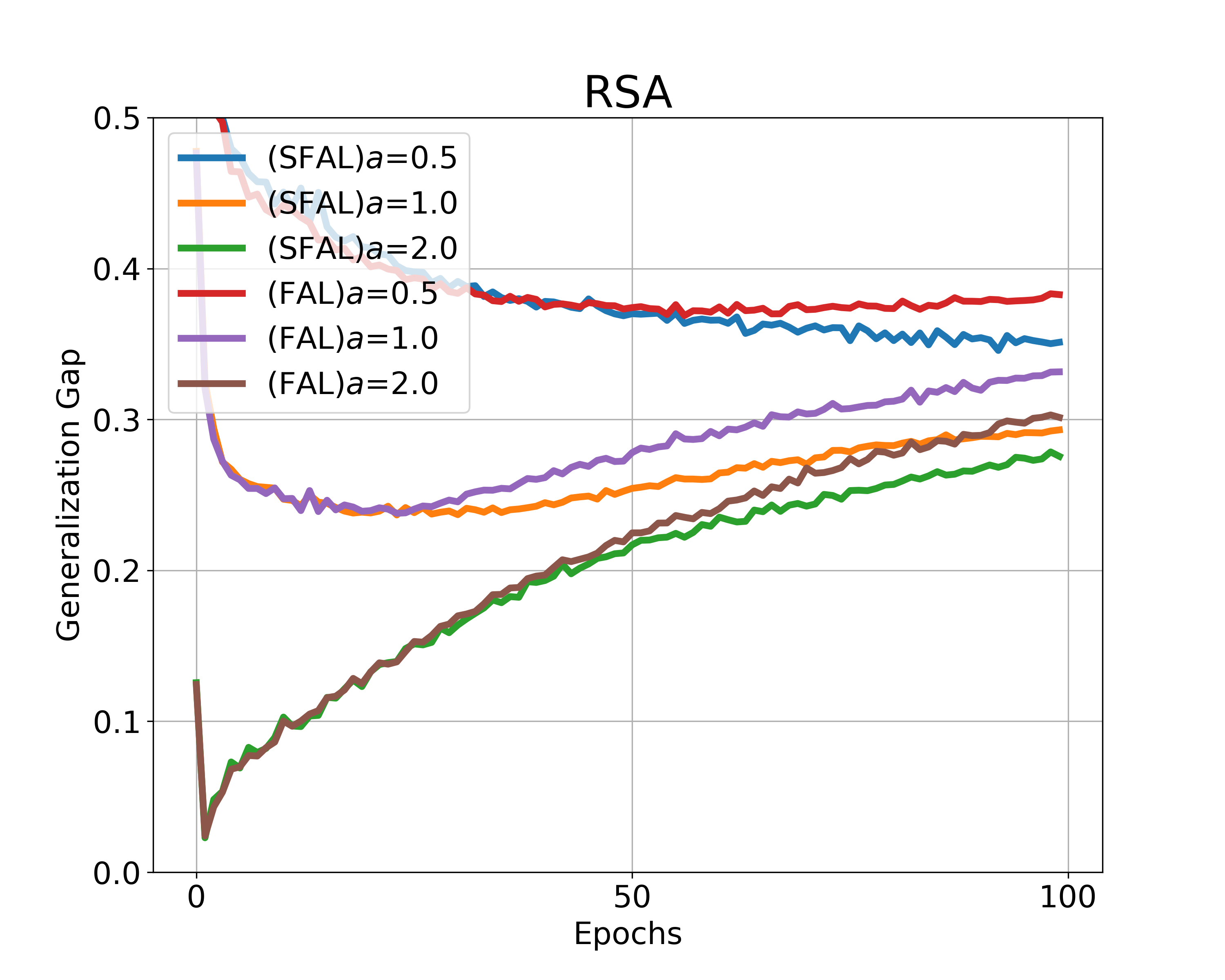}
	\includegraphics[width=0.33\linewidth, height=0.23\linewidth]{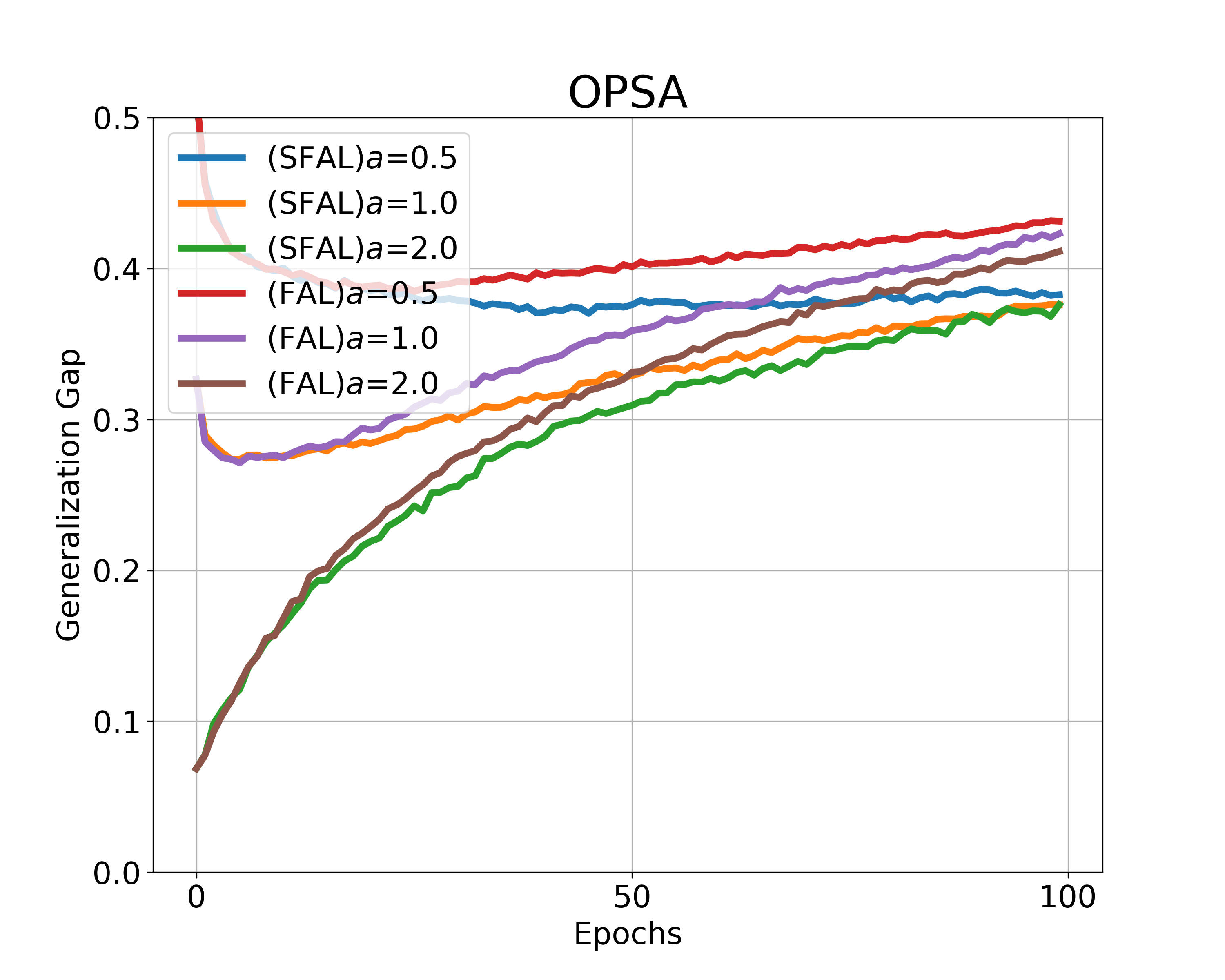}
	\caption{Generalization Gap of the skew parameter $a$ on CIFAR10. ($m=40,\rho=1.0$)}
	\label{cifar10fig:skew}
\end{figure*}
\begin{figure*}[htb]
	\includegraphics[width=0.33\linewidth, height=0.23\linewidth]{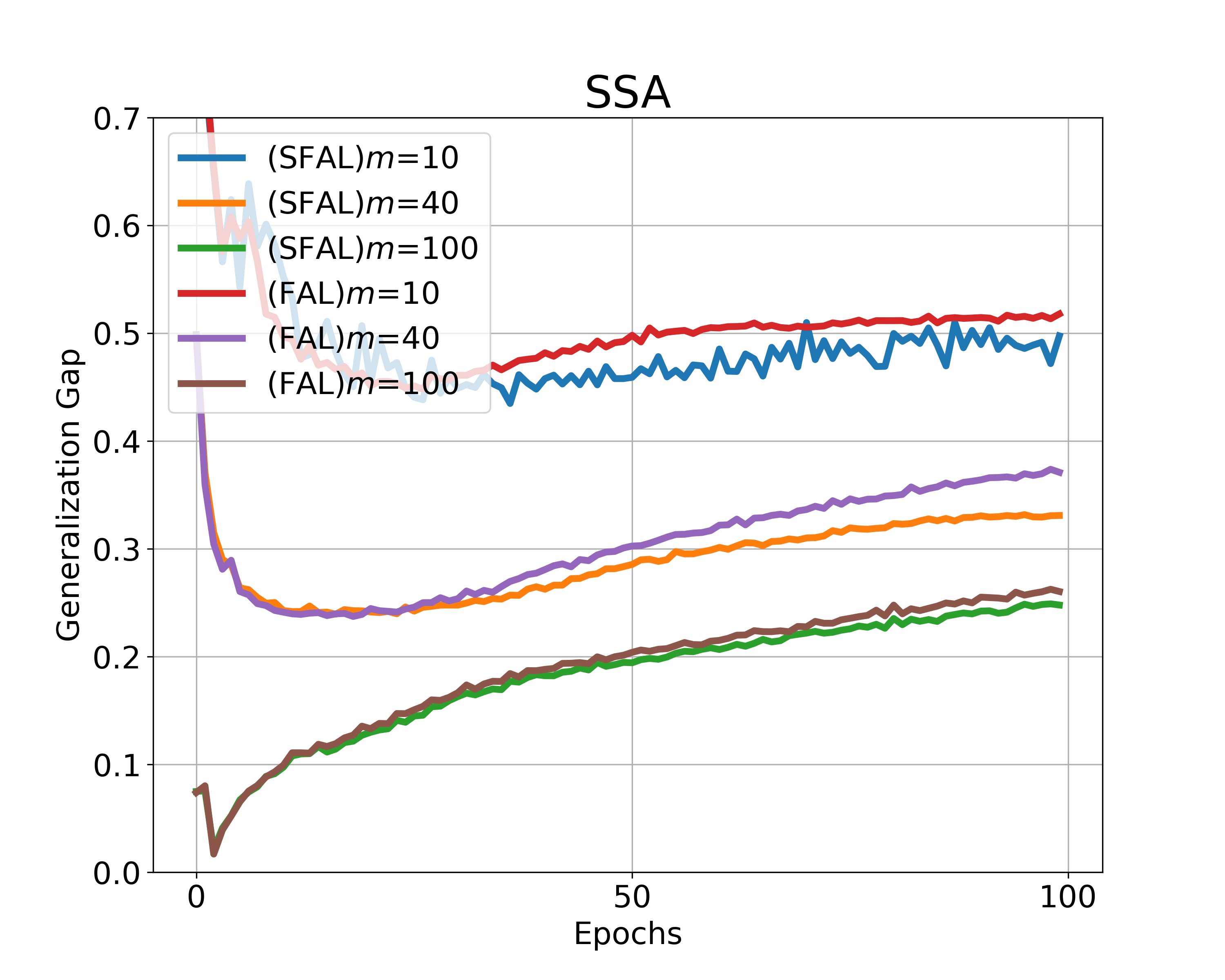}
	\includegraphics[width=0.33\linewidth, height=0.23\linewidth]{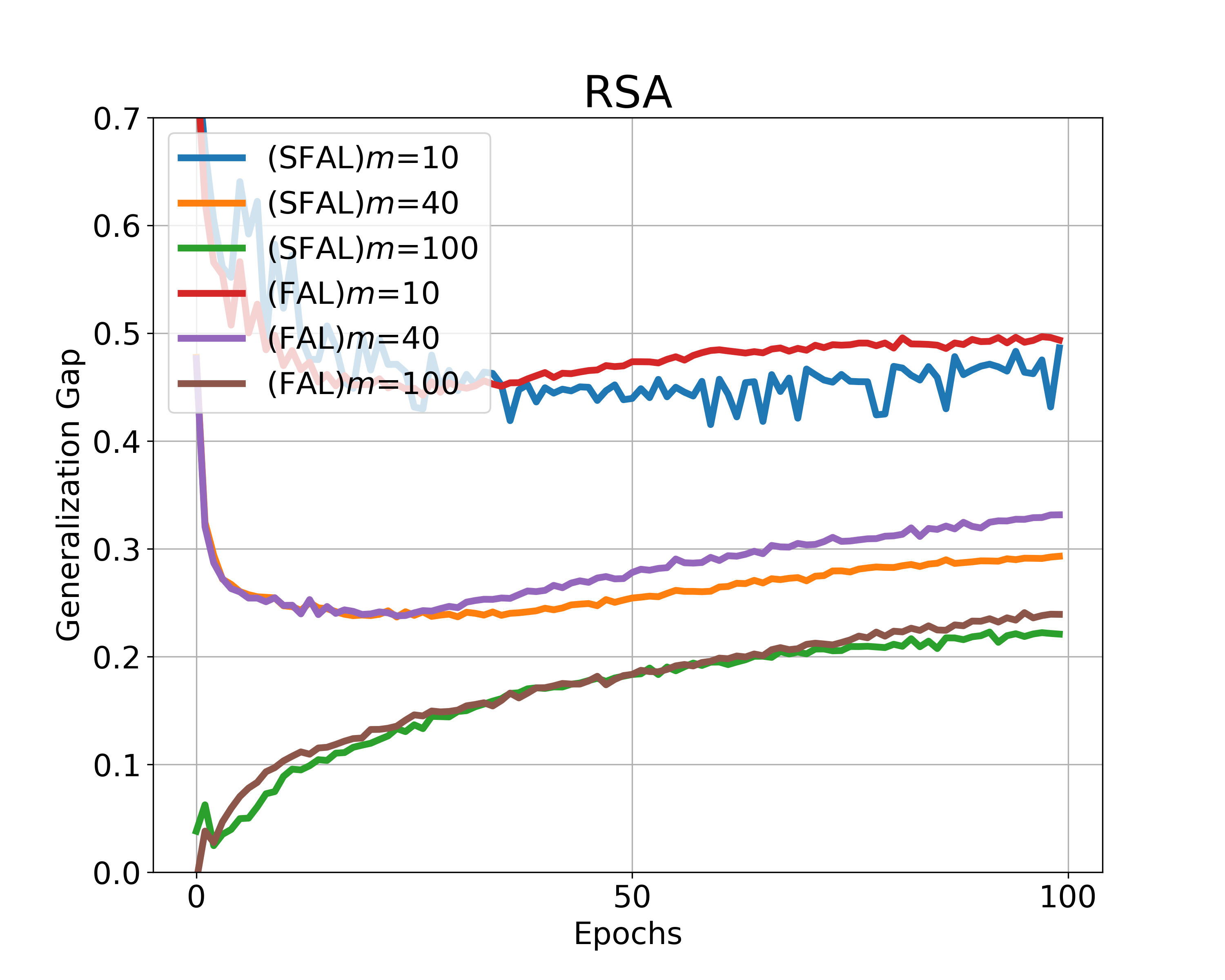}
	\includegraphics[width=0.33\linewidth, height=0.23\linewidth]{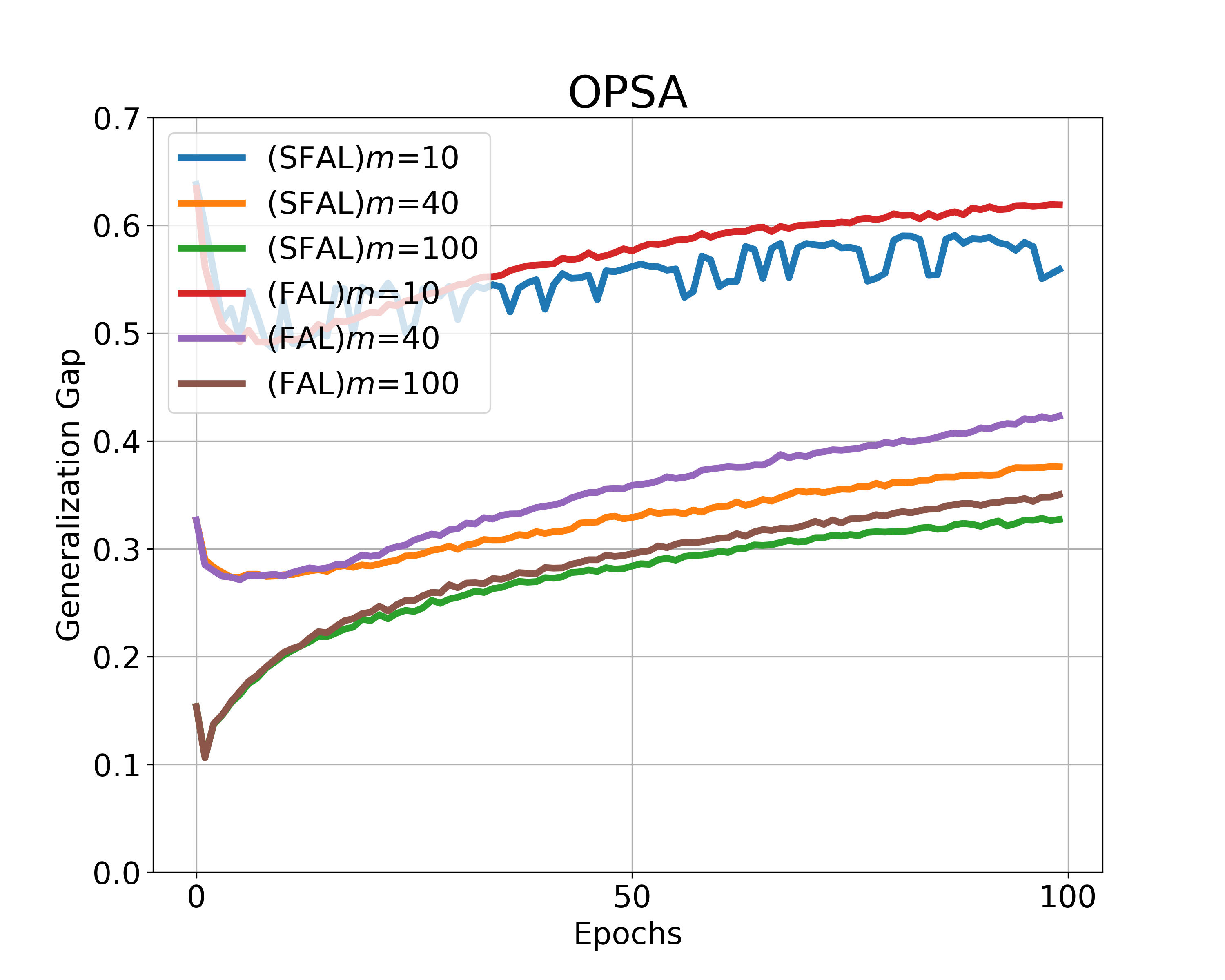}
	\caption{Generalization Gap of the number of client $m$ on CIFAR10. ($a=1.0,\rho=1.0$)}
	\label{cifar10fig:numberclient}
\end{figure*}
\newpage
\section{Useful Lemmas}\label{useful lemmas}
\textbf{Lemma 1}. \textit{Under Assumption \ref{Assumption 1} and given $i\in [m]$, for any $\theta$ we have}
$$\Vert\nabla R_i(\theta_{t})-\nabla R(\theta_{t})\Vert\leq (\xi+6L)D_i,$$
\textit{where} $D_i$ = $\max \{d_{TV}(\tilde{P}_i, P_i), d_{TV}(P_i, P), d_{TV}(\tilde{P},P)\}$. 
\begin{proof}
	For simplicity, we denote $\tilde{z} = \{x+A_{\rho},y\}$, where $A_{\rho}$ is defined in (\ref{attack}).
	\begin{equation*}
		\begin{aligned}
			&\Vert \nabla R_{i}(\theta)-\nabla R(\theta)\Vert \\  
			=       &     \Vert \nabla \int_{\mathcal{Z}_{i}} \ell_{\rho}(\theta ; z) d \tilde{P_{i}}-\nabla \int_{\mathcal{Z}} \ell_{\rho}(\theta ; z) d \tilde{P}\Vert \\
			=       &     \Vert\int_{\mathcal{Z}_{i} \cup \mathcal{Z}}\left(\nabla \ell_{\rho}(\theta ; z) d\tilde{P_{i}} -\nabla \ell_{\rho}(\theta ; z) d\tilde{P}\right)\Vert \\
			\leq    &     \int_{\mathcal{Z}_{i} \cup \mathcal{Z}}\Vert\nabla \ell_{\rho}(\theta ; z) d\tilde{P_{i}}-\nabla \ell_{\rho}(\theta ; z) d\tilde{P}\Vert \\
			=       &      \int_{\mathcal{Z}_{i} \cup \mathcal{Z}}\Vert\nabla \ell(\theta ; \tilde{z}) d\tilde{P_{i}} - \nabla \ell(\theta; \tilde{z}) d\tilde{P}\Vert\\
			\leq    &      \int_{\mathcal{Z}_{i} \cup \mathcal{Z}}\Vert\nabla \ell(\theta ; \tilde{z}) d\tilde{P_{i}} - \nabla \ell(\theta ; \tilde{z}) dP_{i} \Vert 
			+  \int_{\mathcal{Z}_{i} \cup \mathcal{Z}}\Vert\nabla \ell(\theta ; \tilde{z}) dP_{i}-\nabla \ell(\theta ; z) dP_{i} + \nabla\ell(\theta;z)dP - \nabla\ell(\theta;\tilde z)dP\Vert \\
			&   +  \int_{\mathcal{Z}_{i} \cup \mathcal{Z}}\Vert\nabla \ell(\theta ; z) dP_{i}-\nabla \ell(\theta ; z) dP \Vert 
			+  \int_{\mathcal{Z}_{i} \cup \mathcal{Z}}\Vert\nabla \ell(\theta ; z) dP-\nabla \ell(\theta ; z) d\tilde{P}  \Vert \\  
			\leq    &      \int_{\mathcal{Z}_{i} \cup \mathcal{Z}} L\Vert d\tilde{P_i} - P_i \Vert 
			+  \int_{\mathcal{Z}_{i} \cup \mathcal{Z}} L_z \Vert \tilde{z} - z \Vert \Vert dP_i - dP \Vert
			+  \int_{\mathcal{Z}_{i} \cup \mathcal{Z}} L \Vert P_i - P \Vert +  \int_{\mathcal{Z}_{i} \cup \mathcal{Z}} L \Vert \tilde{P} - P \Vert \\
			\leq    &      2Ld_{TV}(d\tilde{P}_i, dP_i)               
			+  (\xi + 2L) d_{TV}(dP_i, dP)
			+  2Ld_{TV}(d\tilde{P},dP) \\
			\leq    &   (\xi+6L)D_i  ,
		\end{aligned}
	\end{equation*} 
	where $D_i$ = $\max \{d_{TV}(\tilde{P}_i, P_i), d_{TV}(P_i, P), d_{TV}(\tilde{P},P)\}$.
	In the first inequality, we use Jensen's inequality. In the third inequality, we use Assumption \ref{Assumption 1}.
	By noting the definition of the total variation of two distributions $P$ and $Q$ is
	$$ d_{TV}(P,Q) = \dfrac{1}{2} \int |dP-dQ|.$$
	This completes the proof.
\end{proof}
\begin{Lemma}[\cite{xiao2022stability}]\label{Lemma 2}
	Let $h$ be the adversarial surrogate loss defined in \ref{adversiarl surrogate loss} and $\ell$ satisfies Assumption \ref{Assumption 1}. $\forall \theta_1, \theta_2$ and $\forall z\in \mathcal{Z}$, the following hold.  
	\begin{equation*}
		\begin{aligned}
			\Vert h(\theta_1,z)-h(\theta_2,z)\Vert \leq & L\Vert\theta_1 - \theta_2\Vert \\
			\Vert\nabla h(\theta_1,z)-\nabla h(\theta_2,z) \Vert  \leq &L_{\theta} \Vert \theta_1 - \theta_2\Vert + 2L_z\rho
		\end{aligned}
	\end{equation*}
\end{Lemma}

\begin{Lemma}\label{Lemma 3}(\citet{xiao2022stability})
	Assume that the function $h$ is $\xi$-approximately $\beta$-gradient Lipschitz, $\forall \theta_1, \theta_2$ and $\forall z\in \mathcal{Z}$
	$$h(\theta_1)-h(\theta_2) \leq \nabla h(\theta^{\prime})^T (\theta_1 - \theta_2) + \frac {\beta}{2} \Vert \theta_1 - \theta_2\Vert^2+\xi\Vert \theta_1 - \theta_2 \Vert.$$
	$$\Vert G_{\alpha,z}(\theta_1) - G_{\alpha,z}(\theta_2)\Vert\leq (1+\eta\beta)\Vert\theta-\theta'\Vert+\eta\xi.$$
\end{Lemma}

\begin{Lemma}\label{Lemma 4}(\citet{reddi2020adaptive})
	For random variables $A_1,\cdots,A_T$, we have
	\begin{equation*}
		\mathbb{E}[\Vert\sum\limits_{k=1}^{T}A_k\Vert^2]\leq k\sum\limits_{k=1}^T\mathbb{E}[\Vert A_k\Vert^2].
	\end{equation*}
	specially, if $A_1,\cdots,A_T$ are independent, mean $0$ random variables, then
	\begin{equation*}
		\mathbb{E}[\Vert\sum\limits_{k=1}^{T}A_k\Vert^2]= \sum\limits_{k=1}^T\mathbb{E}[\Vert A_k\Vert^2].
	\end{equation*}
\end{Lemma}

\section{Generalization Analyses of VFAL}
\textbf{Theorem 1}.\textit{If a VFAL algorithm $\mathcal{A}$ is $\epsilon$-on-averagely stable, we can obtain the generalization error $\varepsilon_{gen}(\mathcal{A})$ as follows:}
$$\mathbb{E}_{\mathcal{S,A}}\left[\frac{1}{m}\sum_{i=1}^{m}(R_i(\mathcal{A}(\mathcal{S}))-R_{\mathcal{S}_i}(\mathcal{A}(\mathcal{S})))\right]\leq \epsilon.$$
\begin{proof}
	Given $\mathcal{S}$ and $\mathcal{S}^{(i')}$ which are neighboring datasets defined in Definition \ref{Definition 1}
	\begin{equation*}
		\begin{aligned}
			\mathbb{E}_{\mathcal{S}}\left[R_{\mathcal{S}_{i}}(\mathcal{A}(\mathcal{S}))\right] 
			=       &     \mathbb{E}_{\mathcal{S}}\left[\frac{1}{n_{i}} \sum_{j=1}^{n_{i}} \ell_{\rho}\left(\mathcal{A}(\mathcal{S}) ; z_{i, j}\right)\right] \\
			=       &     \frac{1}{n_{i}} \sum_{j=1}^{n_{i}} \mathbb{E}_{\mathcal{S}}\left[\ell_{\rho}\left(\mathcal{A}(\mathcal{S}) ; z_{i, j}\right)\right] = \frac{1}{n_{i}} \sum_{j=1}^{n_{i}} \mathbb{E}_{\mathcal{S}, z_{i', j}^{\prime}}\left[\ell_{\rho}\left(\mathcal{A}\left(\mathcal{S}^{(i')}\right) ; z_{i', j}^{\prime}\right)\right]
		\end{aligned}
	\end{equation*}
	Moreover. we have
	$$ \mathbb{E}_{\mathcal{S}}\left[R_{i}(\mathcal{A}(\mathcal{S}))\right]=\frac{1}{n_{i}} \sum_{j=1}^{n_{i}} \mathbb{E}_{\mathcal{S}, z_{i', j}^{\prime}}\left[\ell_{\rho}\left(\mathcal{A}(\mathcal{S}) ; z_{i', j}^{\prime}\right)\right] $$ 
	since $z'_{i',j}$ and $\mathcal{S}$ are independent for any $j$. Thus,
	\begin{equation*}
		\begin{aligned}
			\varepsilon_{gen}(\mathcal{A})=& \mathbb{E}_{\mathcal{A}, \mathcal{S}}\left[ \frac{1}{m}\sum_{i=1}^{m}\left(R_{i}(\mathcal{A}(\mathcal{S}))-R_{\mathcal{S}_{i}}(\mathcal{A}(\mathcal{S}))\right)\right] \\
			=& \frac{1}{m}\sum_{i=1}^{m}\mathbb{E}_{\mathcal{A}}\left[\frac{1}{n_{i}} \sum_{j=1}^{n_{i}} \mathbb{E}_{\mathcal{S}, z_{i', j}^{\prime}}\left(\ell_{\rho}\left(\mathcal{A}(\mathcal{S}) ; z_{i', j}^{\prime}\right)-\ell_{\rho}\left(\mathcal{A}\left(\mathcal{S}^{(i)}\right) ; z_{i', j}^{\prime}\right)\right)\right] \\
			\leq& \epsilon,
		\end{aligned}
	\end{equation*}
	where the last inequality follows Definition \ref{Definition 2}, This completes the proof. 
\end{proof}
Next we state the assumption that bound the variance of stochastic gradient, which is common in stochastic optimization analysis\citep{bubeck2015convex}.
\begin{Assumption}\label{Assumption 4}
	There exists $\sigma>0$ such that for any $\theta$ and $i\in [m],\mathbb{E}_{z_i}[\nabla h(\theta,z_i)] = \nabla R_i(\theta)$ and $ \mathbb{E}\Vert \nabla h(\theta,z) - \nabla R_i(\theta) \Vert^2 \leq \sigma^2$    
\end{Assumption}
\subsection{Generalization analysis of VFAL under Surrogate Smoothness approximation}
For simplicity, we use $g(\theta)$ to denote stochastic gradient of loss function $h$.
\begin{Lemma}
	\label{Lemma 5}
	If the Assumptions \ref{Assumption 1} and \ref{Assumption 4} hold and  for any learning rate satisfying $\eta_t \leq \frac{1}{2\beta K}$, we can find the per-round recursion as
	\begin{equation}
		\label{Lemma 8.1}
		\begin{aligned}
			\mathbb{E}[R(\theta_{t+1})]-\mathbb{E}[R(\theta_{t})]
			\leq    &   - \frac{K\eta_t}{2}\mathbb{E}\Vert \nabla R(\theta_t)\Vert^2 
			+ \frac{\eta_t\beta^2}{m}\sum\limits_{i=1}^{m}\sum\limits_{k=0}^{K-1}\mathbb{E}\Vert\theta^{t+1}_{i,k}-\theta_t\Vert^2 
			+ K\eta_t (\xi^2 + \xi L)  
			+ \frac{\beta K \eta_t^2\sigma^2}{m}  
		\end{aligned}
	\end{equation}
\end{Lemma}
\begin{proof}
	Without otherwise stated, the expectation is conditioned on $\theta$. Beginning from $\xi$-approximately $\beta$-gradient Lipschitz,
	\begin{equation*}
		\begin{aligned}
			R(\theta_{t+1})-R(\theta_{t})-\left \langle\nabla R(\theta_t),\theta_{t+1}-\theta_t\right\rangle 
			\leq    & \frac{\beta}{2}\Vert\theta_{t+1}-\theta_{t}\Vert^2+\xi\Vert\theta_{t+1}-\theta_t\Vert \\
		\end{aligned}
	\end{equation*}
	where the above inequality we use Lemma \ref{Lemma 3}.
	Considering the update rule of FAL,
	$$\theta_{t+1} = \theta_t - \frac{\eta_t}{m}\sum\limits_{i=1}^m \sum\limits_{k=0}^{K-1}g_i(\theta^{t+1}_{i,k}),$$
	Using $ \mathbb{E} [g_i(\theta_{i,k}^{t+1})] = \nabla R_i(\theta_{i,k}^{t+1}) $ and Lemma \ref{Lemma 2}, then we can get
	\begin{equation}\label{Lemma 8.2}
		\begin{aligned}
			\mathbb{E}&[R(\theta_{t+1})]-\mathbb{E}[R(\theta_t)]\\
			\leq    &     \mathbb{E}\left\langle \nabla R(\theta_t),-\frac{\eta_t}{m}\sum\limits_{i=\eta}^{m}\sum\limits_{k=0}^{K-1}\nabla R_i(\theta^{t+1}_{i,k}) \right\rangle 
			+ \frac{\beta}{2}\mathbb{E}\Vert\frac{\eta_t}{m}\sum\limits_{i=1}^m \sum\limits_{k=0}^{K-1}g_i(\theta^{t+1}_{i,k})\Vert^2 + \xi L K\eta_t \\
			=       &   - \underbrace{\frac{\eta_t}{m}\sum\limits_{i=1}^{m}\sum\limits_{k=0}^{K-1}\mathbb{E}\left\langle \nabla R(\theta_t),\nabla R_i(\theta^{t+1}_{i,k}) \right\rangle}_{\Lambda _1}
			+ \underbrace{\frac{\beta}{2}\mathbb{E}\Vert\frac{\eta_t}{m}\sum\limits_{i=1}^m \sum\limits_{k=0}^{K-1}g_i(\theta^{t+1}_{i,k})\Vert^2}_{\Lambda_2} 
			+ \xi L K\eta_t.
		\end{aligned}
	\end{equation}
	Using   $\left \langle a,b \right\rangle = \frac{1}{2}\Vert a\Vert^2 + \frac{1}{2}\Vert b \Vert^2 - \frac{1}{2}\Vert a-b \Vert^2$, with $a = \nabla R(\theta_t)$    and $b = \nabla R_i(\theta^{t+1}_{i,k})$, we have
	\begin{equation*}
		\begin{aligned}
			\Lambda_1 
			\leq    &   -\frac{\eta_t}{2m}\sum\limits_{i=1}^{m}\sum\limits_{k=0}^{K-1}\mathbb{E}\left[\Vert\nabla R(\theta_t)\Vert^2 + \Vert \nabla R_i(\theta^{t+1}_{i,k})\Vert^2 - \Vert\nabla R_i(\theta^{t+1}_{i,k})-\nabla R(\theta_t)\Vert^2\right]   \\
			=       &   -\frac{K\eta_t}{2}\mathbb{E}\Vert \nabla R(\theta_t)\Vert^2-\frac{\eta_t}{2m}\sum\limits_{i=1}^{m} \sum\limits_{k=0}^{K-1}\mathbb{E}\Vert \nabla R_i(\theta^{t+1}_{i,k})\Vert^2   
			+  \frac{\eta_t}{2m}\sum\limits_{i=1}^{m}\sum\limits_{k=0}^{K-1}\mathbb{E}\Vert \nabla R_i(\theta^{t+1}_{i,k}) - \nabla R(\theta _t)\Vert^2\\   
			\leq    &   -\frac{K\eta_t}{2}\mathbb{E}\Vert \nabla R(\theta_t)\Vert^2 - \frac{\eta_t}{2m}\sum\limits_{i=1}^{m}\sum\limits_{k=0}^{K-1}\mathbb{E}\Vert \nabla R_i(\theta^{t+1}_{i,k})\Vert^2 
			+ \frac{\eta_t}{2m}\sum\limits_{i=1}^{m}\sum\limits_{k=0}^{K-1}\left[ \mathbb{E}\left[2\beta^2\Vert\theta^{t+1}_{i,k}-\theta_{t}\Vert^2 \right]+ 2\xi^2\right] \\
			=    &   -\frac{K\eta_t}{2}\mathbb{E}\Vert \nabla R(\theta_t)\Vert^2 
			- \frac{\eta_t}{2m}\sum\limits_{i=1}^{m}\sum\limits_{k=0}^{K-1}\mathbb{E}\Vert \nabla R_i(\theta^{t+1}_{i,k})\Vert^2 
			+ \frac{\eta_t\beta^2}{m}\sum\limits_{i=1}^{m}\sum\limits_{k=0}^{K-1}\mathbb{E}\Vert\theta^{t+1}_{i,k}-\theta_t\Vert^2 
			+ K\xi^2 \eta_t,
		\end{aligned}
	\end{equation*}
	where we used Lemma \ref{Lemma 1} in the second inequality and noticing that $\theta_t=\theta^{t+1}_{i,0}$.
	\begin{equation*}
		\begin{aligned}
			\Lambda_2
			\leq    & \frac{\beta}{2}\left[2\mathbb{E}\Vert \dfrac{\eta_t}{m}\sum\limits_{i=1}^{m}\sum\limits_{k=0}^{K-1}g_i(\theta^{t+1}_{i,k})-\dfrac{\eta_t}{m}\sum\limits_{i=1}^{m}\sum\limits_{k=0}^{K-1}\nabla R_i(\theta^{t+1}_{i,k})\Vert^2 
			+2\mathbb{E}\Vert \dfrac{\eta_t}{m}\sum\limits_{i=1}^{m}\sum\limits_{k=0}^{K-1}\nabla R_i(\theta^{t+1}_{i,k})\Vert^2 \right] \\
			\leq    &   \dfrac{\beta \eta_t^2}{m^2}\sum\limits_{i=1}^{m}\sum\limits_{k=0}^{K-1}\mathbb{E}\Vert g_i(\theta^{t+1}_{i,k}) - R_i(\theta^{t+1}_{i,k})\Vert^2 
			+ \dfrac{mK\eta_t^2}{m^2}\sum\limits_{i=1}^{m}\sum\limits_{k=0}^{K-1} \mathbb{E}\Vert \nabla R_i(\theta^{t+1}_{i,k}) \Vert^2 \\
			\leq    & \frac{\beta K\eta_t^2\sigma^2}{m} 
			+  \dfrac{\beta K\eta_t^2}{m}\sum\limits_{i=1}^{m}\sum\limits_{k=0}^{K-1} \mathbb{E}\Vert \nabla R_i(\theta^{t+1}_{i,k}) \Vert^2.\\         
		\end{aligned}
	\end{equation*}
	In the second inequality, we used (\ref{Lemma 4}) for the first the second item.  
	In the third inequality, we used Assumption \ref{Assumption 4} for the first item. \\
	Substituting $\Lambda_1$ and $\Lambda_2$ into (\ref{Lemma 8.2}), we have
	\begin{equation*}
		\begin{aligned}
			\mathbb{E}[R(\theta_{t+1})]-\mathbb{E}[R(\theta_{t})]
			\leq    &   \Lambda_1 + \Lambda_2 + \xi LK\eta_t \\
			\leq    &   - \frac{K\eta_t}{2}\mathbb{E}\Vert \nabla R(\theta_t)\Vert^2 
			+ \frac{\eta_t\beta^2}{m}\sum\limits_{i=1}^{m}\sum\limits_{k=0}^{K-1}\mathbb{E}\Vert\theta^{t+1}_{i,k}-\theta_t\Vert^2 
			+ K\eta_t (\xi^2 + \xi L)  \\
			&   + \frac{\beta K \eta_t^2\sigma^2}{m}  
			- \sum\limits_{i=1}^{m}\sum\limits_{k=0}^{K-1}\dfrac{\eta_t}{2m}\left(1 - 2\beta K\eta_t\right) \mathbb{E} \Vert \nabla R_i(\theta_{i,k}^{t+1}) \Vert^2.
		\end{aligned}
	\end{equation*}
	Considering $\beta K\eta_t\leq \frac{1}{2}$, we can get $\left(1 - 2\beta K\eta_t\right)\geq 0$. Then, we completes the proof.
\end{proof}
\begin{Lemma}
	\label{Lemma 6}
	If the Assumptions \ref{Assumption 1} and \ref{Assumption 4} hold, for any learning rate satisfying $\eta_t \leq \frac{1}{3\beta K}$, we have
	\begin{equation}
		\label{Lemma 9.1}
		\begin{aligned}
			\frac{1}{m}\sum\limits_{i=1}^{m}\sum\limits_{k=0}^{K-1}\mathbb{E}\Vert\theta^{t+1}_{i,k}-\theta_t\Vert^2 \leq  \frac{9}{2}K^2\eta_t^2\sigma^2  + 6K^3\eta_t^2\xi^2 +\sum\limits_{i=1}^{m}\frac{3}{m}K^3\eta_t^2(\xi+6L)D_i^2 
			+ 3K^3\eta_t^2 \mathbb{E}\Vert\nabla R(\theta_t)\Vert^2.
		\end{aligned}
	\end{equation}
\end{Lemma}

\begin{proof}
	Without otherwise stated, the expectation is conditioned on $\theta$. Beginning with $\mathbb{E}\Vert\theta^{t+1}_{i,k}-\theta_t\Vert^2$. 
	Considering $\theta^{t+1}_{i,k} = \theta_{t} - \eta_t\sum\limits_{j=0}^{k-1}g_i(\theta^{t+1}_{i,j})$, we have
	$$\mathbb{E}\Vert\theta^{t+1}_{i,k}-\theta_t\Vert^2=\eta_t^2\mathbb{E}\Vert\sum\limits_{j=0}^{k-1} g_i(\theta^{t+1}_{i,j})\Vert^2.$$
	Then,
	\begin{equation*}
		\begin{aligned}
			\mathbb{E}\Vert\theta^{t+1}_{i,k}-\theta_t\Vert^2
			\le     &   4\eta_t^2\mathbb{E}\Vert\sum\limits_{j=0}^{k-1}g_i(\theta^{t+1}_{i,j})-\sum\limits_{j=0}^{k-1}\nabla R_i(\theta^{t+1}_{i,j})\Vert^2 
			+ 4\eta_t^2\mathbb{E}\Vert\sum\limits_{j=0}^{k-1}\nabla R_i(\theta^{t+1}_{i,j})-\sum\limits_{j=0}^{k-1}\nabla R_i(\theta_{t})\Vert^2 \\
			&   + 4\eta_t^2\mathbb{E}\Vert\sum\limits_{j=0}^{k-1}\nabla R_i(\theta_{t})-\sum\limits_{j=0}^{k-1}\nabla R(\theta_t)\Vert^2 
			+ 4\eta_t^2\mathbb{E}\Vert\sum\limits_{j=0}^{k-1}\nabla R(\theta_t)\Vert^2.
		\end{aligned}
	\end{equation*}
	Applying (\ref{Lemma 4}) in the above inequality, we get
	\begin{equation*}
		\begin{aligned}
			&\mathbb{E}\Vert\theta^{t+1}_{i,k}-\theta_t\Vert^2 \\
			\leq    &   4\sum\limits_{j=0}^{k-1}\eta_t^2 \mathbb{E}\Vert g_i(\theta^{t+1}_{i,j})-\nabla R_i(\theta^{t+1}_{i,j})\Vert^2 
			+ 4k\sum\limits_{j=0}^{k-1}\eta_t^2\mathbb{E}\Vert\nabla R_i(\theta^{t+1}_{i,j})-\nabla R_i(\theta_{t})\Vert^2\\
			&   + 4k^2\eta_t^2\mathbb{E}\Vert\nabla R_i(\theta_{t})-\nabla R(\theta_{t})\Vert^2 
			+ 4k^2\eta_t^2\mathbb{E}\Vert\nabla R(\theta_{t})\Vert^2\\
			\leq     &     4k\eta_t^2\sigma^2 
			+ 4k\eta_t^2\sum_{j=0}^{k-1}\left[2\beta^2\mathbb{E}\Vert\theta^{t+1}_{i,j}-\theta_t\Vert^2 + 2\xi^2 \right]
			+ 4k^2\eta_t^2 (\xi+6L)^2D_i^2 + 4k^2\eta_t^2\mathbb{E}\Vert\nabla R(\theta_t)\Vert^2\\
			=       &     4k\eta_t^2\sigma^2 
			+ 8k\eta_t^2\beta^2\sum\limits_{j=0}^{k-1}\mathbb{E}\Vert\theta^{t+1}_{i,j}-\theta_t\Vert^2 
			+ 8k^2\eta_t^2 \xi^2 + 4k^2\eta_t^2 (\xi+6L)^2D_i^2 
			+ 4k^2\eta_t^2\mathbb{E}\Vert\nabla R(\theta_t)\Vert^2.\\
		\end{aligned}
	\end{equation*}
	In the second inequality, we used Assumption \ref{Assumption 4} for the first item, (\ref{Lemma 1}) for the second item, Lemma \ref{Lemma 6} for the third item.\\
	Since $\mathbb{E}\|\theta_{i,k}^{t+1}-\theta_t \|=0$ when k=0, now we have
	\begin{equation*}
		\begin{aligned}
			\frac{1}{m}\sum\limits_{i=1}^{m}\sum\limits_{k=0}^{K-1}\mathbb{E}\Vert\theta^{t+1}_{i,k}-\theta_t\Vert^2 
			\leq  &  4\eta_t^2\sigma\sum_{k=1}^{K-1}k + \frac{8\eta_t^2\beta^2}{m}\sum\limits_{i=1}^{m} \sum\limits_{j=0}^{K-1}\mathbb{E}\Vert\theta^{t+1}_{i,j}-\theta_t\Vert^2 \sum_{k=1}^{K-1}k \\
			& + \frac{1}{m}\sum\limits_{i=1}^{m}\sum_{k=1}^{K-1}k^2 (8\eta_t^2\xi^2 + 4\eta_t^2(\xi+6L)^2D_i^2 + 4\eta_t^2\mathbb{E}\Vert\nabla R(\theta_t)\Vert^2) \\
			\leq  &  2K^2\eta_t^2\sigma^2 +  \frac{4K^2\eta^2\beta^2}{m}\sum\limits_{i=1}^{m}\sum\limits_{k=0}^{K-1}\mathbb{E}\Vert\theta^{t+1}_{i,k}-\theta_t\Vert^2 + \frac{8}{3}K^3\eta^2\xi^2 +\sum\limits_{i=1}^{m}\frac{4}{3m}K^3\eta^2(\xi+6L)^2D_i^2 \\
			& + \frac{4}{3}K^3\eta^2 \mathbb{E}\Vert\nabla R(\theta_t)\Vert^2,
		\end{aligned}
	\end{equation*}
	Where we use $\sum_{k=1}^{K-1}k=\frac{(K-1)K}2\leq\frac{K^2}2$ and $\sum_{k=1}^{K-1}k^2=\frac{(K-1)K(2K-1)}6\leq\frac{K^3}3$ in the second ineuqality. Rearrangin the above inequality, we have
	\begin{equation*}
		\begin{aligned}
			(1-4K^2\eta^2\beta^2)\frac{1}{m}\sum\limits_{i=1}^{m}\sum\limits_{k=0}^{K-1}\mathbb{E}\Vert\theta^{t+1}_{i,k}-\theta_t\Vert^2 
			\leq  2K^2\eta^2\sigma^2  + \frac{8}{3}K^3\eta^2\xi^2 +\sum\limits_{i=1}^{m}\frac{4}{3m}K^3\eta^2(\xi+6L)^2D_i^2 
			+ \frac{4}{3}K^3\eta^2 \mathbb{E}\Vert\nabla R(\theta_t)\Vert^2.
		\end{aligned}
	\end{equation*}
	Using the choice of $\eta$, $\beta K \eta \leq \frac{1}{3}$, which implies $1-4K^2\eta^2\beta^2\geq \frac{4}{9}$, we have 
	$$ \frac{1}{m}\sum\limits_{i=1}^{m}\sum\limits_{k=0}^{K-1}\mathbb{E}\Vert\theta^{t+1}_{i,k}-\theta_t\Vert^2 \leq  \frac{9}{2}K^2\eta^2\sigma^2  + 6K^3\eta^2\xi^2 +\sum\limits_{i=1}^{m}\frac{3}{m}K^3\eta^2(\xi+6L)^2D_i^2 
	+ 3K^3\eta^2 \mathbb{E}\Vert\nabla R(\theta_t)\Vert^2.$$
	This completes the proof.
\end{proof}
\begin{Lemma}\label{Lemma 7}
	Suppose Assumption \ref{Assumption 1} and \ref{Assumption 4} hold, $h(\theta)$ is a $\xi$-approximately $\beta$-gradient Lipschitz and non-convex function w.r.t $\theta$.
	$$\mathbb{E}\Vert \theta_{i,k} - \theta_{t}\Vert \leq (1 + \beta \eta_t)^{K - 1}K\eta_{t}(\xi + \sigma + (\xi+6L)D_i + \mathbb{E}\Vert\nabla R(\theta_{t})\Vert).$$
\end{Lemma}
\begin{proof} Considering the local update date of FAL
	\begin{equation*}
		\begin{aligned}
			\mathbb{E}\Vert\theta_{i,k+1} - \theta_{t}\Vert 
			=       &     \mathbb{E}\Vert \theta_{i,k} - \eta_t(g_i(\theta_{i,k}) - g_i(\theta_{t})))\Vert + \eta_t\mathbb{E}\Vert g_i(\theta_{t}) \Vert                                                \\
			\leq    &     (1+\beta\eta_t)\mathbb{E}\Vert\theta_{i,k}-\theta_{t}\Vert + \eta_t\xi                                                                                                   \\
			&   + \eta_t( \mathbb{E}\Vert g_i(\theta_{i,k})-\nabla R_i(\theta_{t})\Vert +\mathbb{E}\Vert \nabla R_i(\theta_{t}) -\nabla R(\theta_t)\Vert + \mathbb{E}\Vert\nabla R(\theta_t)\Vert) \\
			\leq    &     (1+\beta\eta_t)\mathbb{E}\Vert\theta_{i,k}-\theta_{t}\Vert + \eta_t(\xi+ \sigma + (\xi + 6L)D_i + \mathbb{E}\Vert\nabla R(\theta_{t}) \Vert),
		\end{aligned}
	\end{equation*}
	where we use Lemma \ref{Lemma 4} in the first inequality, Assumption \ref{Assumption 4} and Lemma \ref{Lemma 1}  in the second inequality. Unrolling the above and noting $\theta_{i,0} = \theta_{t}$ yields
	\begin{equation*}
		\begin{aligned}
			\mathbb{E}\Vert\theta_{i,k}-\theta_{t}\Vert 
			\leq    &     \sum\limits_{l=0}^{k-1} \eta_t(\xi + \sigma + (\xi+6L)D_i+\mathbb{E}\Vert\nabla R(\theta_{t})\Vert)(1 + \beta \eta_{t})^{k-1-l}     \\
			\leq    &     \sum\limits_{l=0}^{K-1} \eta_t(\xi+ \sigma + (\xi+6L)D_i + \mathbb{E}\Vert\nabla R(\theta_{t})\Vert)(1 + \beta \eta_{t})^{K-1} \\
			=    &     (1 + \beta \eta_{t})^{K - 1} K\eta_{t}(\xi+ \sigma + (\xi+6L)D_i + \mathbb{E}\Vert\nabla R(\theta_{t})\Vert). \\
		\end{aligned}
	\end{equation*}
\end{proof}
\begin{Lemma} \label{Lemma 8}
	Suppose Assumption \ref{Assumption 1} and \ref{Assumption 4} hold, $h(\theta)$ is a $\xi$-approximately $\beta$-gradient Lipschitz and non-convex function w.r.t $\theta$, for $\eta_t\leq \frac{1}{K\beta(t+1)}$ , we have
	\begin{equation*}
		\begin{aligned}
			&\dfrac{K}{m} \sum\limits_{t = 0}^{T - 1} \sum\limits_{i = 1}^{m} \eta_{t} \mathbb{E}\Vert g_i(\theta_{i,k})\Vert \\ 
			\leq    &     \beta e \log T(\sigma + \xi + (\xi+6L) D_{\max})   \\   
			+& e\sqrt{
				6\log T\Delta
				+ 6\log T(\xi^2 + \xi L)  
				+ \dfrac{\log T\sigma^2 \pi^2}{mK\beta}
				+ \log T\zeta(3)c^3\beta^2(\frac{27\sigma^2}{K\beta} + \frac{36\xi^2}{\beta} + \frac{18(\xi+6L)^2D_{\max}^2}{\beta})
			}.
		\end{aligned}
	\end{equation*}
\end{Lemma}
\begin{proof}
	\begin{equation*}
		\begin{aligned}
			& \dfrac{K}{m} \sum\limits_{t = 0}^{T - 1} \sum\limits_{i = 1}^{m} \eta_{t} \mathbb{E}\Vert g_i(\theta_{i,k})\Vert \\
			\leq    &      \dfrac{K}{m} \sum\limits_{t = 0}^{T - 1} \sum\limits_{i = 1}^{m} \eta_t \mathbb{E} 
			\Vert 
			g_i(\theta_{i,k}) - \nabla R_i (\theta _{i,k}^{t+1}) 
			+ \nabla R_i (\theta _{i,k}^{t+1}) - \nabla R_i (\theta _{t})
			+ \nabla R_i (\theta _{t}) - \nabla R(\theta _{t})
			+ \nabla R(\theta _{t}) 
			\Vert  \\
			\leq    &     \dfrac{K}{m} \sum\limits_{t = 0}^{T - 1} \sum\limits_{i = 1}^{m} \eta_t 
			(
			\mathbb{E} \Vert g_i(\theta_{i,k})-\nabla R_i(\theta_{i,k})\Vert +
			\mathbb{E} \Vert \nabla R_i(\theta_{i,k}) - \nabla R_i(\theta_{t}) \Vert +
			\mathbb{E} \Vert \nabla R_i(\theta_{t}) - \nabla R(\theta_{t}) \Vert +
			\mathbb{E} \Vert \nabla R(\theta_{t}) \Vert
			) \\                                                                             \\
			\leq    &     \dfrac{K}{m} \sum\limits_{t = 0}^{T - 1} \sum\limits_{i = 1}^{m} \eta_t
			(\sigma + \beta \Vert \theta_{i,k} - \theta_{t} \Vert + \xi + (\xi+6L) D_i + \mathbb{E} \Vert \nabla R(\theta_{t}) \Vert )  \\
			\leq    &     \dfrac{K}{m} \sum\limits_{t = 0}^{T - 1} \sum\limits_{i = 1}^{m} \eta_t
			(\sigma + \xi + (\xi+6L) D_i + (1 + \beta \eta_{t})^{K - 1} K\beta \eta_{t}(\sigma + \xi + (\xi+6L) D_i + \mathbb{E} \Vert \nabla R(\theta_{t}) \Vert )  + \mathbb{E} \Vert \nabla R(\theta_{t}) \Vert )  \\
			\leq       &      \dfrac{K}{m} \sum\limits_{t = 0}^{T - 1} \sum\limits_{i = 1}^{m}\eta_t\big(1+ (1 + \beta \eta_{t})^{K - 1} K\beta \eta_{t}\big)(\sigma + \xi + (\xi+6L) D_i)
			+  \dfrac{K}{m} \sum\limits_{t = 0}^{T - 1} \sum\limits_{i = 1}^{m} \eta_t\big(1+ (1 + \beta \eta_{t})^{K - 1} K\beta \eta_{t}\big)\mathbb{E} \Vert \nabla R(\theta_t) \Vert  \\
			\leq       &      (1+e)\left[\underbrace{\dfrac{K}{m} \sum\limits_{t = 0}^{T - 1} \sum\limits_{i = 1}^{m}\eta_t(\sigma + \xi + (\xi+6L) D_i)}_{\Lambda_1}
			+  \underbrace{\dfrac{K}{m} \sum\limits_{t = 0}^{T - 1} \sum\limits_{i = 1}^{m} \eta_t\mathbb{E} \Vert \nabla R(\theta_t) \Vert}_{\Lambda_2}\right],  \\ 
		\end{aligned}
	\end{equation*}
	where in the third inequality, we use Assumption \ref{Assumption 4} , Lemma \ref{Lemma 1} and Lemma \ref{Lemma 2}. In the fourth inequality, we use Lemma \ref{Lemma 7}. In the last inequality, we use $(1+x)\leq e^x$ and $\eta_t \leq \frac{1}{\beta K}$.\\
	Using the choice of $\eta_t \leq \frac{1}{ K\beta(t+1)}$
	\begin{equation*}
		\begin{aligned}
			\Lambda_1
			\leq    &     \dfrac{K}{m} \sum\limits_{t = 0}^{T - 1} \sum\limits_{i = 1}^{m} \eta_{t}(\sigma + \xi + (\xi+6L) D_i) 
			\leq   \log T(\sigma + \xi + (\xi+6L) D_{\max}) .         
		\end{aligned}
	\end{equation*}
	Combining Lemma \ref{Lemma 7} and Lemma \ref{Lemma 8}, we obtain
	\begin{equation*}
		\begin{aligned}
			K\eta_{t} \mathbb{E} \Vert \nabla R(\theta_{t}) \Vert^2
			\leq    &     2(\mathbb{E} [R(\theta_{t})] - \mathbb{E} [R(\theta_{t + 1})])
			+  2(\xi^2 + \xi L) K\eta_{t}
			+ \dfrac{2\beta \sigma^2 K\eta_{t}}{m} \\ 
			&   + 2 \beta^2  
			\left(  
			\dfrac{9}{2} K^2 \sigma^2 \eta^3_{t} 
			+ 6 \xi^2 K^2 \eta^3_{t}
			+ 3K^3  (\xi+L)^2 D_i^2 \eta^3_{t}
			+ 3 K^3 \eta^3_{t} \mathbb{E} \Vert \nabla R(\theta_t) \Vert^2
			\right) \\
			\leq    &   2(\mathbb{E} [R(\theta _t )] - \mathbb{E} [R(\theta _{t + 1} )]) 
			+ \dfrac{2(\xi^2 + \xi L )}{\beta(t+1)} 
			+ \dfrac{2\sigma^2}{m K\beta(t+1)^2} 
			+ \dfrac{9\sigma^2}{K\beta(t+1)^3}  \\
			&   + \dfrac{12\xi^2}{\beta(t+1)^3} 
			+ \dfrac{6(\xi+6L)^2 D_{\max}^2}{\beta(t+1)^3}
			+ \dfrac{2}{3} K\eta_{t} \mathbb{E}\Vert \nabla R(\theta_t) \Vert^2. 
		\end{aligned}
	\end{equation*}
	Where we use the choice of $\eta \leq \frac{1}{3K\beta}$ and $\eta \leq \frac{1}{K\beta(t+1)}$. Rearrange the above inequality,
	\begin{equation}
		\label{Lemma 12.1}
		\begin{aligned}
			K \eta_{t} \mathbb{E} \Vert \nabla R(\theta_{t}) \Vert^2
			\leq    &   6(\mathbb{E} [R(\theta _t )] - \mathbb{E} [R(\theta _{t + 1} )]) 
			+ \dfrac{6(\xi^2 + \xi L)}{\beta(t+1)} 
			+ \dfrac{6\sigma^2}{m K\beta(t+1)^2} 
			+ \dfrac{27 \sigma^2}{K\beta(t+1)^3} \\ 
			&   + \dfrac{36 \xi^2}{\beta(t+1)^3} 
			+ \dfrac{18(\xi+6L)^2 D_{\max}^2}{\beta(t+1)^3}.
		\end{aligned}
	\end{equation}
	Substituting it into $\Lambda_2$, noting $\sum\limits_{t = 0}^{ T - 1} \frac{1}{(t+1)^2} \leq \frac{\pi^2}{6}, \sum\limits_{t = 0}^{T-1}\frac{1}{(t+1)^3} \leq \zeta(3) $, which is Apéry's constant. Then we obtain that
	\begin{equation*}
		\begin{aligned}
			&\Lambda_2 
			\leq       \sum\limits_{t = 0}^{T - 1} 
			\eta_{t} \mathbb{E}\Vert \nabla R(\theta_t) \Vert \\
			\leq    &     \sqrt{\sum\limits_{t = 0}^{T - 1} \eta_t} \sqrt{\sum\limits_{t = 0}^{T - 1} \eta_t \mathbb{E} \Vert \nabla R(\theta_t) \Vert^2}\\ 
			\leq    &    \sqrt{\log T} 
			\sqrt{
				6\Delta
				+ 6\log T(\xi^2 + \xi L)  
				+ \dfrac{\sigma^2 \pi^2}{mK\beta}
				+ \zeta(3)c^3\beta^2(\frac{27\sigma^2}{K\beta} + \frac{36\xi^2}{\beta} + \frac{18(\xi+6L)^2D_{\max}^2}{\beta})
			}\\
			=       &    \sqrt{
				6\log T\Delta
				+ 6\log T(\xi^2 + \xi L)  
				+ \dfrac{\log^2 T\sigma^2 \pi^2}{mK\beta}
				+ \log T\zeta(3)c^3\beta^2(\frac{27\sigma^2}{K\beta} + \frac{36\xi^2}{\beta} + \frac{18(\xi+6L)^2D_{\max}^2}{\beta})
			},
		\end{aligned}
	\end{equation*} 
	where we used Cauchy-Schwarz inequality in the second inequality
	\begin{equation*}
		\begin{aligned}
			&\dfrac{K\eta_t}{m} \sum\limits_{t = 0}^{T - 1} \sum\limits_{i = 1}^{m}  \mathbb{E}\Vert g_i(\theta_{i,k})\Vert \\ 
			\leq    &    (1+e)(\Lambda_1 + \Lambda_2) \\
			\leq    &     (1+e)\beta \log T(\sigma + \xi + (\xi+6L) D_{\max})   \\   
			+& (1+e)\sqrt{
				6\log T\Delta
				+ 6\log T(\xi^2 + \xi L)  
				+ \dfrac{\log T\sigma^2 \pi^2}{mK\beta}
				+ \log T\zeta(3)c^3\beta^2(\frac{27\sigma^2}{K\beta} + \frac{36\xi^2}{\beta} + \frac{18(\xi+6L)^2D_{\max}^2}{\beta})
			}    .
		\end{aligned}
	\end{equation*}   
\end{proof}
Next we give the proof of Theorem \ref{Theorem 2}\\
\textbf{Theorem 2} \textit{Let the step size be chosen as $\eta_{t} \leq \frac{1}{\beta K(t+1)}$. Under Assumption~\ref{Assumption 1} and \ref{Assumption 4}, the generalization bound $\varepsilon_{gen}$ with the surrogate smoothness approximation satisfies:}
\begin{equation*}
	\mathcal{O}\left(\rho T\log T+\frac{T\sqrt{\log T \Delta}}{mn_{\mathrm{min}}}+\frac{T\log T(\rho+1)D_{\mathrm{max}}}{mn_{\mathrm{min}}}\right),
\end{equation*}
where $\Delta = \mathbb{E}[R(\theta_0)] - \mathbb{E}[R(\theta^*)]$ and $n_{\min} = \min n_i ,\forall i \in [m]$.
\begin{proof}
	Given time index $t$ and for client $j$ with $j\neq i$, the local dataset are identical since the perturbed data point only occurs at client $i$. Thus, when $j\neq i$, we have 
	\begin{equation*}
		\begin{aligned}
			\mathbb{E}\Vert\theta_{j,k +1}-\theta_{j,k +1}'\Vert
			=       &     \mathbb{E}\Vert\theta_{j,k}-\theta_{j,k}'-\eta_t(g_j(\theta_{j,k})-g_j(\theta_{j,k}))\Vert \\
			\leq    &     (1+\beta\eta_t)\mathbb{E}\Vert \theta_{j,k}-\theta_{j,k}'\Vert+\eta_t\xi,
		\end{aligned}
	\end{equation*}
	where we use Lemma \ref{Lemma 3} in the above inequality. And unrolling it gives
	\begin{equation}\label{Theorem 2.1}
		\begin{aligned}
			\mathbb{E}\Vert\theta_{j,K-1}-\theta_{j,K-1}'\Vert
			\leq    &     \prod_{k=0}^{K-1}(1+\beta\eta_t)\mathbb{E}\Vert\theta_t-\theta_t'\Vert+\sum\limits_{k=0}^{K-1}\eta_t\prod_{l=k+1}^{K-1}(1+\beta\eta_l)\xi \\
			\leq    &     e^{\beta K\eta_t}\mathbb{E}\Vert\theta_t-\theta_t'\Vert+e^{\beta K\eta_{t}}K \eta_{t}\xi                                \\
			=       &     e^{\beta K\eta_{t}}(\mathbb{E}\Vert\theta_t-\theta_t'\Vert+K\eta_{j,t}\xi),
		\end{aligned}
	\end{equation}
	where we use $1+x\le e^x,\forall x$.
	For client $i$, there are two cases to consider. In the first case, SGD selects non-perturbed samples in $\mathcal{S}$ and $\mathcal{S}^{(i)}$, which happens with probability $1-\frac{1}{n_i}$. Then, we have
	\begin{equation*}
		\begin{aligned}
			\Vert\theta_{i,k+1}-\theta_{i,k+1}'\Vert
			=       &     \Vert\theta_{i,k}-\theta_{i,k}'-\eta_t(g_i(\theta_{i,k})-g_i(\theta_{i,k}')) \\
			\leq    &     (1+\beta\eta_t)\Vert\theta_{i,k}-\theta_{i,k}'\Vert+\eta_t\xi,
		\end{aligned}
	\end{equation*}
	where we use Lemma \ref{Lemma 3} in the above inequality.
	In the second case, SGD encounters the perturbed sample at time step $k$, which happens with probability $\frac{1}{n_i}$. Then, we have
	\begin{equation*}
		\begin{aligned}
			\Vert\theta_{i,k+1}-\theta_{i,k+1}'\Vert
			=       &     \Vert\theta_{i,k}-\theta_{i,k}'-\eta_t(g_i(\theta_{i,k})-g_i'(\theta_{i,k}'))\Vert                                                              \\
			\leq    &     \Vert\theta_{i,k}-\theta_{i,k}'-\eta_t(g_i(\theta_{i,k})-g_i(\theta_{i,k}'))\Vert+\eta_t\Vert g_i(\theta_{i,k}')-g_i'(\theta_{i,k}')\Vert \\
			\leq    &     (1+\beta\eta_t)\Vert\theta_{i,k}-\theta_{i,k}'\Vert+\eta_t\xi+\eta_t\Vert g_i(\theta_{i,k}')-g_i'(\theta_{i,k}')\Vert,
		\end{aligned}
	\end{equation*}
	Combining these two cases for client i, we have
	\begin{equation*}
		\begin{aligned}
			\mathbb{E}\Vert\theta_{i,k+1}-\theta_{i,k+1}'\Vert 
			\leq    &     (1-\frac{1}{n_i})\left[(1+\beta\eta_t)\Vert\theta_{i,k}-\theta_{i,k}'\Vert+\eta_t\xi\right]                                                                       \\
			&   + \frac{1}{n_i}\left[  (1+\beta\eta_t)\mathbb{E}\Vert\theta_{i,k}-\theta_{i,k}'\Vert+\eta_t\xi+\eta_t\mathbb{E}\Vert g_i(\theta_{i,k}')-g_i'(\theta_{i,k}')\Vert\right]   \\
			\leq    &     (1+\beta\eta_t)\mathbb{E}\Vert\theta_{i,k}-\theta_{i,k}'\Vert+\eta_t\xi+\frac{\eta_t}{n_i}\mathbb{E}\Vert g_i(\theta_{i,k}')-g_i'(\theta_{i,k}')\Vert \\
			\leq    &     (1+\beta\eta_t)\mathbb{E}\Vert\theta_{i,k}-\theta_{i,k}'\Vert+\eta_t\xi+\frac{2\eta_t}{n_i}\mathbb{E}\Vert g_i(\theta_{i,k})\Vert.                             \\
		\end{aligned}
	\end{equation*}
	
	Then unrolling it gives,
	\begin{equation}\label{Theorem 2.2}
		\begin{aligned}
			\mathbb{E}\Vert\theta_{i,K}-\theta_{i,K}'\Vert
			\leq    &     \prod_{k=0}^{K-1}(1+\beta\eta_t)\mathbb{E}\Vert\theta_t-\theta_t'\Vert \\
			&   + \sum\limits_{k=0}^{K-1}\eta_t\prod_{l=k+1}^{K-1}(1+\beta\eta_{i,l})(\xi + \dfrac{2}{n_i}\mathbb{E} \Vert g_i(\theta_{i,k}) \Vert)                                                                                    \\
			\leq    &     e^{\beta\sum\limits_{k=0}^{K-1}\eta_t}\mathbb{E}\Vert\theta_{t}-\theta_{t}'\Vert 
			+ K\eta_{t}e^{\beta\sum\limits_{k=0}^{K-1}\eta_t}(\xi + \dfrac{2}{n_i}\mathbb{E} \Vert g_i(\theta_{i,k}) \Vert) \\
			=       &     e^{\beta K\eta_{t}}\mathbb{E}\Vert\theta_{t}-\theta_{t}'\Vert 
			+ K\eta_{t}e^{\beta K\eta_t}(\xi + \dfrac{2}{n_i}\mathbb{E} \Vert g_i(\theta_{i,k}) \Vert ) .                              \\
		\end{aligned}
	\end{equation}
	By \ref{Theorem 2.1} and \ref{Theorem 2.2} we have
	\begin{equation*}
		\begin{aligned}
			\mathbb{E}\Vert\theta_{t+1}-\theta_{t+1}'\Vert
			=       &     \mathbb{E}\Vert \dfrac{1}{m} \sum\limits_{j=1}^m(\theta_{j,K} - \theta_{j,K}') \Vert  \\
			\leq    &     \dfrac{1}{m} \sum\limits_{j = 1}^m \mathbb{E}\Vert\theta_{j,K} - \theta_{j,K}' \Vert \\
			=       &     \dfrac{1}{m} \sum\limits_{j = 1,j \neq i}^m \mathbb{E}\Vert\theta_{j,K} - \theta_{j,K}' \Vert + \dfrac{1}{m}\mathbb{E}\Vert \theta_{i,K} - \theta_{i,K}' \Vert\\                                                                                                                                     \\
			\leq    &     \dfrac{1}{m} \sum\limits_{j = 1}^{m}  e^{\beta K{\eta}_{t}}\mathbb{E}\Vert \theta_{t}-\theta_{t}'\Vert 
			+ \dfrac{1}{m} \sum\limits_{j = 1}^{m}  e^{\beta K{\eta}_{t}} K\eta_{t} \xi + \dfrac{}{} 
			+ \dfrac{2 K{\eta}_{t} e^{\beta K\eta_{t}}}{m n_i}\mathbb{E} \Vert g_i (\theta_{i,k}) \Vert. 
		\end{aligned}
	\end{equation*}
	Unrolling it and let $\eta_t\leq \frac{1}{K\beta(t+1)}$, we get
	\begin{equation*}
		\begin{aligned}
			\mathbb{E}\Vert\theta_{t+1}-\theta_{t+1}'\Vert
			\leq    &     \dfrac{1}{m} \sum\limits_{j = 1}^{m}  e^{\beta K\eta_t}\mathbb{E}\Vert \theta_{t}-\theta_{t}'\Vert 
			+ \dfrac{1}{m} \sum\limits_{j = 1}^{m}  e^{\beta K\eta_t} K\eta_t \xi
			+ \sum\limits_{ i = 1}^{m} \dfrac{2 K\eta_t e^{\beta K\eta_t}}{m^2 n_i} \mathbb{E} \Vert g_i (\theta_{i,k}) \Vert \\
			\leq    &     e^{\frac{1}{(1 + t)}}\mathbb{E}\Vert \theta_t - \theta_t' \Vert 
			+ e^{\frac{1}{(1 + t)}} \sum\limits_{j = 1}^{m} \frac{K\eta_t \xi}{m}  
			+ e^{\frac{1}{(1 + t)}} \sum\limits_{ i = 1}^{m} \dfrac{2 K\eta_t}{m^2 n_i} \mathbb{E} \Vert g_i(\theta_{i,k}) \Vert .
		\end{aligned}
	\end{equation*}
	Unrolling it over $t$ and noting $\theta_0 = \theta_0'$, we obtain
	\begin{equation*}
		\begin{aligned}
			&\mathbb{E}\Vert\theta_{T}-\theta_{T}'\Vert \\
			\leq    &     \sum\limits_{t = 0}^{T - 1} \exp(\sum\limits_{l = t + 1}^{T - 1} \frac{1}{l} )\left[\sum_{j = 1}^{m} \dfrac{K\eta_t \xi}{m} + \sum\limits_{ i = 1}^{m} \dfrac{2 K\eta_t}{m^2 n_i} \mathbb{E} \Vert g_i(\theta_{i,k}) \Vert\right]\\ 
			\leq    &     T \sum_{t = 0}^{T - 1} \left[\sum_{j = 1}^{m} \dfrac{K\eta_t \xi}{m}
			+ \sum\limits_{ i = 1}^{m} \dfrac{2 K\eta_t}{m^2 n_i} \mathbb{E} \Vert g_i(\theta_{i,k}) \Vert \right] \\
			\leq    &     T \sum_{t = 0}^{T - 1} \sum_{j = 1}^{m} \dfrac{K\eta_t \xi}{m} 
			+ T^{\beta c} \sum_{t = 0}^{T - 1} \sum\limits_{ i = 1}^{m} \dfrac{2 K\eta_t}{m^2 n_i} \mathbb{E} \Vert g_i(\theta_{i,k}) \Vert   \\
			\leq    &     c\xi T\log T 
			+ \frac{2\beta eT\log T(\sigma + \xi + (\xi+6L) D_{max})}{m n_i} \\
			+ &\frac{2\beta eT}{m n_i}\sqrt{6\log T\Delta
				+ 6\log T(\xi^2 + \xi L)  
				+ \dfrac{\log T\sigma^2 \pi^2}{mK\beta}
				+ \log T\zeta(3)c^3\beta^2(\frac{27\sigma^2}{K\beta} + \frac{36\xi^2}{\beta} + \frac{18(\xi+6L)^2D_{\max}^2}{\beta})}.
		\end{aligned}
	\end{equation*}
	Then we can get the result
	\begin{equation*}
		\begin{aligned}
			\varepsilon_{gen} 
			\leq    &     L\mathbb{E}\Vert\theta_T-\theta_T'\Vert \\     
			\leq    &   \mathcal{O}(\xi T\log T+ \frac{T\sqrt{\log T\Delta} }{mn_{\min}}+\frac{T\sqrt{\log T} \sigma}{m n_{\min} K^{\frac{1}{2}}} +\frac{ T\log T(\xi+1)D_{max}}{mn_{\min}}+ \frac{T\log T \sigma}{mn_{\min}} )\\
			= & \mathcal{O}\left(\rho T\log T+\frac{T\sqrt{\log T \Delta}}{mn_{\mathrm{min}}}+\frac{T\log T(\rho+1)D_{\mathrm{max}}}{mn_{\mathrm{min}}}\right).
		\end{aligned}
	\end{equation*}
\end{proof}
\subsection{Generalization analysis of Randomized Smoothness approximation}
\begin{Lemma} \label{Lemma 9}
	(\citet{lin2022gradient}Under assumption 1 and let $\ell_\gamma(\theta)=\mathbb{E}_{u\sim\mathbb{P}}[\ell_\rho(\theta+\gamma u)]$ where $\mathbb{P}$ is an uniform distribution on an unit ball in $\mathcal{L}_2$ norm. We have 
	$$\|\ell_\gamma(\theta)-\ell_\rho(\theta)\|\leq\gamma L,$$
	$$\|\ell_\gamma(\theta) - \ell_\gamma(\theta^\prime)\|\leq L\|\theta-\theta^\prime\|,$$ 
	$$\|\nabla \ell_\gamma(\theta) - \nabla \ell_\gamma(\theta^\prime)\|\leq \frac{cL\sqrt{d}}\gamma\|\theta-\theta^\prime\|,$$
	where $c\geq0$ is a constant.
\end{Lemma}
\begin{proof}
	Since $\ell_\rho(\theta)$ is $L$-Lipschitz, we have
	\begin{equation*}
		\begin{aligned}
			\|\ell_\gamma(\theta)-\ell_\rho(\theta)\| =  &\|\mathbb{E}[\ell_\rho(\theta+\gamma u)-\ell_\rho(\theta)]\|\\
			\leq &\gamma L\mathbb{E}[\|u\|] \\
			\leq &\gamma L.
		\end{aligned}
	\end{equation*}
	Then we can prove $\ell_\gamma(\theta)$ is also $L$-Lipschitz,
	\begin{equation*}
		\begin{aligned}
			|\ell_\gamma(\theta)-\ell_\gamma(\theta^{\prime})|
			= &|\mathbb{E}[\ell_\rho(\theta+\gamma u)-\ell_\rho(\theta^{\prime}+\gamma u)]| \\
			\leq &L\mathbb{E}[\|\theta-\theta^{\prime}\|]| \\
			= &L\|\theta-\theta^{\prime}\|,
		\end{aligned}
	\end{equation*}
	for all $\theta,\theta^{\prime}\in\mathbb{R}^d.$\\
	If $\|\theta-\theta^\prime\| \geq 2\gamma$, we have
	$$\|\nabla \ell_\gamma(\theta)-\nabla \ell_\gamma(\theta^{\prime})\|\leq 2L \leq \frac{L\sqrt{d}\|\theta-\theta^{\prime}\|}\gamma.$$
	If $\|\theta-\theta^\prime\| \leq 2\gamma$, we have
	\begin{equation*}
		\begin{aligned}
			\|\nabla \ell_\gamma(\theta)-\nabla \ell_\gamma(\theta^{\prime})\|
			=&\|\mathbb{E}[\nabla \ell_\rho(\theta+\gamma u)]-\mathbb{E}[\nabla \ell_\rho(\theta^{\prime}+\gamma u)]\|\\
			=&\frac1{\mathrm{Vol}(\mathbb{B}_1(0))}\left|\int_{u\in\mathbb{B}_1(0)}\nabla \ell_\rho(\theta+\gamma u)du-\int_{u\in\mathbb{B}_1(0)}\nabla \ell_\rho(\theta^{\prime}+\gamma u)du\right|\\
			=&\frac1{\mathrm{Vol}(\mathbb{B}_\gamma(0))}\left|\int_{y\in\mathbb{B}_\gamma(\theta)}\nabla \ell_\rho(y)dy-\int_{y\in\mathbb{B}_\gamma(\theta^{\prime})}\nabla \ell_\rho(y)dy\right| \\
			=&\frac1{\mathrm{Vol}(\mathbb{B}_\gamma(0))}\left|\int_{y\in\mathbb{B}_\gamma(\theta)\setminus\mathbb{B}_\gamma(\theta^{\prime})}\nabla \ell_\rho(y)dy-\int_{y\in\mathbb{B}_\gamma(\theta^{\prime})\setminus\mathbb{B}_\gamma(\theta)}\nabla \ell_\rho(y)\mathrm{~}dy\right| \\
			\leq & \frac{L}{\mathrm{Vol}(\mathbb{B}_{\gamma}(0))}\left(\mathrm{Vol}(\mathbb{B}_{\gamma}(\theta)\setminus\mathbb{B}_{\gamma}(\theta^{\prime}))+\mathrm{Vol}(\mathbb{B}_{\gamma}(\theta^{\prime})\setminus\mathbb{B}_{\gamma}(\theta))\right) \\
			= & \frac{2L}{\mathrm{Vol}(\mathbb{B}_\gamma(0))}\mathrm{Vol}(\mathbb{B}_\gamma(\theta)\setminus\mathbb{B}_\gamma(\theta^{\prime})) \\
			= &\frac{2L}{c_d\gamma^d}\mathrm{Vol}(I),
		\end{aligned}
	\end{equation*}
	where $c_d=\frac{\pi^{d/2}}{\Gamma(d/2+1)}$, the first inequality used Assumption \ref{Assumption 1}. Let $V_{cap}(p)$ denote the volume of the spherical cap with the distance $p$ from the center of the sphere, we have $$\mathrm{Vol}(I)=\mathrm{Vol}(\mathbb{B}_\gamma(0))-2V_{cap}(\frac12\|\theta-\theta^{\prime}\|)=c_d\gamma^d-2V_{cap}(\frac12\|\theta-\theta^{\prime}\|).$$
	The volume of the $d$-dimensional spherical cap with distance $p$ from the center of the sphere can be calculated in terms of the volumes of $(d-1)$-dimensional spheres as follows:
	$$V_{cap}(p)=\int_p^\gamma c_{d-1}(\gamma^2-\rho^2)^{\frac{d-1}2}d\rho,\quad\text{for all }p\in[0,\gamma].$$
	Since $V_cap(\cdot)$ is a convex function over $[0,\Delta]$, we have $V_{cap}(p)\geq V_{cap}(0)+V_{cap}'(0)p$. By the definition, we have $V_{cap}(0)=\frac{1}{2}\mathrm{Vol}(\mathbb{B}_{\gamma}(0))=\frac{1}{2}c_{d}\gamma^{d}\mathrm{~and~}V_{cap}^{\prime}(0)=-c_{d-1}\gamma^{d-1}$. Thus, $V_{cap}(p)\geq\frac12c_d\gamma^d-c_{d-1}\gamma^{d-1}p$. Furthermore, $\frac12\|\theta-\theta'\|\in[0,\gamma]$. Putting these pieces together yields that $\mathrm{Vol}(I)\leq c_{d-1}\gamma^{d-1}\|\theta-\theta'\|$. Therefore, we conclude that  
	$$\|\nabla \ell_\gamma(\theta)-\nabla \ell_\gamma(\theta')\|\leq\frac{2L}{c_d\gamma^d}\mathrm{Vol}(I)\leq\frac{2c_{d-1}}{c_d}\frac{L\|\theta-\theta'\|}\gamma.$$
	Since $c_d=\frac{\pi^{d/2}}{\Gamma(d/2+1)}$, we have 
	\begin{equation*}
		\frac{2c_{d-1}}{c_d}=\left\{\begin{array}{ll}\frac{d!!}{(d-1)!!}&\text{if }d\text{ is odd,}\\\frac2\pi\frac{d!!}{(d-1)!!}&\text{otherwise.}\end{array}\right.
	\end{equation*}
	and $\frac1{\sqrt d}\frac{2c_{d-1}}{c_d}\to\sqrt{\frac\pi2}$. Therefore, we conclude that the gradient $\nabla _{\gamma}$ is $\frac{cL\sqrt{d}}\gamma $-Lipschitz where $c\geq 0$.
\end{proof}
\begin{Lemma} \label{Lemma 10}
	Suppose $\ell_\gamma(\theta)$ is a $\frac{cL\sqrt{d}}\gamma$-Lipschitz smooth, nonconvex function with respect ot $x$. Consider gradient descent operator $G_\alpha(\theta):=\theta-\eta\nabla \ell_\gamma(\theta)$. Then, 
	$$\|G_\alpha(\theta_1)-G_\alpha(\theta_2)\|\leq(1+\eta)\|\theta_1-\theta_2\|$$
\end{Lemma}
\begin{proof}
	Since f is $\frac{cL\sqrt{d}}\gamma$-smooth, we obtain
	\begin{equation*}
		\begin{aligned}
			\|G_\alpha(\theta_1)-G_\alpha(\theta_2)\| \leq & \|x-y-\eta(\nabla \ell_\gamma(\theta_1)-\nabla \ell_\gamma(\theta_2))\|\\
			\leq & \|x-y\| + \frac{cL\sqrt{d}}\gamma\eta\|\nabla \ell_\gamma(\theta_1)-\nabla \ell_\gamma(\theta_2)\| \\
			\leq & (1+\frac{cL\sqrt{d}}\gamma\eta) \|x-y\|.
		\end{aligned}
	\end{equation*}
\end{proof}
\begin{Lemma} \label{Lemma 11}
	Let $\mathcal{G}(\theta) = \frac{1}{Q}\sum_{q=1}^Q \nabla \ell_{\rho}(\theta+\gamma u_q,z_q)$, Consider gradient descent operator $G(\theta):=\theta-\eta\mathcal{G}(\theta)$. we have
\end{Lemma}
\begin{proof}
	\begin{equation*}
		\begin{aligned}
			\|G(\theta_1)-G(\theta_2)\| =& \|\theta_1-\theta_2-\eta(\mathcal{G}(\theta_1)-\mathcal{G}(\theta_2))\|\\
			\leq& \|\theta_1-\theta_2-\eta(\nabla \ell_\gamma(\theta_1)-\nabla \ell_\gamma(\theta_2)\| + \eta\|\mathcal{G}(\theta_1)-\nabla \ell_\gamma(\theta_1)\| + \eta\|\mathcal{G}(\theta_2)- \nabla \ell_\gamma(\theta_2)\|\\
			\leq & (1+\eta)\frac{cL\sqrt{d}}\gamma\| \theta_1-\theta_2 \| + 2\eta v,
		\end{aligned}
	\end{equation*}
	where $v:= \|\mathcal{G}(\theta) - \ell_\gamma(\theta)\|$ is the error in the gradinet estimate.
\end{proof}
since the objective function $\ell_\gamma$ is $\frac{cL\sqrt{d}}\gamma$ gradient-Lipschitz, we can get the following result from \citet{li2023federated,karimireddy2020scaffold}.
\begin{Theorem} \label{Theorem 7}
	If the Assumptions \ref{Assumption 1} and \ref{Assumption 4} hold and $\eta_t\leq \frac{\gamma}{6KcL\sqrt{d}}$, considering $\theta_{i,k+1}=\theta_{i,k}-\eta_t\nabla \ell_{\gamma,i}(\theta_{i,k})$, we have
	$$\frac{1}{T}\sum_{t=0}^{T-1} \mathbb{E}\|R(\theta_t)\|^2\leq \mathcal{O}(\frac{\sqrt{\sigma}}{\sqrt{mKT}}+\frac{((\xi+1)D_{\max})^{\frac23}}{T^{\frac23}}+\frac\Delta T).$$
	where $\Delta = R(\theta_0)-R(\theta^*)$
\end{Theorem} 
\begin{Lemma} \label{Lemma 12}
	Suppose $\ell_\gamma(\theta)$ is a $\frac{cL\sqrt{d}}\gamma$-gradient Lipschitz and non-convex function w.r.t $\theta$.
	$$\mathbb{E}\Vert \theta_{i,k} - \theta_{t}\Vert \leq (1 + \frac{cL\sqrt{d}}\gamma \eta_t)^{K - 1}K\eta_{t}( \sigma + (\xi+6L)D_i + \mathbb{E}\Vert\nabla R(\theta_{t})\Vert).$$
\end{Lemma}
\begin{proof} Considering $\theta_{i,k+1}=\theta_{i,k}-\eta_t\nabla \ell_{\gamma,i}(\theta_{i,k})$
	\begin{equation*}
		\begin{aligned}
			\mathbb{E}\Vert\theta_{i,k+1} - \theta_{t}\Vert 
			=       &     \mathbb{E}\Vert \theta_{i,k} - \eta_t(\nabla \ell_{\gamma,i}(\theta_{i,k}) - \nabla \ell_{\gamma,i}(\theta_{t})))\Vert + \eta_t\mathbb{E}\Vert \nabla \ell_{\gamma,i}(\theta_t)\Vert                                                \\
			\leq    &     (1+\frac{cL\sqrt{d}}\gamma\eta_t)\mathbb{E}\Vert\theta_{i,k}-\theta_{t}\Vert           
			+ \eta_t( \mathbb{E}\Vert \nabla \ell_{\gamma,i}-\nabla R_i(\theta_{t})\Vert +\mathbb{E}\Vert \nabla R_i(\theta_{t}) -\nabla R(\theta_t)\Vert + \mathbb{E}\Vert\nabla R(\theta_t)\Vert) \\
			\leq    &     (1+\frac{cL\sqrt{d}}\gamma\eta_t)\mathbb{E}\Vert\theta_{i,k}-\theta_{t}\Vert + \eta_t( \sigma + (\xi + 6L)D_i + \mathbb{E}\Vert\nabla R(\theta_{t}) \Vert),
		\end{aligned}
	\end{equation*}
	where we use Lemma \ref{Lemma 10} in the first inequality, Assumption \ref{Assumption 4} and Lemma \ref{Lemma 1}  in the second inequality. Unrolling the above and noting $\theta_{i,0} = \theta_{t}$ yields
	\begin{equation*}
		\begin{aligned}
			\mathbb{E}\Vert\theta_{i,k}-\theta_{t}\Vert 
			\leq    &     \sum\limits_{l=0}^{k-1} \eta_{i,l}(\sigma + (\xi+6L)D_i+\mathbb{E}\Vert\nabla R(\theta_{t})\Vert)(1 + \frac{cL\sqrt{d}}\gamma \eta_t)^{k-1-l}     \\
			\leq    &     \sum\limits_{l=0}^{K-1} \eta_{i,l}( \sigma + (\xi+6L)D_i + \mathbb{E}\Vert\nabla R(\theta_{t})\Vert)(1 + \frac{cL\sqrt{d}}\gamma \eta_t)^{K-1} \\
			\leq    &     (1 + \frac{cL\sqrt{d}}\gamma \eta_t)^{K - 1} K\eta_{t}( \sigma + (\xi+6L)D_i + \mathbb{E}\Vert\nabla R(\theta_{t})\Vert). \\
		\end{aligned}
	\end{equation*}
\end{proof} 
\begin{Lemma} \label{Lemma 13}
	Suppose the Assumptions \ref{Assumption 1} and \ref{Assumption 4} hold, $\ell_\gamma(\theta)$ is a $\frac{cL\sqrt{d}}\gamma$-gradient Lipschitz and non-convex function w.r.t $\theta$.
	$$\mathbb{E}\|\nabla \ell_{\gamma,i}(\theta_{i,k})\| \leq\big(1+(1+\frac{cL\sqrt{d}}\gamma K\eta_t)\big)\big(\mathbb{E}\|\nabla R(\theta_t)\|+(\xi+6L)D_i+\sigma\big).$$
\end{Lemma}
\begin{proof}
	Considering $\theta_{i,k+1}=\theta_{i,k}-\eta_t\nabla \ell_{\gamma,i}(\theta_{i,k})$, we obtain
	\begin{equation*}
		\begin{aligned}
			\mathbb{E}\|\nabla \ell_{\gamma,i}(\theta_{i,k})\|& \leq\mathbb{E}\|\nabla \ell_{\gamma,i}(\theta_{i,k})-\nabla R_i(\theta_{i,k})\|+\mathbb{E}\|\nabla R_i(\theta_{i,k})\|  \\
			&\leq\mathbb{E}\|\nabla R_i(\theta_{i,k})\|+\sigma  \\
			&\leq\mathbb{E}\|\nabla R_i(\theta_t)\|+\mathbb{E}\|\nabla R_i(\theta_{i,k})-\nabla R_i(\theta_t)\|+\sigma  \\
			&\leq\mathbb{E}\|\nabla R(\theta_t)\|+\mathbb{E}\|\nabla R_i(\theta_t)-\nabla R(\theta_t)\|+\frac{cL\sqrt{d}}\gamma\mathbb{E}\|\theta_{i,k}-\theta_t\|+\sigma  \\
			&\leq\big(1+(1+\frac{cL\sqrt{d}}\gamma K\eta_t)\big)\big(\mathbb{E}\|\nabla R(\theta_t)\|+(\xi+6L)D_i+\sigma\big).
		\end{aligned}
	\end{equation*} 
	where we use Assumption \ref{Assumption 4} in the first inequality, Lemma \ref{Lemma 10} in the third inequality, Lemma \ref{Lemma 1} in the last inequality.
\end{proof}
Next we give the proof of Theorem \ref{Theorem 3}.\\

\textbf{Theorem 3}. \textit{Let $c \geq 0$ be a constant and the step size be chosen as $\eta_t \leq \frac{\gamma}{4K\sqrt{d}cL(t+1)}$. Under Assumption \ref{Assumption 1} and \ref{Assumption 4}, the generalization bound $\varepsilon_{gen}$ with the randomized smoothness approximation satisfies: }
\begin{equation*}
	\mathcal{O}\left(\frac{T^{\frac{1}{4}}\log T}{\sqrt{Q}}+\frac{T^{\frac{3}{4}}\sqrt{\Delta}}{mn_{\mathrm{min}}}+\frac{T((\rho+1)D_{\mathrm{max}})^{\frac{1}{3}}}{mn_{\mathrm{min}}}\right),
\end{equation*}
where $\Delta = \mathbb{E}[R(\theta_0)] - \mathbb{E}[R(\theta^*)], D_{\max}=\max_{i\in [m]} D_i$ and $n_{\min}$ is the smallest data size on client $i$, $\forall i \in [m]$.
\begin{proof}
	Similar to  that proof of Theorem \ref{Theorem 2}, Given time index $t$ and for client $j$ with $j\ne i$, we have
	
	\begin{equation*}
		\begin{aligned}
			\mathbb{E}\Vert\theta_{j,k +1}-\theta_{j,k +1}'\Vert
			=       &     \mathbb{E}\Vert\theta_{j,k}-\theta_{j,k}'-\eta_t(\mathcal{G}_j(\theta_{j,k})-\mathcal{G}_j(\theta_{j,k}))\Vert \\
			\leq    &     (1+\frac{cL\sqrt{d}}\gamma\eta_t)\mathbb{E}\Vert \theta_{j,k}-\theta_{j,k}'\Vert+2\eta_t\mathbb{E}[v],
		\end{aligned}
	\end{equation*}
	where we use Lemma \ref{Lemma 11} in the above inequality.
	And unrolling it gives
	\begin{equation}\label{Theorem 3.1}
		\begin{aligned}
			\mathbb{E}\Vert\theta_{j,K-1}-\theta_{j,K-1}'\Vert
			\leq    &     \prod_{k=0}^{K-1}(1+\frac{cL\sqrt{d}}\gamma\eta_t)\mathbb{E}\Vert\theta_t-\theta_t'\Vert+2\sum\limits_{k=0}^{K-1}\eta_t\prod_{l=k+1}^{K-1}(1+\frac{cL\sqrt{d}}\gamma\eta_l)\mathbb{E}[v] \\
			\leq    &     e^{\frac{cL\sqrt{d}}\gamma K\eta_{t}}\mathbb{E}\Vert\theta_t-\theta_t'\Vert+2e^{\frac{cL\sqrt{d}}\gamma K\eta_{t}}K\eta_{t}\mathbb{E}[v]                                \\
			=       &     e^{\frac{cL\sqrt{d}}\gamma K\eta_{t}}(\mathbb{E}\Vert\theta_t-\theta_t'\Vert+2K\eta_{t}\mathbb{E}[v]),
		\end{aligned}
	\end{equation}
	where we use $1+x\le e^x,\forall x$.
	For client $i$, there are two cases to consider. In the first case, SGD selects non-perturbed samples in $\mathcal{S}$ and $\mathcal{S}^{(i)}$, which happens with probability $1-\frac{1}{n_i}$. Then, we have
	\begin{equation*}
		\begin{aligned}
			\Vert\theta_{i,k+1}-\theta_{i,k+1}'\Vert
			=       &     \Vert\theta_{i,k}-\theta_{i,k}'-\eta_t(\mathcal{G}_i(\theta_{i,k})-\mathcal{G}_i(\theta_{i,k}')) \\
			\leq    &     (1+\frac{cL\sqrt{d}}\gamma\eta_t)\Vert\theta_{i,k}-\theta_{i,k}'\Vert+2\eta_t\mathbb{E}[v],
		\end{aligned}
	\end{equation*}
	where we use Lemma \ref{Lemma 11} in the above inequality.
	In the second case, SGD encounters the perturbed sample at time step $k$, which happens with probability $\frac{1}{n_i}$. Then, we have
	\begin{equation*}
		\begin{aligned}
			\Vert\theta_{i,k+1}-\theta_{i,k+1}'\Vert
			=       &     \Vert\theta_{i,k}-\theta_{i,k}'-\eta_t(\mathcal{G}_i(\theta_{i,k})-\mathcal{G}_i'(\theta_{i,k}'))\Vert         \\
			\leq    &      \Vert\theta_{i,k}-\theta_{i,k}'-\eta_t(\nabla \ell_{\gamma,i}(\theta_{i,k})-\nabla \ell_{\gamma,i}'(\theta_{i,k}'))\Vert + 2\eta_t v\\
			\leq    &     \Vert\theta_{i,k}-\theta_{i,k}'-\eta_t(\nabla \ell_{\gamma,i}(\theta_{i,k})-\nabla \ell_{\gamma,i}(\theta_{i,k}'))\Vert+\eta_t\Vert \nabla \ell_{\gamma,i}(\theta_{i,k}')-\nabla \ell_{\gamma,i}'(\theta_{i,k}')\Vert +2\eta_t v\\
			\leq    &     (1+\frac{cL\sqrt{d}}\gamma\eta_t)\Vert\theta_{i,k}-\theta_{i,k}'\Vert+2\eta_tv+\eta_t\Vert \nabla \ell_{\gamma,i}(\theta_{i,k}')-\nabla \ell_{\gamma,i}'(\theta_{i,k}')\Vert.
		\end{aligned}
	\end{equation*}
	Combining these two cases for client i, we have
	\begin{equation*}
		\begin{aligned}
			\mathbb{E}\Vert\theta_{i,k+1}-\theta_{i,k+1}'\Vert 
			\leq    &     (1-\frac{1}{n_i})\left[(1+\frac{cL\sqrt{d}}\gamma\eta_t)\mathbb{E}\Vert\theta_{i,k}-\theta_{i,k}'\Vert+2\eta_t\mathbb{E}[v]\right]                                                                       \\
			&   + \frac{1}{n_i}\left[  (1+\frac{cL\sqrt{d}}\gamma\eta_t)\mathbb{E}\Vert\theta_{i,k}-\theta_{i,k}'\Vert+2\eta_t\mathbb{E}[v]+\eta_t\mathbb{E}\Vert \nabla \ell_{\gamma,i}(\theta_{i,k}')-\nabla \ell_{\gamma,i}'(\theta_{i,k}')\Vert\right]   \\
			\leq    &     (1+\frac{cL\sqrt{d}}\gamma\eta_t)\mathbb{E}\Vert\theta_{i,k}-\theta_{i,k}'\Vert+2\eta_t\mathbb{E}[v]+\frac{\eta_t}{n_i}\mathbb{E}\Vert \nabla \ell_{\gamma,i}(\theta_{i,k}')-\nabla \ell_{\gamma,i}'(\theta_{i,k}')\Vert \\
			\leq    &     (1+\frac{cL\sqrt{d}}\gamma\eta_t)\mathbb{E}\Vert\theta_{i,k}-\theta_{i,k}'\Vert+2\eta_t\mathbb{E}[v]+\frac{2\eta_t}{n_i}\mathbb{E}\Vert \nabla \ell_{\gamma,i}(\theta_{i,k})\Vert                  .           \\
		\end{aligned}
	\end{equation*}
	
	Then unrolling it gives,
	\begin{equation}\label{Theorem 3.2}
		\begin{aligned}
			\mathbb{E}\Vert\theta_{i,K}-\theta_{i,K}'\Vert
			\leq    &     \prod_{k=0}^{K-1}(1+\frac{cL\sqrt{d}}\gamma\eta_t)\mathbb{E}\Vert\theta_t-\theta_t'\Vert 
			+ \sum\limits_{k=0}^{K-1}\eta_t\prod_{l=k+1}^{K-1}(1+\frac{cL\sqrt{d}}\gamma\eta_{i,l})(2\mathbb{E}[v] + \dfrac{2}{n_i}\mathbb{E} \Vert \nabla \ell_{\gamma,i}(\theta_{i,k}) \Vert)                                                                                    \\
			\leq    &     e^{\frac{cL\sqrt{d}}\gamma\sum\limits_{k=0}^{K-1}\eta_t}\mathbb{E}\Vert\theta_{t}-\theta_{t}'\Vert 
			+ K\eta_{t} e^{\frac{cL\sqrt{d}}\gamma\sum\limits_{k=0}^{K-1}\eta_t}( 2\mathbb{E}[v] + \dfrac{2}{n_i}\mathbb{E} \Vert \nabla \ell_{\gamma,i}(\theta_{i,k}) \Vert) \\
			=       &     e^{\frac{cL\sqrt{d}}\gamma K\eta_{t}}\mathbb{E}\Vert\theta_{t}-\theta_{t}'\Vert 
			+ K\eta_{t} e^{\frac{cL\sqrt{d}}\gamma K\eta_{t}}(2\mathbb{E}[v] + \dfrac{2}{n_i}\mathbb{E} \Vert \nabla \ell_{\gamma,i}(\theta_{i,k}) \Vert ) .                              \\
		\end{aligned}
	\end{equation}
	By \ref{Theorem 3.1} and \ref{Theorem 3.2} we have
	\begin{equation*}
		\begin{aligned}
			\mathbb{E}\Vert\theta_{t+1}-\theta_{t+1}'\Vert
			=       &     \mathbb{E}\Vert \dfrac{1}{m} \sum\limits_{j=1}^m(\theta_{j,K} - \theta_{j,K}') \Vert  \\
			\leq    &     \dfrac{1}{m} \sum\limits_{j = 1}^m \mathbb{E}\Vert\theta_{j,K} - \theta_{j,K}' \Vert \\
			=       &     \dfrac{1}{m} \sum\limits_{j = 1,j \neq i}^m \mathbb{E}\Vert\theta_{j,K} - \theta_{j,K}' \Vert + \dfrac{1}{m}\mathbb{E}\Vert \theta_{i,K} - \theta_{i,K}' \Vert\\         
			\leq    &     \dfrac{1}{m} \sum\limits_{j = 1}^{m}  e^{\frac{cL\sqrt{d}}\gamma K\eta_t}\mathbb{E}\Vert \theta_{t}-\theta_{t}'\Vert 
			+ \dfrac{2}{m} \sum\limits_{j = 1}^{m}  e^{\frac{cL\sqrt{d}}\gamma K\eta_t} K\eta_t \mathbb{E}[v]  
			+ \dfrac{2 K\eta_t e^{\frac{cL\sqrt{d}}\gamma K\eta_t}}{m n_i}\mathbb{E} \Vert \nabla \ell_{\gamma,i}(\theta_{i,k}) \Vert .
		\end{aligned}
	\end{equation*}
	Unrolling it over $t$ and noting $\theta_0 = \theta_0'$ and let $\eta_t\leq \frac{\gamma}{4KcL\sqrt{d}(t+1)}$, we obtain
	\begin{equation*}
		\begin{aligned}
			&\mathbb{E}\Vert\theta_{T}-\theta_{T}'\Vert \\
			\leq    &     \sum\limits_{t = 0}^{T - 1} \exp(\sum\limits_{l = t + 1}^{T - 1} \frac{1}{4l} )\left[\sum_{j = 1}^{m} \dfrac{2K\eta_t\mathbb{E}[v]}{m} +  \dfrac{2 K\eta_t}{m n_i} \mathbb{E} \Vert \nabla \ell_{\gamma,i}(\theta_{i,k}) \Vert\right]\\ 
			\leq    &     T^\frac{1}{4} \sum_{t = 0}^{T - 1} \left[\sum_{j = 1}^{m} \dfrac{2K\eta_t\mathbb{E}[v]}{m}
			+ \dfrac{2 K\eta_t}{m n_i} \mathbb{E} \Vert \nabla \ell_{\gamma,i}(\theta_{i,k}) \Vert \right] \\
			\leq    &     T^\frac{1}{4} \sum_{t = 0}^{T - 1} \sum_{j = 1}^{m} \dfrac{K\eta_t \mathbb{E}[v]}{m} 
			+ T^\frac{1}{4} \sum_{t = 0}^{T - 1}\dfrac{2 K\eta_t}{m n_i}\big(1+(1+\frac{cL\sqrt{d}}\gamma K\eta_t)\big)\big(\mathbb{E}\|\nabla R(\theta_t)\|+(\xi+6L)D_i+\sigma\big)   \\
			\leq    &   \mathcal{O}(\frac{T^\frac{1}{4}\log T}{\sqrt{Q}}+ \frac{T^\frac{3}{4}\sqrt{\Delta}}{mn_{\min}} + \frac{T \sigma^\frac{1}{4}}{m^\frac{5}{4}n_{\min}K^\frac{1}{4}}  + \frac{T((\xi+1)D_{\max})^\frac{1}{3}}{mn_{\min}}+\frac{T^\frac{1}{4}\log T\sigma}{mn_{\min}}) \\
			=       & \mathcal{O}\left(\frac{T^{\frac{1}{4}}\log T}{\sqrt{Q}}+\frac{T^{\frac{3}{4}}\sqrt{\Delta}}{mn_{\mathrm{min}}}+\frac{T((\rho+1)D_{\mathrm{max}})^{\frac{1}{3}}}{mn_{\mathrm{min}}}\right),
		\end{aligned}
	\end{equation*}
	where we use Lemma \ref{Lemma 13} in the third inequality. Additionally, we use $\mathbb{E}[v]\leq \frac{1}{\sqrt{Q}}$ which comes from \cite{duchi2012randomized}.
\end{proof}
\subsection{Generalization analysis of Over-Parameterized Smoothness approximation}
In the following denote $\ell(\theta,(\tilde{x},y))=\ell(W,(\tilde{x},y))\stackrel{\mathrm{def}}{=}\frac{1}{2}(f_{W}(\tilde{x})-y)^2$. Unless stated otherwise, we work with vectorised quantities so $W\in \mathbb{R}^{ds}$. We also use notation $(W)_{\tau}$ so selected $\tau$-th block of size d, that is $(W)_{\tau}=[W_{(d-1)\tau+1},...,W_{d\tau}]^\top$. 
\begin{Lemma}
	Suppose Assumption \ref{Assumption 2}-\ref{Assumption 3} hold, we have  
	$$\|\nabla^2_{WW}\ell(W,\tilde{z})\|_2 \leq \zeta_\theta,
	$$
	where $\zeta_\theta = 2(C_x^2+\rho^2)B_{\phi^{\prime}}^2+\frac{2(C_x^2+\rho^2)B_{\phi^{\prime\prime}}}{\sqrt{s}}(\sqrt{s}B_\phi+C_y).$
\end{Lemma}
\begin{proof}
	Vectorising allows the loss's Hessian to be denoted 
	$$\nabla_{W,W}^2\ell(W,\tilde{z})=\nabla f_W(\tilde{x})\nabla f_W(\tilde{x})^\top+\nabla^2f_W(\tilde{x})(f_W(\tilde{x})-y),$$
	where 
	\begin{equation*}
		\nabla_{\theta} f_W(\tilde{x})=\begin{pmatrix}\mu_1\tilde{x}\phi'\left(\langle(W)_1,\tilde{x}\rangle\right)\\\mu_2\tilde{x}\phi'\left(\langle(W)_2,\tilde{x}\rangle\right)\\\vdots\\\mu_s\tilde{x}\phi'\left(\langle(W)_s,\tilde{x}\rangle\right)\end{pmatrix}\in\mathbb{R}^{ds},
	\end{equation*}
	and $\nabla^2f_W(\tilde{x})\in\mathbb{R}^{ds\times ds}$ with
	\begin{equation*}
		\nabla^2_{WW}f_{W}(\tilde{x})=\begin{pmatrix}\mu_1\tilde{x}\tilde{x}^\top\phi''(\langle(W)_1,\tilde{x}\rangle)&0&0&\dots&0\\0&\mu_2\tilde{x}\tilde{x}^\top\phi''(\langle(W)_2,\tilde{x}\rangle)&0&\dots&0\\\vdots&\ddots&\ddots&\vdots&\vdots\\0&0&0&\dots&\mu_s\tilde{x}\tilde{x}^\top\phi''(\langle(W)_s,\tilde{x}\rangle)\end{pmatrix}.
	\end{equation*}
	Note that we then immediately have with $v=(v_1, v_2,\ldots,v_s)\in\mathbb{R}^{ds}$ with $v_i\in\mathbb{R}^d$
	\begin{equation*}
		\begin{aligned}
			\|\nabla^2_{WW}f_W(\tilde{x})\|_2& =\max_{v:\|v\|_2\leq1}\sum_{\tau=1}^s\mu_\tau\langle v_\tau,\tilde{x}\rangle^2\phi''(\langle(W)_\tau,\tilde{x}\rangle)  \\
			&\leq\frac1{\sqrt{s}}\|\tilde{x}\|_2^2B_{\phi^{\prime\prime}}\max_{v:\|v\|_2\leq1}\sum_{\tau=1}^s\|v_\tau\|_2^2 \\
			&\leq\frac{2(C_x^2+\rho^2)B_{\phi^{\prime\prime}}}{\sqrt{s}}.
		\end{aligned}
	\end{equation*}
	We then see that the maximum Eigenvalue of the Hessian is upper bounded for any $W\in\mathbb{R}^{ds}$, that is 
	\begin{equation*}
		\begin{aligned}
			\|\nabla^2_{WW}\ell(W,\tilde{z})\|_2& \begin{aligned}\leq\|\nabla_{W} f_W(\tilde{x})\|_2^2+\|\nabla^2_{WW}f_W(\tilde{x})\|_2|f_W(\tilde{x})-y|\end{aligned}  \\
			&\leq 2(C_x^2+\rho^2)B_{\phi^{\prime}}^2+\frac{2(C_x^2+\rho^2)B_{\phi^{\prime\prime}}}{\sqrt{s}}(\sqrt{s}B_\phi+C_y),
		\end{aligned}
	\end{equation*}
	and therefore the objective is $\zeta_\theta $-smooth with $\zeta_\theta=2(C_x^2+\rho^2)\big(B_{\phi^{\prime}}^2+B_{\phi^{\prime\prime}}B_\phi+\frac{B_{\phi^{\prime\prime}}C_y}{\sqrt{s}}\big).$ 
\end{proof}
\begin{Lemma}
	Suppose Assumption \ref{Assumption 2}-\ref{Assumption 3} hold, we have  
	$$ \|\nabla^2_{Wx}\ell(W,\tilde{z})\|_2 \leq \zeta_x,$$
	where $\zeta_x = (C_x+\rho)C_W B_{\phi^{\prime}}^2 + \big(B_{\phi^{\prime}}+(C_x+\rho)C_W B_{\phi^{\prime\prime}}\big)(\sqrt{s}B_\phi+C_y)$.
\end{Lemma}
\begin{proof}
	Vectorising allows the loss's Hessian to be denoted 
	$$\nabla_{Wx}^2\ell(W,\tilde{z})=\nabla_W f_W(\tilde{x})\nabla_x f_W(\tilde{x})^\top+\nabla^2_{Wx}f_W(\tilde{x})(f_W(\tilde{x})-y),$$
	where 
	\begin{equation*}
		\nabla_{x} f_W(\tilde{x})=\begin{pmatrix}\mu_1(W)_1\phi'\left(\langle(W)_1,\tilde{x}\rangle\right)\\\mu_2(W)_2\phi'\left(\langle(W)_2,\tilde{x}\rangle\right)\\\vdots\\\mu_r(W)_r\phi'\left(\langle(W)_d,\tilde{x}\rangle\right)\end{pmatrix}\in\mathbb{R}^{ds},
	\end{equation*}
	and $\nabla^2_{W x}f_W(\tilde{x})\in\mathbb{R}^{s\times d\times d}$ with
	\begin{equation*}
		\nabla^2_{Wx}f_{W}(\tilde{x})=\begin{pmatrix}\mu_1\left(\phi^{\prime\prime}(\langle(W)_1,\tilde{x}\rangle)W_1\tilde{x}^\top+\phi^{\prime}(\langle(W)_1,\tilde{x}\rangle)\mathbf{I}\right)\\\mu_2\left(\phi^{\prime\prime}(\langle(W)_2,\tilde{x}\rangle)W_2\tilde{x}^\top+\phi^{\prime}(\langle(W)_2,\tilde{x}\rangle)\mathbf{I}\right)\\\vdots\\\mu_s\left(\phi^{\prime\prime}(\langle(W)_s,\tilde{x}\rangle)W_s\tilde{x}^\top+\phi^{\prime}(\langle(W)_s,\tilde{x}\rangle)\mathbf{I}\right)\end{pmatrix}.
	\end{equation*}
	Then using Assumption we obtain 
	\begin{equation*}
		\begin{aligned}
			\|\nabla^2_{Wx}f_{W}(\tilde{x})\|_2
			\leq &\sqrt{\sum_{k=1}^s\left(|\mu_k|\left(B_{\phi^{\prime}}+(C_x+\rho)B_{\phi^{\prime\prime}}\|W_k\|_2\right)\right)^2} \\
			\leq & \left(B_{\phi^{\prime}}+(C_x+\rho)B_{\phi^{\prime\prime}}C_W\right)\sqrt{\sum_{k=1}^s\mu_k^2} \\
			\leq & B_{\phi^{\prime}}+(C_x+\rho)C_W B_{\phi^{\prime\prime}}.
		\end{aligned}
	\end{equation*}
	Then we completes the proof
	\begin{equation*}
		\begin{aligned}
			\|\nabla^2_{Wx}\ell(W,\tilde{z})\|_2 
			\leq & \|\nabla_{W} f_W(\tilde{x})\|_2\|\nabla_{x} f_W(\tilde{x})\|_2 + \|\nabla^2_{Wx}f_W(\tilde{x})\|_2|f_W(\tilde{x})-y| \\
			\leq & (C_x+\rho)C_W B_{\phi^{\prime}}^2 + \big(B_{\phi^{\prime}}+(C_x+\rho)C_W B_{\phi^{\prime\prime}}\big)(\sqrt{s}B_\phi+C_y),
		\end{aligned}
	\end{equation*}
	and therefor it's $\zeta_x $-smooth where $\zeta_x = (C_x+\rho)C_W B_{\phi^{\prime}}^2 + \big(B_{\phi^{\prime}}+(C_x+\rho)C_W B_{\phi^{\prime\prime}}\big)(\sqrt{s}B_\phi+C_y)$ and is increasing with $\mathcal{O}(\rho\sqrt{s})$.
\end{proof}
\begin{Lemma}
	Under Assumption \ref{Assumption 1}-\ref{Assumption 3} and given $i\in [m]$, for any $\theta$ we have
	$$\Vert\nabla R_i(\theta_{t})-\nabla R(\theta_{t})\Vert\leq (\zeta_x+6L)D_i,$$
	where $D_i$ = $\max \{d_{TV}(\tilde{P}_i, P_i), d_{TV}(P_i, P), d_{TV}(\tilde{P},P)\}$. 
\end{Lemma}
\begin{proof}
	This proof is similar to \ref{Lemma 6},
	\begin{equation*}
		\begin{aligned}
			&\Vert \nabla R_{i}(\theta)-\nabla R(\theta)\Vert \\  
			\leq    &      \int_{\mathcal{Z}_{i} \cup \mathcal{Z}}\Vert\nabla \ell(\theta ; \tilde{z}) d\tilde{P_{i}} - \nabla \ell(\theta ; \tilde{z}) dP_{i} \Vert 
			+  \int_{\mathcal{Z}_{i} \cup \mathcal{Z}}\Vert\nabla \ell(\theta ; \tilde{z}) dP_{i}-\nabla \ell(\theta ; z) dP_{i} + \nabla\ell(\theta;z)dP - \nabla\ell(\theta;\tilde z)dP\Vert \\
			&   +  \int_{\mathcal{Z}_{i} \cup \mathcal{Z}}\Vert\nabla \ell(\theta ; z) dP_{i}-\nabla \ell(\theta ; z) dP \Vert 
			+  \int_{\mathcal{Z}_{i} \cup \mathcal{Z}}\Vert\nabla \ell(\theta ; z) dP-\nabla \ell(\theta ; z) d\tilde{P}  \Vert \\  
			\leq    &      \int_{\mathcal{Z}_{i} \cup \mathcal{Z}} L\Vert d\tilde{P_i} - P_i \Vert 
			+  \int_{\mathcal{Z}_{i} \cup \mathcal{Z}} \zeta_x \Vert \tilde{z} - z \Vert \Vert dP_i - dP \Vert
			+  \int_{\mathcal{Z}_{i} \cup \mathcal{Z}} L \Vert P_i - P \Vert +  \int_{\mathcal{Z}_{i} \cup \mathcal{Z}} L \Vert \tilde{P} - P \Vert \\
			\leq    &      2Ld_{TV}(d\tilde{P}_i, dP_i)               
			+  (2\zeta_x + 2L) d_{TV}(dP_i, dP)
			+  2Ld_{TV}(d\tilde{P},dP) \\
			\leq    &   (2\zeta_x+6L)D_i  ,
		\end{aligned}
	\end{equation*} 
	where $\zeta_x= (C_x+\rho)C_\theta B_{\phi^{\prime}}^2 + \big(B_{\phi^{\prime}}+(C_x+\rho)C_\theta B_{\phi^{\prime\prime}}\big)(\sqrt{s}B_\phi+C_y)$.
\end{proof}
Following the similar proof of Lemma \ref{Lemma 12} and Lemma \ref{Lemma 13}, we can the following two Lemmas.
\begin{Lemma}\label{Lemma 17}
	Suppose $\ell_\rho(\theta)$ is a $\zeta_\theta$-gradient Lipschitz and non-convex function w.r.t $\theta$, we have
	$$\mathbb{E}\Vert \theta_{i,k} - \theta_{t}\Vert \leq (1 + \zeta_\theta \eta_t)^{K - 1}K\eta_{t}( \sigma + (2\zeta_x+6L)D_i + \mathbb{E}\Vert\nabla R(\theta_{t})\Vert).$$
\end{Lemma}
\begin{Lemma}\label{Lemma 18}
	Suppose $\ell_\rho(\theta)$ is a $\zeta_\theta$-gradient Lipschitz and non-convex function w.r.t $\theta$. we have
	$$\mathbb{E}\|g(\theta_{i,k})\| \leq(1+(1+\zeta_\theta K\eta_t))\big(\mathbb{E}\|\nabla R(\theta_t)\|+(2\zeta_x+6L)D_i+\sigma\big).$$
\end{Lemma}
\begin{Lemma}\label{Lemma 19}
	(Almost Co-coercivity of the Gradient Operator\cite{lei2022stability}). Suppose Assumption \ref{Assumption 2}-\ref{Assumption 3} hold, if $\eta \leq \frac{1}{2\zeta_\theta}$, then we have
	\begin{equation*}
		\begin{aligned}\langle\theta_t-\theta_t^{(i)},\ell(\theta_t;z_i)-\nabla\ell(\theta_t^{(i)};z_i)\rangle&\geq2\eta\Big(1-\frac{\eta\zeta_\theta}2\Big)\|\nabla\ell(\theta_t;z_i)-\nabla\ell(\theta_t^{(i)};z_i)\|_2^2\\&-\varsigma_t\Big\|\theta_t-\theta_t^{(i)}-\eta\big(\nabla\ell(\theta_t;z_i)-\nabla\ell(\theta_t^{(i)};z_i)\big)\Big\|_2^2,\end{aligned}
	\end{equation*}
	where
	\begin{equation*}
		\begin{aligned}
			\varsigma_t =\frac{C_{x}^{2}B_{\phi^{\prime\prime}}}{\sqrt{m}}\Big(B_{\phi^{\prime}}C_{x}(1+2\eta\zeta_\theta)\max\{\|\theta_{t}-\theta_{0}\|_{2},\|\theta_{t}^{(i)}-\theta_{0}\|_{2}\}+\sqrt{2C_{0}}\Big).
		\end{aligned}
	\end{equation*}
\end{Lemma}
Next we give the proof of Theorem \ref{Theorem 4}, for simplicity, we use $g(\theta)$ to denote strochastic gradient of loss function $\ell_\rho$\\
\textbf{Theorem 4.} \textit{ Without loss of generality, we assume $4K\eta_t C_0 \geq 1$ and $s\geq 16\eta_t^2T^2K^2(b'H_K)^2(1+2\eta_t\zeta_\theta)^2$,where $b'= C_x^2B_{\varphi''}(C_xB_{\varphi'}+\sqrt{2C_0}),H_K=2\sqrt{K\eta_t C_0}$. Let the step size be chosen as $\eta_t\leq \frac{1}{6K\zeta_\theta}$. Under Assumptions \ref{Assumption 1}-\ref{Assumption 4}, the generalization bound $\varepsilon_{gen}$ with the over-parameterized smoothness approximation satisfies:}
\begin{equation*}
	\mathcal{O}\left(\frac{T^{\frac{1}{2}}\sqrt{\Delta}}{mn_{\mathrm{min}}}+\frac{T(\rho^2\sqrt{s}+1)D_{\mathrm{max}}}{mn_{\mathrm{min}}}\right),
\end{equation*}
where $\Delta = \mathbb{E}[R(\theta_0)] - \mathbb{E}[R(\theta^*)], D_{\max}=\max_{i\in [m]} D_i$ and $n_{\min}$ is the smallest data size on client $i$, $\forall i \in [m]$.
\begin{proof}
	Similar to  that proof of Theorem \ref{Theorem 2}, Given time index $t$ and for client $j$ with $j\ne i$, we have
	\begin{equation*}
		\begin{aligned}
			\Vert\theta_{j,k +1}-\theta_{j,k +1}'\Vert^2
			=       &     \Vert\theta_{j,k}-\theta_{j,k}'-\eta_t(g_j(\theta_{j,k})-g_j(\theta_{j,k}))\Vert^2 \\
			\leq    &     \Vert\theta_{j,k}-\theta_{j,k}'\Vert^2 + \eta_t^2\|g_j(\theta_{j,k})-g_j(\theta_{j,k})\|^2 -2\eta_t\langle\theta_{j,k}-\theta_{j,k}', g_j(\theta_{j,k})-g_j(\theta_{j,k})\rangle
		\end{aligned}
	\end{equation*}
	According to Lemma \ref{Lemma 19},we further have
	\begin{equation*}
		\begin{aligned}
			\|\theta_{j,k +1}-\theta_{j,k +1}'\|^2 \leq \Vert\theta_{j,k}-\theta_{j,k}'\Vert^2 + \eta^2_t\|g_j(\theta_{j,k})-g_j(\theta_{j,k})\|^2 + \eta_t^2(2\eta_t\zeta_\theta-3)\|g_j(\theta_{j,k})-g_j(\theta_{j,k})\| + 2\eta\varsigma_k \| \theta_{j,k +1}-\theta_{j,k +1}' \|^2.
		\end{aligned}
	\end{equation*}
	Then following $\eta_t\leq \frac{1}{2\zeta_\theta}$ we have
	\begin{equation*}
		\begin{aligned}
			\|\theta_{j,k +1}-\theta_{j,k +1}'\|^2 \leq \frac{1}{1-2\eta_t\varsigma_{t,k}} \Vert\theta_{j,k}-\theta_{j,k}'\Vert^2.
		\end{aligned}
	\end{equation*}
	If Unrolling it gives
	\begin{equation}
		\label{Theorem 4.1}
		\begin{aligned}
			\|\theta_{j,K}-\theta_{j,K}'\|^2 
			\leq &\prod_{k=0}^{K-1}\frac{1}{1-2\eta_t\varsigma_{t,k}} \Vert\theta_t-\theta_t'\Vert^2. \\
		\end{aligned}
	\end{equation}
	For client $i$, there are two cases to consider. In the first case, SGD selects non-perturbed samples in $\mathcal{S}$ and $\mathcal{S}^{(i)}$, which happens with probability $1-\frac{1}{n_i}$. Then, we have
	\begin{equation*}
		\begin{aligned}
			\Vert\theta_{i,k+1}-\theta_{i,k+1}'\Vert^2
			=       &     \Vert\theta_{i,k}-\theta_{i,k}'-\eta_t(g_i(\theta_{i,k})-g_i(\theta_{i,k}')) \|^2 \\
			\leq    &     \frac{1}{1-2\eta_t\varsigma_{t,k}}\Vert\theta_{i,k}-\theta_{i,k}'\Vert^2.
		\end{aligned}
	\end{equation*}
	In the second case, SGD encounters the perturbed sample at time step $k$, which happens with probability $\frac{1}{n_i}$. Then, we have
	\begin{equation*}
		\begin{aligned}
			\Vert\theta_{i,k+1}-\theta_{i,k+1}'\Vert^2
			=       &     \Vert\theta_{i,k}-\theta_{i,k}'-\eta_t(g_i(\theta_{i,k})-g_i'(\theta_{i,k}'))\Vert^2         \\
			\leq    &      (1+p)\Vert\theta_{i,k+1}-\theta_{i,k+1}'\Vert^2 + (1+\frac{1}{p})\eta_t^2\Vert g(\theta_{i,k})-g(\theta_{i,k}')\Vert^2\\
			\leq    &     (1+p)\|\theta_{i,k+1}-\theta_{i,k+1}'\|^2 + 2(1+\frac{1}{p})\eta_t^2\big(\|g(\theta_{i,k})\|^2+\|g(\theta_{i,k})^{\prime}\|^2\big).
		\end{aligned}
	\end{equation*}
	Combining these two cases for client i, we have
	\begin{equation*}
		\begin{aligned}
			\mathbb{E}\Vert\theta_{i,k+1}-\theta_{i,k+1}'\Vert^2 
			\leq    &     (\frac{1}{1-2\eta_t\varsigma_{t,k}} + \frac{p}{n_i})\mathbb{E}\|\theta_{i,k}-\theta_{i,k}'\|^2 + \frac{4(1+1/p)\eta_t^2}{n_i}\mathbb{E}\|g(\theta_{i,k})\|^2.   \\
		\end{aligned}
	\end{equation*}
	Then unrolling it gives,
	\begin{equation}
		\label{Theorem 4.2}
		\begin{aligned}
			\mathbb{E}\Vert\theta_{i,K}-\theta_{i,K}'\Vert^2
			\leq    &     \prod_{k=0}^{K-1}(\frac{1}{1-2\eta_t\varsigma_{t,k}} + \frac{p}{n_i})\mathbb{E}\Vert\theta_t-\theta_t'\Vert^2 
			+ \sum\limits_{k=0}^{K-1}\prod_{l=k+1}^{K-1}(\frac{1}{1-2\eta_{i,l}\varsigma_{t,l}}+\frac{p}{n_i})\dfrac{4(1+1/p)\eta_t^2}{n_i}\mathbb{E} \Vert g(\theta_{i,k}) \Vert^2. \\
		\end{aligned}
	\end{equation}
	Following the similar proof of Theorem 7 in \cite{lei2022stability}, since $\|\theta_{i,k}-\theta_{i,0}\|\leq H_K$ and $\theta_{i,k}'-\theta_{i,0}\|\leq H_K$, we know
	\begin{equation*}
		\begin{aligned}
			\varsigma_{t,k} \leq \frac{(1+2\eta_t\zeta_\theta)b'H_K}{\sqrt{m}}.
		\end{aligned}
	\end{equation*} 
	Furthermore, $s\geq 16\eta_t^2T^2K^2(b'H_K)^2(1+2\eta_t\zeta_\theta)^2$ implies $2\eta\varsigma_{t,k}\leq \frac{1}{tK+1}$ and therefore 
	\begin{equation}\label{Theorem 4.3}
		\prod_{l=k+1}^K\left(\frac1{1-2\eta_t\varsigma_{t,l}}+\frac1K\right)\leq\left(\frac1{1-1/(K+1)}+\frac1K\right)^K\leq\left(1+\frac2K\right)^K\leq e^2.
	\end{equation}
	\begin{equation}\label{Theorem 4.4}
		\prod_{l=s+1}^{t-1}\prod_{k=0}^{K-1}\left(\frac1{1-2\eta_t\varsigma_{l,k}}+\frac1{tK}\right)\leq\left(\frac1{1-1/(tK+1)}+\frac1{tK}\right)^{tK}\leq e^2.
	\end{equation}
	By \ref{Theorem 4.1},\ref{Theorem 4.2},\ref{Theorem 4.3} and \ref{Theorem 4.4}, we have
	\begin{equation*}
		\begin{aligned}
			\mathbb{E}\Vert\theta_{t+1}-\theta_{t+1}'\Vert
			=       &     \mathbb{E}\Vert \dfrac{1}{m} \sum\limits_{j=1}^m(\theta_{j,K} - \theta_{j,K}') \Vert  \\
			\leq    &     \dfrac{1}{m} \sum\limits_{j = 1}^m \sqrt{\mathbb{E}\Vert\theta_{j,K} - \theta_{j,K}' \Vert^2} \\    
			\leq    &     \dfrac{1}{m} \sum\limits_{j = 1}^{m}  \sqrt{\prod_{k=0}^{K-1}(\frac{1}{1-2\eta_t\varsigma_{t,k}} + \frac{p}{n_i})}\mathbb{E}\Vert\theta_t-\theta_t'\Vert
			+ \frac{1}{m}\sqrt{ \sum\limits_{k=0}^{K-1} \dfrac{4e^2(1+1/p)\eta_t^2}{n_i}}\mathbb{E} \Vert g(\theta_{i,k}) \Vert \\
			\leq    & \frac{1}{m}\sum_{s=0}^{t-1}\sqrt{\prod_{l=s+1}^{t-1}\prod_{k=0}^{K-1}(\frac{1}{1-2\eta_t\varsigma_{l,k}} + \frac{p}{n_i}) \sum\limits_{k=0}^{K-1} \dfrac{4e^2(1+1/p)\eta_t^2}{n_i}}\mathbb{E} \Vert g(\theta_{i,k}) \Vert \\
			\leq & \frac{1}{m}\sum_{s=0}^{t-1}\sqrt{\Big(\frac{1}{1-1/(tK+1)}+\frac{1}{tK}\Big)^{tK}\sum_{k=0}^{K-1}\dfrac{4e^2(1+tK/n_i)}{n_i}\eta_t^2}\mathbb{E}\|g(\theta_{i,k})\|\\
			\leq & \frac{1}{m}\sum_{s=0}^{t-1}\sqrt{\sum_{k=0}^{K-1}\frac{4e^4(n_i+tK)}{n_i^2}}\mathbb{E}\|g(\theta_{i,k})\|.
		\end{aligned}
	\end{equation*}
	Then we can obtain
	\begin{equation*}
		\begin{aligned}
			\mathbb{E}\Vert\theta_{T}-\theta_{T}'\Vert 
			\leq    &    \frac{1}{m}\sum_{t=0}^{T-1}\sqrt{\sum_{k=0}^{K-1}\frac{4e^4(n_i+tK)}{n_i^2}}\mathbb{E}\|g(\theta_{i,k})\| \\
			\leq & \frac{1}{m}\sum_{t=0}^{T-1}\sqrt{\sum_{k=0}^{K-1}\frac{4e^4(n_i+tK)}{n_i^2}}(1+(1+\zeta_\theta K\eta_t))(\mathbb{E}\|\nabla R(\theta_t)\|+(2\zeta_x+6L)D_i+\sigma).
		\end{aligned}
	\end{equation*}
	Combining Theorem \ref{Theorem 7} we get
	\begin{equation*}
		\begin{aligned}
			\epsilon_{gen} \leq & \mathcal{O}( \frac{T^\frac{1}{2}\sqrt{\Delta}}{mn_{\min}}+ \frac{T^\frac{3}{4}\sigma}{m^{\frac{5}{4}}n_{\min}K^\frac{1}{4}}  +\frac{T(\rho \sqrt{s}+1)D_{\max}}{mn_{\min}}+ \frac{T\sigma^\frac{1}{4}}{mn_{\min}} ) \\
			=& \mathcal{O}\left(\frac{T^{\frac{1}{2}}\sqrt{\Delta}}{mn_{\mathrm{min}}}+\frac{T(\rho^2\sqrt{s}+1)D_{\mathrm{max}}}{mn_{\mathrm{min}}}\right).
		\end{aligned}
	\end{equation*}
\end{proof}
\textbf{Discussion of the term including $K$ and $\sigma$.} Similar to the discussion of the term including \(\Delta\) in the main paper, the differences in the order of terms involving \(K\) and \(\sigma\) among three smoothness approximations also relate to their smoothing effects. For instance, in the RSA method, the term that includes both \(K\) and \(\sigma\) is optimized compared to SSA, owing to its derivation from a more flattened weight space.
\section{Generalization Analyses of SFAL}
As we discuss in the main paper, SFAL only modifies the global aggregation method compared to VFAL, so we only give the proof of the theorem and Lemma related to global aggregation, and the rest is similar to the proof of VFAL.\\
\textbf{Theorem 5}. \textit{If a SFAL algorithm $\mathcal{A}$ is $\epsilon$-on-averagely stable, we can obtain the generalization error $\varepsilon_{gen}(\mathcal{A})$ as follows: }
\begin{equation}
	\begin{split}
		& \mathbb{E}_{\mathcal{S,A}} \left[\frac{1+\alpha}{\tilde{m}}\sum_{i=1}^{\hat{m}}(R^{\phi_{(i)}}_i(\mathcal{A}(\mathcal{S})) -R^{\phi_{(i)}}_{\mathcal{S}_i}(\mathcal{A}(\mathcal{S})))\right] +\\
		& \mathbb{E}_{\mathcal{S,A}} \left[\frac{1-\alpha}{\tilde{m}}\sum_{i=\hat{m}+1}^{m}(R^{\phi_{(i)}}_i(\mathcal{A}(\mathcal{S}))-R^{\phi_{(i)}}_{\mathcal{S}_i}(\mathcal{A}(\mathcal{S})))\right] \leq \epsilon.\nonumber
	\end{split}
\end{equation}  
\begin{proof}
	Given $\mathcal{S}$ and $\mathcal{S}^{(i')}$ which are neighboring datasets defined in Definition \ref{Definition 1}
	\begin{equation*}
		\begin{aligned}
			\mathbb{E}_{\mathcal{S}}\left[R_{\mathcal{S}_{i}}(\mathcal{A}(\mathcal{S}))\right] 
			=       &     \mathbb{E}_{\mathcal{S}}\left[\frac{1}{n_{i}} \sum_{j=1}^{n_{i}} \ell_{\rho}\left(\mathcal{A}(\mathcal{S}) ; z_{i, j}\right)\right] \\
			=       &     \frac{1}{n_{i}} \sum_{j=1}^{n_{i}} \mathbb{E}_{\mathcal{S}}\left[\ell_{\rho}\left(\mathcal{A}(\mathcal{S}) ; z_{i, j}\right)\right] = \frac{1}{n_{i}} \sum_{j=1}^{n_{i}} \mathbb{E}_{\mathcal{S}, z_{i', j}^{\prime}}\left[\ell_{\rho}\left(\mathcal{A}\left(\mathcal{S}^{(i)}\right) ; z_{i', j}^{\prime}\right)\right].
		\end{aligned}
	\end{equation*}
	Moreover. we have
	$$ \mathbb{E}_{\mathcal{S}}\left[R_{i}(\mathcal{A}(\mathcal{S}))\right]=\frac{1}{n_{i}} \sum_{j=1}^{n_{i}} \mathbb{E}_{\mathcal{S}, z_{i', j}^{\prime}}\left[\ell_{\rho}\left(\mathcal{A}(\mathcal{S}) ; z_{i', j}^{\prime}\right)\right], $$ 
	since $z'_{i',j}$ and $\mathcal{S}$ are independent for any $j$. There is two cases need to consider. \\
	In the first case, we consider $i\in[1, \hat{m}]$
	\begin{equation*}
		\begin{aligned}
			& \mathbb{E}_{\mathcal{A}, \mathcal{S}}\left[ \frac{1+\alpha}{\tilde{m}}\sum_{i=1}^{\hat{m}}\big(R^{\phi_{(i)}}_{i}(\mathcal{A}(\mathcal{S}))-R^{\phi_{(i)}}_{\mathcal{S}_{i}}(\mathcal{A}(\mathcal{S}))\big) \right]\\
			=& \frac{1+\alpha}{\tilde{m}}\sum_{i=1}^{\hat{m}}\mathbb{E}_{\mathcal{A}}\left[\frac{1}{n_{i}} \sum_{j=1}^{n_{i}} \mathbb{E}_{\mathcal{S}, z_{i', j}^{\prime}}\left(\ell_{\rho}\left(\mathcal{A}(\mathcal{S}) ; z_{i', j}^{\prime}\right)-\ell_{\rho}\left(\mathcal{A}\left(\mathcal{S}^{(i)}\right) ; z_{i', j}^{\prime}\right)\right)\right] \\
			\leq& \frac{(1+\alpha)\hat{m}}{\tilde{m}}\epsilon.
		\end{aligned}
	\end{equation*}
	In the second case, we consider $i\in[\hat{m}+1,m]$
	\begin{equation*}
		\begin{aligned}
			& \mathbb{E}_{\mathcal{A}, \mathcal{S}}\left[ \frac{1-\alpha}{\tilde{m}}\sum_{i=\hat{m}+1}^{m}\big(R^{\phi_{(i)}}_{i}(\mathcal{A}(\mathcal{S}))-R^{\phi_{(i)}}_{\mathcal{S}_{i}}(\mathcal{A}(\mathcal{S}))\big) \right]\\
			=& \frac{1-\alpha}{\tilde{m}}\sum_{i=\hat{m}+1}^{m}\mathbb{E}_{\mathcal{A}}\left[\frac{1}{n_{i}} \sum_{j=1}^{n_{i}} \mathbb{E}_{\mathcal{S}, z_{i', j}^{\prime}}\left(\ell_{\rho}\left(\mathcal{A}(\mathcal{S}) ; z_{i', j}^{\prime}\right)-\ell_{\rho}\left(\mathcal{A}\left(\mathcal{S}^{(i)}\right) ; z_{i', j}^{\prime}\right)\right)\right] \\
			\leq& \frac{(1-\alpha)(m-\hat{m})}{\tilde{m}}\epsilon.
		\end{aligned}
	\end{equation*}
	Combining these two cases, we have
	\begin{equation*}
		\begin{aligned}
			\varepsilon_{gen}(\mathcal{A})=&  \mathbb{E}_{\mathcal{A}, \mathcal{S}}\left[ \frac{1+\alpha}{\tilde{m}}\sum_{i=1}^{\hat{m}}\big(R^{\phi_{(i)}}_{i}(\mathcal{A}(\mathcal{S}))-R^{\phi_{(i)}}_{\mathcal{S}_{i}}(\mathcal{A}(\mathcal{S}))\big) \right] + \mathbb{E}_{\mathcal{A}, \mathcal{S}}\left[ \frac{1-\alpha}{\tilde{m}}\sum_{i=\hat{m}+1}^{m}\big(R^{\phi_{(i)}}_{i}(\mathcal{A}(\mathcal{S}))-R^{\phi_{(i)}}_{\mathcal{S}_{i}}(\mathcal{A}(\mathcal{S}))\big) \right] \\
			\leq & \frac{(1+\alpha)\hat{m}}{\tilde{m}}\epsilon + \frac{(1-\alpha)(m-\hat{m})}{\tilde{m}}\epsilon\\
			= & \epsilon.
		\end{aligned}
	\end{equation*}
\end{proof}
\begin{Lemma}
	\label{Lemma 21}
	If the Assumptions \ref{Assumption 1} and \ref{Assumption 4} hold and  for any learning rate satisfying $\eta \leq \frac{1}{2\beta K}$, we can find the per-round recursion as
	\begin{equation*}
		\begin{aligned}
			\mathbb{E}[R(\theta_{t+1})]-\mathbb{E}[R(\theta_{t})]
			\leq    &   - \frac{K\eta}{2}\mathbb{E}\Vert \nabla R(\theta_t)\Vert^2 
			+ \frac{\eta\beta^2}{\tilde{m}}\sum\limits_{i=1}^{m}\sum\limits_{k=0}^{K-1}\phi^t_{(\alpha,i)}\mathbb{E}\Vert\theta^{t+1}_{i,k}-\theta_t\Vert^2 
			+ K\eta (\xi^2 + \xi L)  
			+ \frac{2\beta K \eta^2\sigma^2}{\tilde{m}} . 
		\end{aligned}
	\end{equation*}
\end{Lemma}
\begin{proof}
	Without otherwise stated, the expectation is conditioned on $\theta$. Beginning from $\xi$-approximately $\beta$-gradient Lipschitz,
	\begin{equation*}
		\begin{aligned}
			R(\theta_{t+1})-R(\theta_{t})-\left \langle\nabla R(\theta_t),\theta_{t+1}-\theta_t\right\rangle 
			\leq    & \frac{\beta}{2}\Vert\theta_{t+1}-\theta_{t}\Vert^2+\xi\Vert\theta_{t+1}-\theta_t\Vert, \\
		\end{aligned}
	\end{equation*}
	where the above inequality we use Lemma \ref{Lemma 2}.
	Considering the update rule of SFAL,
	$$\theta_{t+1} = \theta_t - \frac{\eta}{\tilde{m}}\sum\limits_{i=1}^m \sum\limits_{k=0}^{K-1}\phi^t_{(\alpha,i)}g_i(\theta^{t+1}_{i,k}),$$
	where $\tilde{m}=(1+\alpha)\hat{m}+(1-\alpha)(m-\hat{m}), \phi^t_{(\alpha,i)}$ denotes the weight assigned to the $i$-th client based on the ascending sort of weighted client losses compared with the $\hat{m}$-th one, which can be $1+\alpha$ or $1-\alpha$. Using $ \mathbb{E} [g_i(\theta_{i,k}^{t+1})] = \nabla R_i(\theta_{i,k}^{t+1}) $ and Lemma \ref{Lemma 1}, then we can get
	\begin{equation}\label{Lemma 21.2}
		\begin{aligned}
			\mathbb{E}&[R(\theta_{t+1})]-\mathbb{E}[R(\theta_t)]\\
			\leq    &     \mathbb{E}\left\langle \nabla R(\theta_t),-\frac{\eta}{\tilde{m}}\sum\limits_{i=\eta}^{m}\sum\limits_{k=0}^{K-1}\phi^t_{(\alpha,i)}\nabla R_i(\theta^{t+1}_{i,k}) \right\rangle 
			+ \frac{\beta}{2}\mathbb{E}\Vert\frac{\eta}{\tilde{m}}\sum\limits_{i=1}^m \sum\limits_{k=0}^{K-1}\phi^t_{(\alpha,i)} g_i(\theta^{t+1}_{i,k})\Vert^2 + \xi L K\eta \\
			=       &   - \underbrace{\frac{\eta}{\tilde{m}}\sum\limits_{i=1}^{m}\sum\limits_{k=0}^{K-1}\phi^t_{(\alpha,i)}\mathbb{E}\left\langle \nabla R(\theta_t),\nabla R_i(\theta^{t+1}_{i,k}) \right\rangle}_{\Lambda _1}
			+ \underbrace{\frac{\beta}{2}\mathbb{E}\Vert\frac{\eta}{\tilde{m}}\sum\limits_{i=1}^m \sum\limits_{k=0}^{K-1}\phi^t_{(\alpha,i)}g_i(\theta^{t+1}_{i,k})\Vert^2}_{\Lambda_2} 
			+ \xi L K\eta.
		\end{aligned}
	\end{equation}
	Using   $\left \langle a,b \right\rangle = \frac{1}{2}\Vert a\Vert^2 + \frac{1}{2}\Vert b \Vert^2 - \frac{1}{2}\Vert a-b \Vert^2$, with $a = \nabla R(\theta_t)$    and $b = \nabla R_i(\theta^{t+1}_{i,k})$, we have
	\begin{equation*}
		\begin{aligned}
			\Lambda_1 
			\leq    &   -\frac{\eta}{2\tilde{m}}\sum\limits_{i=1}^{m}\sum\limits_{k=0}^{K-1}\phi^t_{(\alpha,i)}\mathbb{E}\left[\Vert\nabla R(\theta_t)\Vert^2 + \Vert \nabla R_i(\theta^{t+1}_{i,k})\Vert^2 - \Vert\nabla R_i(\theta^{t+1}_{i,k})-\nabla R(\theta_t)\Vert^2\right]   \\
			=       &   -\frac{K\eta}{2}\mathbb{E}\Vert \nabla R(\theta_t)\Vert^2-\frac{\eta}{2\tilde{m}}\sum\limits_{i=1}^{m} \sum\limits_{k=0}^{K-1}\phi^t_{(\alpha,i)}\mathbb{E}\Vert \nabla R_i(\theta^{t+1}_{i,k})\Vert^2   
			+  \frac{\eta}{2\tilde{m}}\sum\limits_{i=1}^{m}\sum\limits_{k=0}^{K-1}\phi^t_{(\alpha,i)}\mathbb{E}\Vert \nabla R_i(\theta^{t+1}_{i,k}) - \nabla R(\theta _t)\Vert^2\\   
			\leq    &   -\frac{K\eta}{2}\mathbb{E}\Vert \nabla R(\theta_t)\Vert^2 - \frac{\eta}{2\tilde{m}}\sum\limits_{i=1}^{m}\sum\limits_{k=0}^{K-1}\phi^t_{(\alpha,i)}\mathbb{E}\Vert \nabla R_i(\theta^{t+1}_{i,k})\Vert^2 
			+ \frac{\eta}{2\tilde{m}}\sum\limits_{i=1}^{m}\sum\limits_{k=0}^{K-1}\phi^t_{(\alpha,i)}\left[ \mathbb{E}\left[2\beta^2\Vert\theta^{t+1}_{i,k}-\theta_{t}\Vert^2 \right]+ 2\xi^2\right] \\
			=    &   -\frac{K\eta}{2}\mathbb{E}\Vert \nabla R(\theta_t)\Vert^2 
			- \frac{\eta}{2\tilde{m}}\sum\limits_{i=1}^{m}\sum\limits_{k=0}^{K-1}\phi^t_{(\alpha,i)}\mathbb{E}\Vert \nabla R_i(\theta^{t+1}_{i,k})\Vert^2 
			+ \frac{\eta\beta^2}{\tilde{m}}\sum\limits_{i=1}^{m}\sum\limits_{k=0}^{K-1}\phi^t_{(\alpha,i)}\mathbb{E}\Vert\theta^{t+1}_{i,k}-\theta_t\Vert^2 
			+ K\xi^2 \eta,
		\end{aligned}
	\end{equation*}
	where we used Lemma \ref{Lemma 1} in the second inequality and noticing that $\theta_t=\theta^{t+1}_{i,0}, \frac{\sum_{i=1}^m\phi^t_{(\alpha,i)}}{\tilde{m}}=1$.
	\begin{equation*}
		\begin{aligned}
			\Lambda_2
			\leq    & \frac{\beta}{2}\left[2\mathbb{E}\Vert \dfrac{\eta}{\tilde{m}}\sum\limits_{i=1}^{m}\sum\limits_{k=0}^{K-1}\phi^t_{(\alpha,i)}g_i(\theta^{t+1}_{i,k})-\dfrac{\eta}{\tilde{m}}\sum\limits_{i=1}^{m}\sum\limits_{k=0}^{K-1}\phi^t_{(\alpha,i)}\nabla R_i(\theta^{t+1}_{i,k})\Vert^2 
			+2\mathbb{E}\Vert \dfrac{\eta}{\tilde{m}}\sum\limits_{i=1}^{m}\sum\limits_{k=0}^{K-1}\phi^t_{(\alpha,i)}\nabla R_i(\theta^{t+1}_{i,k})\Vert^2 \right] \\
			\leq    &   \dfrac{\beta \eta^2}{\tilde{m}^2}\sum\limits_{i=1}^{m}\sum\limits_{k=0}^{K-1}(\phi^t_{(\alpha,i)})^2\mathbb{E}\Vert g_i(\theta^{t+1}_{i,k}) - R_i(\theta^{t+1}_{i,k})\Vert^2 
			+ \dfrac{mK\eta^2}{\tilde{m}^2}\sum\limits_{i=1}^{m}\sum\limits_{k=0}^{K-1} (\phi^t_{(\alpha,i)})^2\mathbb{E}\Vert \nabla R_i(\theta^{t+1}_{i,k}) \Vert^2 \\
			\leq    & \frac{2\beta K\eta^2\sigma^2}{\tilde{m}} 
			+  \dfrac{\beta mK\eta^2}{\tilde{m}^2}\sum\limits_{i=1}^{m}\sum\limits_{k=0}^{K-1} (\phi^t_{(\alpha,i)})^2\mathbb{E}\Vert \nabla R_i(\theta^{t+1}_{i,k}) \Vert^2.\\         
		\end{aligned}
	\end{equation*}
	In the second inequality, we used Lemma \ref{Lemma 4} for the first and the second item.  
	In the third inequality, we used Assumption \ref{Assumption 4} and $\frac{\sum_{i=1}^m (\phi^t_{(\alpha,i)})^2}{\tilde{m}}\leq 2$ for the first item. \\
	Substituting $\Lambda_1$ and $\Lambda_2$ into (\ref{Lemma 21.2}), we have
	\begin{equation*}
		\begin{aligned}
			\mathbb{E}[R(\theta_{t+1})]-\mathbb{E}[R(\theta_{t})]
			\leq    &   \Lambda_1 + \Lambda_2 + \xi LK\eta \\
			\leq    &   - \frac{K\eta}{2}\mathbb{E}\Vert \nabla R(\theta_t)\Vert^2 
			+ \frac{\eta\beta^2}{\tilde{m}}\sum\limits_{i=1}^{m}\sum\limits_{k=0}^{K-1}\phi^t_{(\alpha,i)}\mathbb{E}\Vert\theta^{t+1}_{i,k}-\theta_t\Vert^2 
			+ K\eta (\xi^2 + \xi L)  \\
			&   + \frac{2\beta K \eta^2\sigma^2}{\tilde{m}}  
			- \sum\limits_{i=1}^{m}\sum\limits_{k=0}^{K-1}\dfrac{\phi^t_{(\alpha,i)}\eta}{2\tilde{m}}\left(1 - \frac{m\phi^t_{(\alpha,i)}}{2\tilde{m}}2\beta K\eta\right) \mathbb{E} \Vert \nabla R_i(\theta_{i,k}^{t+1}) \Vert^2.
		\end{aligned}
	\end{equation*}
	Considering $\frac{m}{\tilde{m}}\beta K\eta\leq \frac{1}{2}$, we can get $\left( 1 - \frac{m\phi^t_{(\alpha,i)}}{2\tilde{m}}2\beta K\eta\right)\geq 0$. Then, this completes the proof.
\end{proof}
Next we give the proof of Theorem \ref{Theorem 6}.\\
\textbf{Theorem 6}. \textit{Given $\alpha \in [0,1),\hat{m} \leq \frac{m}{2}, r_\alpha = 1+ \frac{\alpha}{1-\alpha}\frac{2\hat{m}}{m}$, we can obtain the following results in SFAL:} \\
\noindent \textit{1. Under Assumption~\ref{Assumption 1} and \ref{Assumption 4}, let the step size be chosen as $\eta_{t} \leq \frac{\tilde{m}}{m\beta K(t+1)}$, the generalization bound $\varepsilon_{gen}$ with the surrogate smoothness approximation satisfies:}
\begin{equation*}
	\mathcal{O}\left(\rho T\log T+\frac{T\sqrt{\log T \Delta}}{r_\alpha mn_{\mathrm{min}}}+\frac{T\log T(\rho+1)D_{\mathrm{max}}}{r_\alpha mn_{\mathrm{min}}}\right).
\end{equation*}
\noindent \textit{2. Let $c\geq 0$ be a constant and the step size be chosen as $\eta_t\leq \frac{\tilde{m}\gamma}{4mK\sqrt{d}cL(t+1)}$. Under Assumption \ref{Assumption 1} and \ref{Assumption 4}, the generalization bound $\varepsilon_{gen}$ with the randomized smoothness approximation satisfies:} 
\begin{equation*}
	\mathcal{O}\left(\frac{T^{\frac{1}{4}}\log T}{\sqrt{Q}}+\frac{T^{\frac{3}{4}}\sqrt{\Delta}}{r_\alpha mn_{\mathrm{min}}}+\frac{T((\rho+1)D_{\mathrm{max}})^{\frac{1}{3}}}{r_\alpha mn_{\mathrm{min}}}\right).
\end{equation*}
\noindent \textit{3. Without loss of generality, we assume $4K\eta_t C_0 \geq 1$ and $s\geq 16\eta_t^2T^2K^2(b'H_K)^2(1+2\eta_t\zeta_\theta)^2$,where $b'= C_x^2B_{\varphi''}(C_xB_{\varphi'}+\sqrt{2C_0}),H_K=2\sqrt{K\eta_t C_0}$. Let the step size be chosen as $\eta_t\leq \frac{\tilde{m}}{6mK\zeta_\theta}$. Under Assumptions \ref{Assumption 1}-\ref{Assumption 4}, the generalization bound $\varepsilon_{gen}$ with the over-parameterized smoothness approximation satisfies:}
\begin{equation*}
	\mathcal{O}\left(\frac{T^{\frac{1}{2}}\sqrt{\Delta}}{r_\alpha mn_{\mathrm{min}}}+\frac{T(\rho^2\sqrt{s}+1)D_{\mathrm{max}}}{r_\alpha mn_{\mathrm{min}}}\right),
\end{equation*}
where $\Delta = \mathbb{E}[R(\theta_0)] - \mathbb{E}[R(\theta^*)], D_{\max}=\max_{i\in [m]} D_i$ and $n_{\min}$ is the smallest data size on client $i$, $\forall i \in [m]$.
\begin{proof}
	Similar to  that proof of Theorem \ref{Theorem 2}, Given time index $t$ and for client $j$ with $j\ne i$, we have
	\begin{equation*}
		\mathbb{E}\|\theta_{j,K-1}-\theta_{j,K-1}^{\prime}\| \leq
		e^{\beta K \eta_{t}}(\mathbb{E}\|\theta_{t}-\theta_{t}^{\prime}\|+K\eta_{t}\xi). 
	\end{equation*}
	for client $i$, we have
	\begin{equation*}
		\mathbb{E}\|\theta_{i,K}-\theta_{i,K}^{\prime}\| \leq e^{\beta K\eta_{t}}\mathbb{E}\|\theta_{t}-\theta_{t}^{\prime}\|+K\eta_{t}e^{\beta K\eta_{t}}(\xi+\frac{2}{n_{i}}\mathbb{E}\|g_{i}(\theta_{i,k})\|).
	\end{equation*}
	Noting that we need to consider both cases where client $i$ is in up-weighting and down-weighting at round $t$. In the first case, client $i$ is in up-weighting, then we have
	\begin{equation}
		\label{case 1}
		\begin{aligned}
			\mathbb{E}\Vert\theta_{t+1}-\theta_{t+1}'\Vert
			=       &     \mathbb{E}\Vert \dfrac{\phi^t_{(\alpha,j)}}{\tilde{m}} \sum\limits_{j=1}^m(\theta_{j,K} - \theta_{j,K}') \Vert  \\
			\leq    &     \dfrac{\phi^t_{(\alpha,j)}}{\tilde{m}} \sum\limits_{j = 1}^m \mathbb{E}\Vert\theta_{j,K} - \theta_{j,K}' \Vert \\
			=       &     \dfrac{\phi^t_{(\alpha,j)}}{\tilde{m}} \sum\limits_{j = 1,j \neq i}^m \mathbb{E}\Vert\theta_{j,K} - \theta_{j,K}' \Vert + \dfrac{(1+\alpha)}{\tilde{m}}\mathbb{E}\Vert \theta_{i,K} - \theta_{i,K}' \Vert\\                       
			\leq    &      e^{\beta K\eta_{t}}\mathbb{E}\Vert \theta_{t}-\theta_{t}'\Vert 
			+  e^{\beta K\eta_{t}} K\eta_{t} \xi + \dfrac{(1+\alpha)2 K\eta_{t} e^{\beta K\eta_{t}}}{\tilde{m} n_i}\mathbb{E} \Vert g_i (\theta_{i,k}) \Vert, 
		\end{aligned}
	\end{equation}
	where $\tilde{m}=(1+\alpha)\hat{m}+(1-\alpha)(m-\hat{m})$. In the second case, client $i$ is in down-weighting, then we have
	\begin{equation}
		\label{case 2}
		\begin{aligned}
			\mathbb{E}\Vert\theta_{t+1}-\theta_{t+1}'\Vert
			=       &     \mathbb{E}\Vert \dfrac{\phi^t_{(\alpha,j)}}{\tilde{m}} \sum\limits_{j=1}^m(\theta_{j,K} - \theta_{j,K}') \Vert  \\
			\leq    &     \dfrac{\phi^t_{(\alpha,j)}}{\tilde{m}} \sum\limits_{j = 1}^m \mathbb{E}\Vert\theta_{j,K} - \theta_{j,K}' \Vert \\
			=       &     \dfrac{\phi^t_{(\alpha,j)}}{\tilde{m}} \sum\limits_{j = 1,j \neq i}^m \mathbb{E}\Vert\theta_{j,K} - \theta_{j,K}' \Vert + \dfrac{(1-\alpha)}{\tilde{m}}\mathbb{E}\Vert \theta_{i,K} - \theta_{i,K}' \Vert\\                       
			\leq    &     e^{\beta K\eta_{t}}\mathbb{E}\Vert \theta_{t}-\theta_{t}'\Vert 
			+ e^{\beta K\eta_{t}} K\eta_{t} \xi + \dfrac{(1-\alpha)2 K\eta_{t} e^{\beta K\eta_{t}}}{\tilde{m} n_i}\mathbb{E} \Vert g_i (\theta_{i,k}) \Vert \\
			\leq    & e^{\beta K\eta_{t}}\mathbb{E}\Vert \theta_{t}-\theta_{t}'\Vert 
			+ e^{\beta K\eta_{t}} K\eta_{t} \xi + \dfrac{2 K\eta_{t} e^{\beta K\eta_{t}}}{\tilde{m} n_i}\mathbb{E} \Vert g_i (\theta_{i,k}) \Vert .
		\end{aligned}
	\end{equation}
	Using (\ref{case 2}) $\times 2 - $(\ref{case 1}), we obtain
	\begin{equation*}
		\begin{aligned}
			\mathbb{E}\Vert\theta_{t+1}-\theta_{t+1}'\Vert 
			\leq    &     e^{\beta K\eta_{t}}\mathbb{E}\Vert \theta_{t}-\theta_{t}'\Vert 
			+  e^{\beta K\eta_{t}} K\eta_{t} \xi + \dfrac{(1-\alpha)2 K\eta_{t} e^{\beta K\eta_{t}}}{\tilde{m} n_i}\mathbb{E} \Vert g_i (\theta_{i,k}) \Vert \\
			=       &   e^{\beta K\eta_{t}}\mathbb{E}\Vert \theta_{t}-\theta_{t}'\Vert 
			+e^{\beta K\eta_{t}} K\eta_{t} \xi + \dfrac{2 K\eta_{t} e^{\beta K\eta_{t}}}{(1+\frac{\alpha}{1-\alpha}\frac{2\hat{m}}{m}) mn_i}\mathbb{E} \Vert g_i (\theta_{i,k}) \Vert \\
			=       &  e^{\beta K\eta_{t}}\mathbb{E}\Vert \theta_{t}-\theta_{t}'\Vert 
			+  e^{\beta K\eta_{t}} K\eta_{t} \xi + \dfrac{2 K\eta_{t} e^{\beta K\eta_{t}}}{(1+\frac{\alpha}{1-\alpha}r) mn_i}\mathbb{E} \Vert g_i (\theta_{i,k}) \Vert,
		\end{aligned}
	\end{equation*}
	where $r=\frac{2\hat{m}}{m} \in (0,1]$.\\
	Unrolling it and let $\eta_t\leq \frac{1}{K\beta(t+1)}$, we get the result of Surrogate smoothness approximation
	\begin{equation*}
		\begin{aligned}
			\varepsilon_{gen} 
			\leq    &     L\mathbb{E}\Vert\theta_T-\theta_T'\Vert \\     
			\leq    &   \mathcal{O}(\xi T\log T+ \frac{T\sqrt{\log T\Delta} }{(1+\frac{\alpha}{1-\alpha}r)mn_{\min}}+\frac{T\sqrt{\log T}\sigma}{(1+\frac{\alpha}{1-\alpha}r)m n_{\min} K^{\frac{1}{2}}} +\frac{ T\log T(\xi+1)D_{max}}{(1+\frac{\alpha}{1-\alpha}r)mn_{\min}}+ \frac{T\log T \sigma}{(1+\frac{\alpha}{1-\alpha}r)mn_{\min}} )\\
			= & \mathcal{O}\left(\rho T\log T+\frac{T\sqrt{\log T \Delta}}{r_\alpha mn_{\mathrm{min}}}+\frac{T\log T(\rho+1)D_{\mathrm{max}}}{r_\alpha mn_{\mathrm{min}}}\right).
		\end{aligned}
	\end{equation*}
	Following similar steps, we can get the result of Randomized smoothness approximation and Over-parameterized smoothness approximation.
\end{proof}

\end{document}